\documentclass{article}

     \usepackage[preprint,nonatbib]{neurips_2020}

\usepackage[utf8]{inputenc} 
\usepackage[T1]{fontenc}    
\usepackage[draft]{hyperref}       
\usepackage{url}            
\usepackage{booktabs}       
\usepackage{amsfonts}       
\usepackage{nicefrac}       
\usepackage{microtype}      

\usepackage{graphicx}
\usepackage{amsmath,amsthm,amssymb}
\usepackage{color}
\usepackage{subfigure}
\usepackage[font=small]{caption}
\usepackage{pifont}
\usepackage{diagbox}
\usepackage{multirow}
\usepackage{enumitem}
\usepackage{algorithm}
\usepackage[noend]{algorithmic}

\usepackage{geometry}  

\usepackage{amsmath,amsthm,amssymb}
\usepackage{color}

\newcommand{\beq}{\begin{equation}}
\newcommand{\eeq}{\end{equation}}
\newcommand{\beqs}{\begin{eqnarray}}
\newcommand{\eeqs}{\end{eqnarray}}
\newcommand{\barr}{\begin{array}}
	\newcommand{\earr}{\end{array}}
\newcommand{\bali}{\begin{aligned}}
	\newcommand{\eali}{\end{aligned}}

\newcommand{\Jc}[0]{\ensuremath{\mathcal{J}} }

\newcommand{\Lc}[0]{\ensuremath{\mathcal{L}} }

\newcommand{\Dbb}[0]{\ensuremath{\mathbb{D}} }
\newcommand{\Ebb}[0]{\ensuremath{\mathbb{E}} }

\newcommand{\Rbb}[0]{\ensuremath{\mathbb{R}} }
\newcommand{\Sbb}[0]{\ensuremath{\mathbb{S}} }

\newcommand{\ie}[0]{\emph{i.e., }}

\newcommand{\eg}[0]{\emph{e.g., }}

\newcommand{\Hten}[0]{\ensuremath{{\boldsymbol{\mathsf{H}}}} }

\newcommand{\Gmat}[0]{\ensuremath{{\bf G}} }
\newcommand{\Hmat}[0]{\ensuremath{{\bf H}} }

\newcommand{\Wmat}[0]{\ensuremath{{\bf W}} }

\newcommand{\av}[0]{\ensuremath{\boldsymbol{a}} }

\newcommand{\cv}[0]{\ensuremath{\boldsymbol{c}} }

\newcommand{\gv}[0]{\ensuremath{\boldsymbol{g}} }

\newcommand{\kv}[0]{\ensuremath{\boldsymbol{k}} }

\newcommand{\pv}[0]{\ensuremath{\boldsymbol{p}} }
\newcommand{\qv}[0]{\ensuremath{\boldsymbol{q}} }

\newcommand{\vv}[0]{\ensuremath{\boldsymbol{v}} }
\newcommand{\wv}[0]{\ensuremath{\boldsymbol{w}} }
\newcommand{\xv}[0]{\ensuremath{\boldsymbol{x}} }
\newcommand{\yv}[0]{\ensuremath{\boldsymbol{y}} }
\newcommand{\zv}[0]{\ensuremath{\boldsymbol{z}} }

\newcommand{\alphav}[0]{\ensuremath{\boldsymbol{\alpha}} }
\newcommand{\betav}[0]{\ensuremath{\boldsymbol{\beta}} }
\newcommand{\gammav}[0]{\ensuremath{\boldsymbol{\gamma}} }

\newcommand{\thetav}[0]{\ensuremath{\boldsymbol{\theta}} }

\newcommand{\muv}[0]{\ensuremath{\boldsymbol{\mu}} }

\newcommand{\sigmav}[0]{\ensuremath{\boldsymbol{\sigma}} }

\newcommand{\phiv}[0]{\ensuremath{\boldsymbol{\phi}} }

\newcommand{\omegav}[0]{\ensuremath{\boldsymbol{\omega}} }

\newcommand{\varthetav}[0]{\ensuremath{\boldsymbol{\vartheta}} }

\newcommand{\Deltav}[0]{\ensuremath{\boldsymbol{\Delta}}}

\newcommand{\Gam}[0]{\ensuremath{\mathrm{Gam}} }
\newcommand{\Pois}[0]{\ensuremath{\mathrm{Pois}} }

\newcommand{\NB}[0]{\ensuremath{\mathrm{NB}} }

\newcommand{\KL}[0]{\ensuremath{\mathrm{KL}} }

\newcommand{\ELBO}[0]{\ensuremath{\mathrm{ELBO}} }
\newcommand{\NN}[0]{\ensuremath{\mathrm{NN}} }

\newcommand{\argmin}{\operatornamewithlimits{argmin}}

\newtheorem{theorem}{Theorem}
\newtheorem{assumption}{Assumption}

\title{GO Hessian for Expectation-Based Objectives}

\author{%
	Yulai Cong\thanks{
		Equal Contribution.
		Correspondence to: Yulai Cong <yulaicong@gmail.com> and Miaoyun Zhao <miaoyun9zhao@gmail.com>.
	}
	\qquad
	Miaoyun Zhao$^{*}$
	\qquad
	Jianqiao Li
	\qquad
	Junya Chen
	\qquad
	Lawrence Carin
	\\
	Department of Electrical and Computer Engineering, Duke University\\
}

\begin{document}

\maketitle

\begin{abstract}

An unbiased low-variance gradient estimator, termed GO gradient, was proposed recently for expectation-based objectives $\Ebb_{q_{\gammav}(\yv)} [f(\yv)]$, where the random variable (RV) $\yv$ may be drawn from a stochastic computation graph with continuous (non-reparameterizable) internal nodes and continuous/discrete leaves. 
Upgrading the GO gradient, we present for $\Ebb_{q_{\gammav}(\yv)} [f(\yv)]$ an unbiased low-variance Hessian estimator, named GO Hessian. 
Considering practical implementation, we reveal that GO Hessian is easy-to-use with auto-differentiation and Hessian-vector products, enabling efficient cheap exploitation of curvature information over stochastic computation graphs. 
As representative examples, we present the GO Hessian for non-reparameterizable gamma and negative binomial RVs/nodes.
Based on the GO Hessian, we design a new second-order method for $\Ebb_{q_{\gammav}(\yv)} [f(\yv)]$, with rigorous experiments conducted to verify its effectiveness and efficiency.

\end{abstract}

\section{Introduction}

Many machine learning problems can be formulated as an optimization problem involving an expectation. A classic such setup \cite{Robbins1951stochastic} is of the form
\beq\label{eq:forward_issue}
\text{Framework I:  }
\min\nolimits_{\varthetav} \Jc(\varthetav) \triangleq \Ebb_{q(\xv)} [h(\xv, \varthetav)],
\eeq
where the random variable (RV) $\xv$ obeys a distribution $q(\xv)$ unrelated to the parameters $\varthetav$ of interest, and $h(\xv, \varthetav)$ is a continuous function wrt $\varthetav$. General assumptions making $\Jc(\varthetav)$ (and the following $\Lc(\gammav)$) a valid loss function are omitted for simplicity. In practice one often encounters its finite-sum form
$ 
\min_{\varthetav} \frac{1}{N} \sum_{i=1}^N h(\xv_i, \varthetav)
$ 
with $q(\xv) = \frac{1}{N} \sum_{i=1}^N \delta(\xv - \xv_i)$, where $\delta(\cdot)$ is the Dirac delta function (this discrete form is typically an approximation, based on $N$ observed samples drawn from the true underlying data distribution).
A popular example of Framework I is maximum-likelihood learning with the data distribution $q(\xv)$ and the negative log-likelihood $h(\xv, \varthetav) = -\log p(\xv ; \varthetav)$, where $p(\xv ; \varthetav)$ represents the model.

An alternative framework, attracting increasing attention recently, considers the form  
\beq\label{eq:reverse_issue}
\text{Framework II:  }
\min\nolimits_{\gammav} \Lc(\gammav) \triangleq \Ebb_{q_{\gammav}(\yv)} [f(\yv)],
\eeq
where parameters $\gammav$ of interest determine the distribution $q_{\gammav}(\yv)$ that, for example, models a stochastic computational graph \cite{schulman2015gradient}. 
Note in general the function $f(\cdot)$ may also be related to $\gammav$; however, as the generalization is straight-forward, we focus on the setup in \eqref{eq:reverse_issue} for simpler derivations. Popular examples of Framework II include the ELBO in variational inference \cite{bishop_2006_PRML,kingma2014auto}, the generator training objective of generative adversarial networks \cite{goodfellow2014generative,arjovsky2017wasserstein,gulrajani2017improved}, and many objectives associated with reinforcement learning  \cite{schulman2015high,finn2017model,foerster2018dice}.

Many optimization methods have been proposed for Framework I, utilizing the first-order gradient information \cite{allen2018natasha,jin2019stochastic} or exploiting the second-order Hessian information \cite{tripuraneni2018stochastic,zhou2019stochastic}. 
Compared with first-order methods, second-order ones are often characterized by convergence in fewer training iterations, requiring less tweaking of meta-parameters (like learning rate), scale invariance to linear parameter rescaling, navigating better when facing pathological curvature in deep learning, and converging to a second-order stationary point \cite{martens2010deep,tripuraneni2018stochastic}. 
For computation and memory efficiency in high-dimensions (like for deep neural networks), recent second-order methods often resort to Hessian-free techniques, \ie Hessian-vector products (HVP) \cite{pearlmutter1994fast,martens2010deep}, which can be computed as efficiently as gradients \cite{pearlmutter1994fast} and remove the need to construct the full Hessian \cite{kohler2017sub,kasai2018inexact}.

In contrast to the classic Framework I, few optimization methods have been proposed for Framework II in \eqref{eq:reverse_issue}, partially because of the significant challenge in even estimating its gradient with low variance without bias in general/non-reparameterizable (subsequently abbreviated as ``non-rep'') situations \cite{cong2019go,weber2019credit,liu2019taming}. 
For second-order optimization of Framework II, most existing works resort to the log-trick,\footnote{
	Also named the likelihood ratio, score function, or REINFORCE estimator. See Section \ref{sec:GO_Hess} for details.
} often suffering from high variance and poor sample efficiency, and therefore seeking help from variance reduction control variates with a potential variance-bias trade-off \cite{heess2015learning,foerster2018dice,rothfuss2018promp}.
Moreover, to facilitate the implementation via auto-differentiation (AD), cumbersome designs of surrogate losses and control variates are often necessary \cite{foerster2018dice,mao2019baseline}, which are challenging when derivatives of different orders are used simultaneously \cite{liu2019taming,farquhar2019loaded}, like in meta reinforcement learning \cite{finn2017model}.
Therefore, an easy-to-use unbiased (gradient and) Hessian estimator for Framework II, with low variance and high sample efficiency, is highly appealing \cite{liu2019taming}.

Different from existing methods that leverage the log-trick, we follow a different research path that tries to generalize the classic deterministic derivatives (obeying the chain rule) to Framework II \cite{figurnov2018implicit,jankowiak2018pathwise,cong2019go}.
Specifically, we upgrade the general GO gradient \cite{cong2019go} to propose an unbiased Hessian estimator for Framework II in \eqref{eq:reverse_issue}, where $\yv$ may be drawn from a stochastic computation graph with continuous rep/non-rep internal nodes and continuous/discrete leaves. The proposed approach is named GO Hessian, and we show that it often works well empirically with one sample without variance reduction techniques. 
Our other contributions are listed as follows.
\vspace{-0.2 cm}
\begin{itemize}[leftmargin=*]
	\setlength{\itemsep}{1pt}
	\setlength{\parskip}{1pt}
	\setlength{\parsep}{1pt}
		
	\item We reveal the proposed GO Hessian is easy to use with AD and HVP, enabling computationally and memory efficient exploitation of curvature information over stochastic graphs.
	
	\item We derive GO Hessian for non-rep gamma and negative binomial RVs; 
	we reveal a simple yet effective method to make optimization over gamma RVs more friendly to gradient-based methods.
	
	\item Marrying the GO Hessian to an existing method for Framework I, we present a novel second-order method for Framework II, theoretically analyze its convergence, and empirically verify its effectiveness and efficiency with rigorous experiments.
\end{itemize}

\vspace{-0.2 cm}
\section{Preliminary}
\vspace{-0.1 cm}

We briefly review 
($i$) the GO gradient \cite{cong2019go}, on which our GO Hessian is based; 
($ii$) Hessian-free techniques for high-dimensional second-order optimization;
and ($iii$) stochastic cubic regularization \cite{tripuraneni2018stochastic}, to which GO Hessian is married to form a novel second-order method for Framework II.

\vspace{-0.1 cm}
\subsection{General and one-sample (GO) gradient}
\label{sec:pre_GO_gradient}
\vspace{-0.1 cm}

Containing as special cases the low-variance reparameterization gradient \cite{salimans2013fixed,rezende2014stochastic} and the pathwise derivative\footnote{
Rigorously, the GO gradient \cite{cong2019go} cannot fully cover the pathwise derivative \cite{jankowiak2018pathwise} on multivariate correlated RVs; but that uncoverage is rare in practice, because common multivariate RVs are either rep (like a multivariate normal RV) or can be reparametrized before GO gradient is applied (like a Dirichlet RV). 
} \cite{figurnov2018implicit,jankowiak2018pathwise}, the GO gradient \cite{cong2019go} serves as a general framework of unbiased low-variance gradient estimates for Framework II in \eqref{eq:reverse_issue}, 
where RV $\yv$ may be drawn from a stochastic computation graph \cite{schulman2015gradient,parmas2018total,weber2019credit} with continuous rep/non-rep internal nodes and continuous/discrete leaves
\cite{cong2019go}. 
With the GO gradient, one can forward pass through the stochastic graph with \emph{one sample} activated for each node to estimate the objective, followed by backward-propagating an unbiased low-variance gradient estimate through each node again to the parameters of that graph for updating (see Theorem 3 of \cite{cong2019go}).
The low-variance and one-sample properties make the GO gradient easy-to-use in practice, for example in variational inference with a complicated inference distribution.

To introduce the approach, the simplest setup, \ie a single-layer RV $\yv$ satisfying the conditional-independent assumption $q_{\gammav} (\yv) = \prod\nolimits_{v} q_{\gammav} (y_v)$, is employed to demonstrate the GO gradient, \ie
\beq\label{eq:GO_1}
\nabla_{\gammav} \Lc(\gammav) = 
\nabla_{\gammav} \Ebb_{q_{\gammav} (\yv)} [f(\yv)] 
= \Ebb_{q_{\gammav} (\yv)} \big[
\Gmat_{\gammav}^{q_{\gammav} (\yv)}
\Dbb_{\yv} f(\yv)
\big],
\eeq
where $\Dbb_{\yv} f(\yv) = \big[ \cdots, \Dbb_{y_v} f(\yv), \cdots \big]^T$ with $\Dbb_{y_v} f(\yv) \triangleq \nabla_{y_v} f(\yv)$ for continuous $y_v$ while $\Dbb_{y_v} f(\yv) \triangleq f(\yv^{v+})-f(\yv)$ for discrete $y_v$, where $\yv^{v+} \triangleq [\cdots,y_{v\!-\!1},y_{v}+1,y_{v\!+\!1},\cdots]^T$.
$\Gmat_{\gammav}^{q_{\gammav} (\yv)} = \big[ \cdots, g_{\gammav}^{q_{\gammav} (y_v)}, \cdots \big]$ gathers the \emph{variable-nabla} $g_{\gammav}^{q_{\gammav} (y_v)} \triangleq  \frac{-1}{q_{\gammav} (y_v)} \nabla_{\gammav}  Q_{\gammav} (y_v)$, which has the intuitive meaning of the ``derivative'' of a RV $y_v$ wrt its parameters $\gammav$ \cite{cong2019go}. $Q_{\gammav} (y_v)$ is the CDF of $q_{\gammav} (y_v)$.

With the \emph{variable-nabla}, one can informally interpret $\Gmat_{\gammav}^{q_{\gammav} (\yv)}$ as the ``gradient'' of the RV $\yv$ wrt the parameters $\gammav$.
Similar intuitive patterns hold for deep stochastic computation graphs with continuous internal nodes \cite{cong2019go}.
As an informal summarization, the GO gradient \emph{in expectation} obeys the chain rule and acts like its special case of the classic back-propagation algorithm \cite{Rumelhart1986learning,cong2019go}.

\subsection{Hessian-free techniques}
\label{sec:HessianFree}

Developed for efficient implementation of second-order optimization in high-dimensions (like for deep neural networks, where the explicit construction of the full Hessian is prohibitive), Hessian-free techniques \cite{martens2010deep,byrd2011use} exploit HVP for \emph{implicit} usage of the Hessian information, for example, via
\beq\label{eq:Hess_vec}
[\nabla_{\varthetav}^2 \Jc(\varthetav)] \pv = \nabla_{\varthetav} \big[ [\nabla_{\varthetav} \Jc(\varthetav)]^T \pv \big],
\eeq
where $\pv$ is a vector uncorrelated with the parameters $\varthetav$ of interest. 
For better efficiency than the above $2$-backward technique, \cite{pearlmutter1994fast} proposed a faster HVP calculation that takes about the same amount of computation as a gradient evaluation.
The low-cost HVP is essential because common subsolvers used to search for second-order directions (like the conjugate gradient method or the cubic-subsolver from \cite{agarwal2017finding,tripuraneni2018stochastic,zhou2019stochastic}) merely exploit Hessian information via HVP.

\subsection{Stochastic cubic regularization (SCR)}

As a second-order method for Framework I, the SCR \cite{tripuraneni2018stochastic} searches for a second-order stationary point via iteratively minimizing a local third-order Taylor expansion of the objective $\Jc(\varthetav)$, \ie
\beq\label{eq:subproblem_cubic}
\varthetav_{t+1} = \argmin_{\varthetav}
	\Jc(\varthetav_t) + \tilde \gv_t^T (\varthetav - \varthetav_t) + 
	\frac{1}{2} (\varthetav - \varthetav_t)^T \tilde \Hmat_t (\varthetav - \varthetav_t) + \frac{\rho}{6} \|\varthetav - \varthetav_t\|^3,
\eeq
where $\tilde \gv_t \!=\! \tilde \nabla_{\varthetav} \Jc(\varthetav_t)$ and $\tilde \Hmat_t \!=\! \tilde \nabla_{\varthetav}^2 \Jc(\varthetav_t)$ are the stochastic gradient and Hessian at $\varthetav_t$, respectively,\footnote{Often $\tilde \gv$ and $\tilde \Hmat$ are estimated via Monte Carlo (MC) estimation, \ie $\tilde \gv = \frac{1}{N_g} \sum\nolimits_{i=1}^{N_g} \nabla_{\varthetav} h(\xv_i, \varthetav), \xv_i \sim q(\xv)$ and $\tilde \Hmat = \frac{1}{N_H} \sum\nolimits_{j=1}^{N_H} \nabla_{\varthetav}^2 h(\xv'_j, \varthetav), \xv'_j \sim q(\xv)$.}
$\rho$ is the cubic penalty coefficient,
and \eqref{eq:subproblem_cubic} can be solved efficiently with gradient decent \cite{carmon2016gradient}.
Since Newton-like methods are much more tolerant to the Hessian estimation error than that of the gradient \cite{byrd2011use}, one can often use significantly less data samples to calculate the stochastic Hessian for better efficiency \cite{tripuraneni2018stochastic}.

\section{GO Hessian for Framework II}

Targeting an efficient second-order optimization of Framework II in \eqref{eq:reverse_issue}, we first propose for it an unbiased low-variance Hessian estimator, termed General and One-sample (GO) Hessian, that systematically upgrades the GO gradient \cite{cong2019go} and is easy-to-use in practice.  
We then marry the proposed GO Hessian to the SCR \cite{tripuraneni2018stochastic} to propose a novel second-order method for Framework II.

\subsection{GO Hessian}
\label{sec:GO_Hess}

A straight-forward way to estimate the Hessian of Framework II in \eqref{eq:reverse_issue} lies in exploiting the log-trick $\nabla_{\gammav} q_{\gammav} (\yv) =q_{\gammav} (\yv) \nabla_{\gammav}  \log q_{\gammav} (\yv)$, generalizing the REINFORCE gradient \cite{williams1992simple}, \ie
\beq\label{eq:Hess_REINFORCE}
\bali
\nabla_{\gammav}^2 \Lc(\gammav) = 
\Ebb_{q_{\gammav} (\yv)} \left[ \bali
f(\yv) [\nabla_{\gammav} \log q_{\gammav} (\yv)] [\nabla_{\gammav} \log q_{\gammav} (\yv)]^T
+ f(\yv) \nabla^2_{\gammav} \log q_{\gammav} (\yv)
\eali \right].
\eali
\eeq
However, such a log-trick estimation shows high MC variance in both theory and practice \cite{rezende2014stochastic,ruiz2016generalized,foerster2018dice,cong2019go}, often seeking help from variance-reduction techniques \cite{grathwohl2017backpropagation,mao2019baseline}.
Moreover, for practical implementation with AD, cumbersome designs of surrogate losses and control variates are often necessary \cite{foerster2018dice,mao2019baseline,liu2019taming,farquhar2019loaded}.

Different from the above method based on the log-trick, our GO Hessian estimates the curvature of Framework II in a pathwise manner like the classic deterministic Hessian.
Specifically, with the GO Hessian, one can forward pass through a stochastic computation graph (\ie $q_{\gammav}(\yv)$; with continuous internal nodes) with \emph{one sample} activated for each node to estimate the objective (\ie the one-sample estimation $f(\yv)$), followed by backward-propagating an unbiased low-variance Hessian estimate through that graph (obeying the chain rule \emph{in expectation}) to estimate the curvature information.
No surrogate loss is necessary for our GO Hessian, which cooperates harmoniously with the GO gradient and often works well in practice with only one sample (see Figure \ref{fig:Hess_var_toy_gam_main} and the experiments).

The key observations motivating our GO Hessian include
($i$) naively employing the integration-by-parts (foundation of the GO gradient) twice fails to deliver an easy-to-use Hessian estimator (see Appendix \ref{secapp:Naive_derivation_HessianII});
($ii$) the \emph{variable-nabla} in \eqref{eq:GO_1} is differentiable with often a simple expression (see Table 3 of \cite{cong2019go});
and ($iii$) the GO gradient empirically shows low variance and often works well with only one sample.
Accordingly, we view the GO gradient of $\Lc(\gammav)$ as another expectation-based vector objective, followed by calculating the GO gradient for that vector objective to form our GO Hessian of $\Lc(\gammav)$.

To simplify notation, we employ the simplest single-layer continuous settings first to demonstrate our main results, which are then generalized to deep stochastic computation graphs with continuous rep/non-rep internal nodes and continuous/discrete leaves.
Detailed proofs are in Appendix \ref{secapp:derivation_GO_Hessian}.

Assuming a single-layer continuous setup with $q_{\gammav} (\yv) = \prod\nolimits_{v} q_{\gammav} (y_v)$, the GO Hessian is defined as 
\beq\label{eq:GO_Hessian_1}
\resizebox{0.94\hsize}{!}{$
	\nabla_{\gammav}^2 \Lc(\gammav) \!=\! 
	\nabla_{\gammav} \Ebb_{q_{\gammav} (\yv)} \!\big[
	\Gmat_{\gammav}^{q_{\gammav} (\yv)}
	\nabla_{\yv} f(\yv) 
	\big]
	\!=\! \Ebb_{q_{\gammav} (\yv)} \!\left[
	\Gmat_{\gammav}^{q_{\gammav}(\yv)} \![\nabla_{\yv}^2 f(\yv)] \Gmat_{\gammav}^{q_{\gammav}(\yv)} {}^T
	\!+\! \Hten_{\gammav\gammav}^{q_{\gammav}(\yv)} \nabla_{\yv} f(\yv)
	\right],
	$}
\eeq
where $\Hten_{\gammav\gammav}^{q_{\gammav}(\yv)}$ is a three-dimensional tensor with its element
\beq\label{eq:Hessian_RV_definition}
	[\Hten_{\gammav\gammav}^{q_{\gammav}(\yv)}]_{b,a,v} = g_{\gamma_b}^{q_{\gammav} (y_v)} \nabla_{y_v} g_{\gamma_a}^{q_{\gammav} (y_v)} + \nabla_{\gamma_b} g_{\gamma_a}^{q_{\gammav} (y_v)}
    \triangleq  h_{\gamma_b \gamma_a}^{q_{\gammav}(y_v)} 
\eeq 
and the tensor-vector product $\Hten\av$ outputs a matrix whose element $[\Hten\av]_{b,a} \!=\! \sum_{v} \Hten_{b,a,v} \av_{v}$. 
We name $h_{\gamma_b \gamma_a}^{q_{\gammav}(y_v)}$ the \emph{variable-hess}, because of its intuitive meaning of the second-order ``derivative'' of a RV $y_v$ wrt parameters $\{\gamma_a,\gamma_b\}$ (see below).

For better understanding, we draw parallel comparisons to deterministic optimization with objective $\hat \Lc(\gammav) = f[\hat \yv(\gammav)]$, which is a special case of Framework II with $q_{\gammav}(\yv) = \delta(\yv - \hat \yv(\gammav))$ and where
\beq\label{eq:grad12_simple_reverse_issue}
\nabla_{\gammav}^2 \hat \Lc(\gammav) 
= [\nabla_{\gammav} \hat \yv(\gammav)] [\nabla_{\hat\yv}^2 f(\hat \yv)] [\nabla_{\gammav} \hat \yv(\gammav)]^T
+ [\nabla_{\gammav}^2 \hat \yv(\gammav)] \nabla_{\hat\yv} f(\hat\yv).
\eeq
By comparing \eqref{eq:GO_Hessian_1} and \eqref{eq:grad12_simple_reverse_issue}, interesting conclusions include
($i$) the interpretation of $\Gmat_{\gammav}^{q_{\gammav}(\yv)}$ as the ``gradient'' of the RV $\yv$ wrt parameters $\gammav$ (informally $\Gmat_{\gammav}^{q_{\gammav}(\yv)} \leftrightarrow \nabla_{\gammav} \yv$) also holds in second-order settings;
($ii$) the newly-introduced $\Hten_{\gammav\gammav}^{q_{\gammav}(\yv)}$ can be intuitively interpreted as the ``Hessian'' of the RV $\yv$ wrt parameters $\gammav$ (informally $\Hten_{\gammav\gammav}^{q_{\gammav}(\yv)} \leftrightarrow \nabla_{\gammav}^2 \yv$), 
but with an additional component originating from the RV randomness (\ie the first item of the \emph{variable-hess} in \eqref{eq:Hessian_RV_definition}); and ($iii$) the GO Hessian contains the deterministic Hessian as a special case (see Appendix \ref{appsec:GOHess_Hessian_specialcase} for the proof).

\begin{figure}[tb]
	\setlength{\abovecaptionskip}{2.0pt}
	\centering	
	\includegraphics[width=0.3\columnwidth]{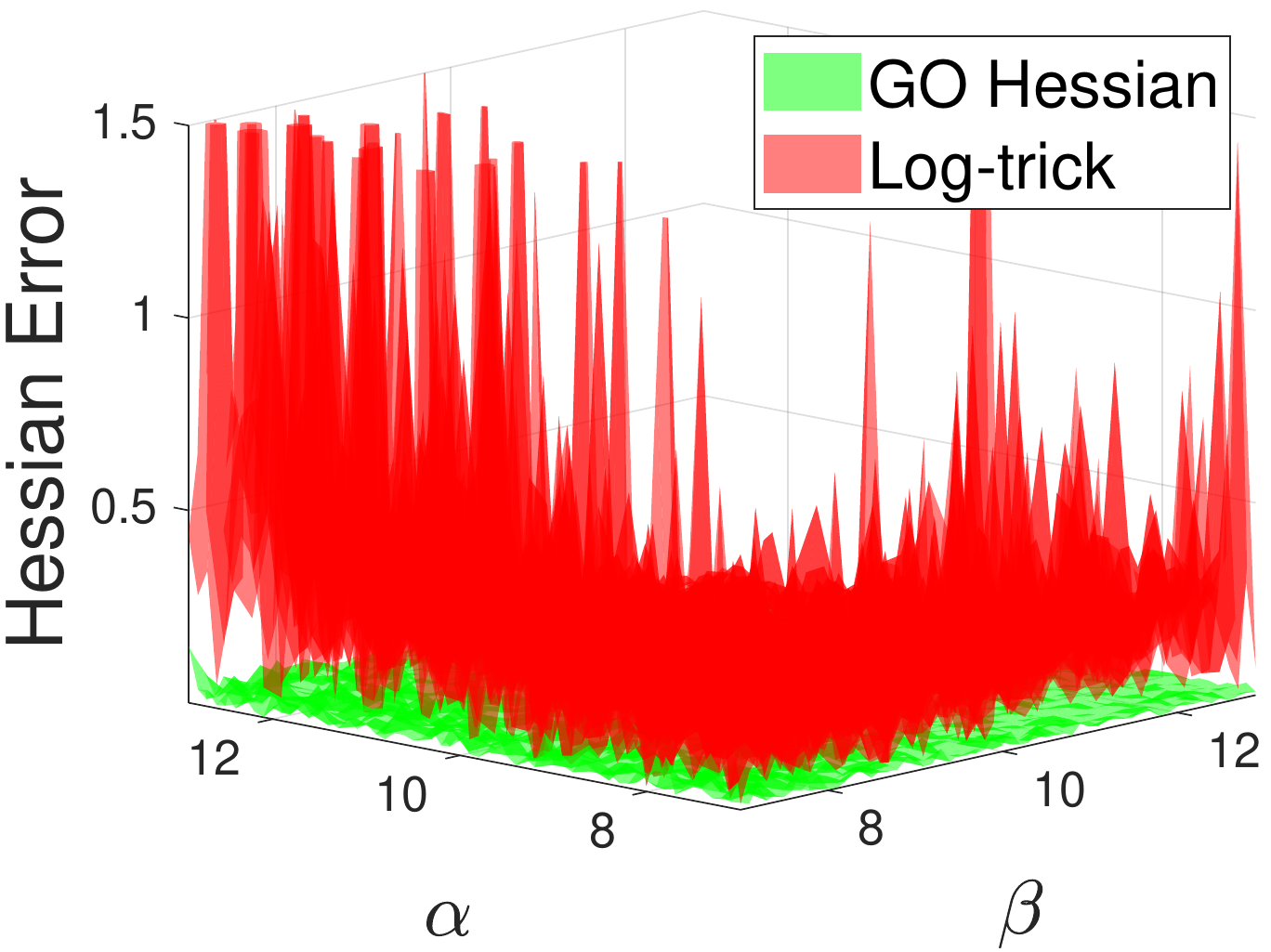}
	\qquad\qquad
	\includegraphics[width=0.3\columnwidth]{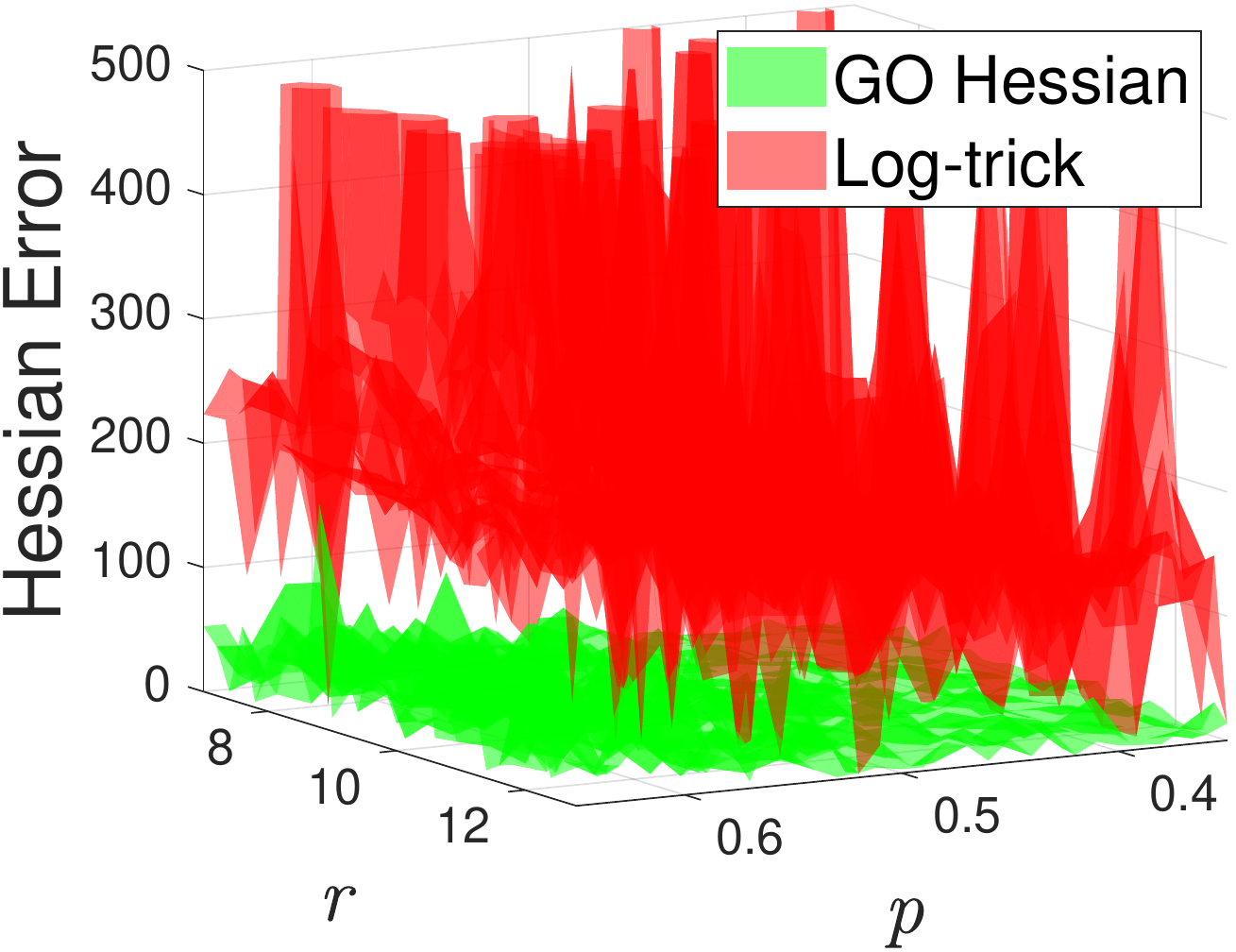}
	\caption{
		Variance comparisons of the one-sample-based log-trick estimation and GO Hessian on 
		(left) $\nabla_{\{\alpha, \beta\}}^2 \Ebb_{\Gam(\alpha, \beta)} \big[\log \frac{\Gam(\alpha, \beta)}{\Gam(10, 10)} \big] $
		and (right) $\nabla_{\{r, p\}}^2 \Ebb_{\NB(r, p)} \big[\log \frac{\NB(r, p)}{\NB(10, 0.5)} \big]$. 
		See Appendix \ref{appsec:GO_Hessian_Variance} for details. 	
	} \label{fig:Hess_var_toy_gam_main}
	\vspace{-0.3 cm}
\end{figure}

\textbf{On discrete RVs}
Upgrading the GO gradient, the GO Hessian for single-layer/leaf discrete RVs is
\beq\label{eq:GO_Hessian_12}
\nabla_{\gammav}^2 \Lc(\gammav) 
\!=\! \Ebb_{q_{\gammav} (\yv)} \!\big[
\Gmat_{\gammav}^{q_{\gammav}(\yv)} \![\Dbb_{\yv}^2 f(\yv)] \Gmat_{\gammav}^{q_{\gammav}(\yv)} {}^T
\!+\! \overline{\Hten_{\gammav\gammav}^{q_{\gammav}(\yv)} \Dbb_{\yv} f(\yv)}
\big],
\eeq
where $\scriptstyle \overline{\Hten_{\gammav\gammav}^{q_{\gammav}(\yv)} \Dbb_{\yv} f(\yv)}$ represents a matrix with its elements 
\beq\label{eq:Hessian_discreteRV_definition}
	\big[{\scriptstyle{\overline{\Hten_{\gammav\gammav}^{q_{\gammav}(\yv)} \Dbb_{\yv} f(\yv)}}} \big]_{b,a} = \sum\nolimits_{v} \left[ 
	[g_{\gamma_b}^{q_{\gammav} (y_v)} \Dbb_{y_v} g_{\gamma_a}^{q_{\gammav} (y_v)}] \Dbb_{y_v} f(\yv^{v+}) 
	+ [\nabla_{\gamma_b} g_{\gamma_a}^{q_{\gammav} (y_v)}] \Dbb_{y_v} f(\yv) 
	 \right]. 
\eeq
It is clear that \eqref{eq:GO_Hessian_1} for continuous RVs and \eqref{eq:GO_Hessian_12} for discrete RVs show similar patterns but with slight differences, like the gradient/difference of $f(\yv)$ and the definition of the \emph{variable-hess}.
We leave discrete situations as future research and focus mainly on continuous cases in this paper.

Based on the above derivations/statements for the single-layer setup, we prove in Appendix \ref{secapp:derivation_GO_Hessian} similar patterns hold for stochastic computation graphs with continuous rep/non-rep internal nodes and continuous/discrete leaves.
In short, the GO Hessian acts like its special case of the classic deterministic Hessian and \emph{in expectation} obeys the chain rule, enabling an one-sample-based forward pass for loss estimation and subsequent backward passes for unbiased low-variance Hessian estimation.

\textbf{On limited $f(\cdot)$-information}
For practical situations where only zero-order $f(\yv)$-information is available at the current sample $\yv$ (\eg $f(\yv)$ is non-differentiable or $\Dbb_{\yv} f(\yv)$ is not accessible), we reveal the LAX technique \cite{grathwohl2017backpropagation} to facilitate our GO Hessian. Specifically, with a surrogate function $c_{\omegav}(\yv)$ (often a neural network) to approximate $f(\yv)$, we unify the zero-order $f$-evaluation from the log-trick estimation (see \eqref{eq:Hess_REINFORCE}) and the low-variance from our GO Hessian via
\beq\label{eq:lax_continuous}
\setlength\abovedisplayskip{4.0pt}
\setlength\belowdisplayskip{4.0pt}
\Hmat_{\text{LAX}}[f] = \Hmat_{\text{log-trick}}[f] - \Hmat_{\text{log-trick}}[c_{\omegav}] + \Hmat_{\text{GO}}[c_{\omegav}],
\eeq
where $\Hmat_{\texttt{method}}[\texttt{func}]$ denotes the Hessian estimator of objective $\Ebb_{q_{\gammav}(\yv)} [\texttt{func}(\yv)]$ based on the \texttt{method}.
The surrogate parameters $\omegav$ can be optimized by minimizing the MC variance of $\Hmat_{\text{LAX}}[f]$ \cite{grathwohl2017backpropagation}. 
Note when $c_{\omegav}(\yv)=f(\yv)$, $\Hmat_{\text{LAX}}[f]$ delivers the same low variance as our GO Hessian.

\vspace{-0.15 cm}
\subsection{GO Hessian is easy-to-use}
\label{sec:GOHess_HVP}
\vspace{-0.1 cm}

By considering practical implementation, to explicitly construct/store the GO Hessian may be prohibitively expensive, especially for stochastic computation graphs with neural-network components. 
Fortunately, we find the GO Hessian is easy to use in practice with AD and HVP.
The key observation is the one-sample-estimated GO Hessian acts the same as its special case of deterministic Hessian (see \eqref{eq:GO_Hessian_1} and Appendix \ref{secapp:derivation_GO_Hessian}), 
despite the \emph{variable-nabla}/\emph{variable-hess} as the first-order/second-order ``derivative'' for each RV node.
Accordingly, one can easily manipulate well-developed AD software (like PyTorch \cite{paszke2017automatic} or TensorFlow \cite{tensorflow2015-whitepaper})
to enable easy-to-use exploitation of the GO Hessian.

Consider the example in Figure \ref{fig:GO_grad}, 
where we focus on a scalar RV node $y$ of a stochastic graph $\yv$ thanks to the conditional independence, $\alphav$ denotes the distribution parameters (\eg the shape and rate of a gamma RV) of that node, and one sample is \emph{stochastically activated} for the subsequent forward pass.
To exploit the GO Hessian with AD, we only need to define the backward pass for each stochastic activation with the approach shown in Figure \ref{fig:PseudoCode}, which guarantees correct \emph{variable-nabla}/\emph{variable-hess} for each RV node and delivers a seamless (double) back-propagation through the whole stochastic graph (as the rest computations are deterministic and well-defined in AD).
Note the HVP for the GO Hessian (GO-HVP) can be similarly implemented as in \eqref{eq:Hess_vec}.

\begin{figure*}[tb]
	\centering
	\resizebox{8.5 cm}{!}{
	\begin{minipage}{11 cm}
			\setlength{\abovecaptionskip}{1.0pt}
			\centering
			\subfigure[Forward/backward pass] {\label{fig:GO_grad}
				\includegraphics[width=0.43\columnwidth]{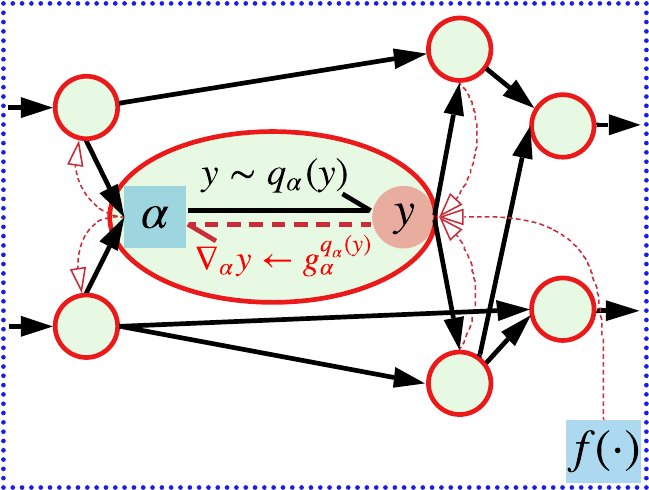}
			}
			\subfigure[PyTorch-like pseudo code] {\label{fig:PseudoCode}
				\includegraphics[width=0.52\columnwidth]{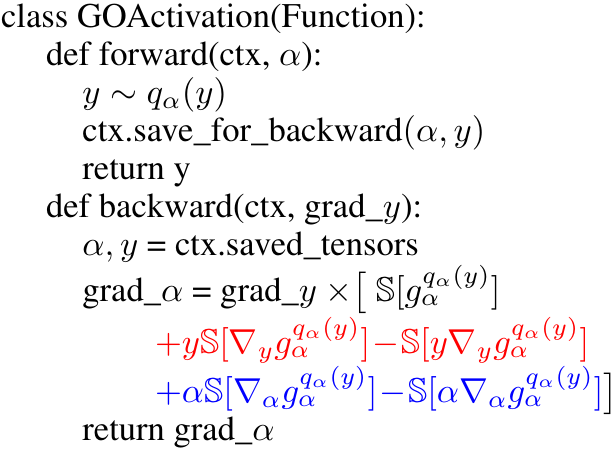}
			}
			\caption{Demonstration of the simplicity in jointly implementing the GO gradient and the GO Hessian.
				A red circle denotes a RV node.
				Black solid arrows indicate both the forward pass and parameters $\gammav$ of the stochastic computation graph.
				Red dashed arrows show the backward pass through the current node. 
				$\Sbb[\cdot]$ is the stop-gradient operator.
			}
			\label{fig:GO_grad_Hess_usage}			
	\end{minipage}	
	}
	\resizebox{5.3 cm}{!}{
	\begin{minipage}{7.5 cm}
		\begin{algorithm}[H]
			\caption{SCR-GO for 
				\protect\resizebox{3.8 cm}{!}{$\min_{\gammav} \Ebb_{q(\xv) q_{\gammav}(\yv|\xv)} [f(\xv,\yv)]$}
			} \label{alg:SCR_GO}
			\begin{algorithmic}
				\REQUIRE {Batch sizes $N_g,N_H$, initialization $\gammav_0$, total iterations $T$, and final tolerance $\epsilon$.}
				\ENSURE {$\epsilon$-second-order stationary point $\gammav^{*}$ or $\gammav_{T+1}$.}
				\FOR {$t = 0, 1, \cdots, T $} 
				
				\STATE {Sample $\{\xv_{i}\}_{i=1}^{N_g}$ and $\{\xv_{j}'\}_{j=1}^{N_H}$ from $q(\xv)$}
				\STATE {Sample $\yv_{i} \sim q_{\gammav_t}(\yv | \xv_i)$ and $\yv_{j}' \sim q_{\gammav_t}(\yv | \xv_{j}')$}
				\STATE {Estimate GO gradient $\tilde \gv_t$ with $\{\xv_{i},\yv_{i}\}_{i=1}^{N_g}$} 
				\STATE {Estimate GO Hessian $\tilde \Hmat_t [\cdot]$ with $\{\xv_{j}',\yv_{j}'\}_{j=1}^{N_H}$}	
				\STATE {$\Deltav, \Delta \leftarrow \text{Cubic-Subsolver}(\tilde \gv_t, \tilde \Hmat_t [\cdot], \epsilon)$}
				\STATE {$\gammav_{t+1} = \gammav_t + \Deltav$}
				
				\IF {$\Delta > 
					-\sqrt{{\epsilon^3}/{\rho}} / 100$}
				
				\STATE {$\Deltav \leftarrow \text{Cubic-Finalsolver}(\tilde \gv_t, \tilde \Hmat_t [\cdot], \epsilon)$}
				\STATE {$\gammav^{*} = \gammav_t + \Deltav$ and \textbf{break}} 
				
				\ENDIF 
				
				\ENDFOR
			\end{algorithmic}
		\end{algorithm}
	\end{minipage}
	}
	\vspace{-0.3 cm}
\end{figure*}

\vspace{-0.15 cm}
\subsection{
	Second-order optimization of Framework II}
\vspace{-0.1 cm}

Benefiting from the low-variance and easy-to-use properties of the GO Hessian, one can readily marry it with existing second-order methods for Framework I to develop novel variants for Framework II in \eqref{eq:reverse_issue}.
Considering practical applications like variational inference, we employ a more common objective for presentation, \ie 
$
\min_{\gammav} \Lc(\gammav) \triangleq \Ebb_{q(\xv) q_{\gammav}(\yv|\xv)} [f(\xv,\yv)],
$
where $q(\xv)$ is defined in \eqref{eq:forward_issue}.

We employ the stochastic cubic regularization (SCR) \cite{tripuraneni2018stochastic}, that exploits stochastic gradient/Hessian information within its subroutine (see \eqref{eq:subproblem_cubic}).
By leveraging the GO gradient and our GO Hessian in place of the classic gradient/Hessian, we present in Algorithm \ref{alg:SCR_GO} a new second-order method for Framework II, termed SCR-GO.
The Cubic-Subsolver and Cubic-Finalsolver (given in Appendix \ref{secapp:SCR_subsolvers}) minimize the local third-order Taylor approximation of Framework II, mimicking \eqref{eq:subproblem_cubic}. 
The detailed convergence analysis is provided in Appendix \ref{secapp:converge_alg1}, where a gamma-related special case is discussed.

\vspace{-0.25 cm}
\section{GO Hessian for common RVs}
\label{sec:GOHess_CommonRV}
\vspace{-0.1 cm}

Based on the \emph{variable-nabla}s summarized in Table 3 of \cite{cong2019go} and the definitions in \eqref{eq:GO_Hessian_1} and \eqref{eq:GO_Hessian_12}, one can derive the GO Hessians for many kinds of RVs, which are essential for easy-to-use curvature exploitation over stochastic computation graphs that are flexibly constructed by those RVs.
Following \cite{cong2019go}, we highlight two challenging RVs: continuous non-rep gamma and discrete negative binomial.

\vspace{-0.2 cm}
\subsection{GO Hessian for non-rep gamma RVs}
\label{sec:GOHess_gamma}
\vspace{-0.1 cm}

To demonstrate the effectiveness of the proposed techniques, we focus on situations with non-rep gamma RVs in our experiments.
Such a concentration is motivated by their broad practical applications \cite{boland2007statistical,mendoza2013characterising,al2014gamma,wright2014response,belikov2017number}
and by their fundamental utility in statistics and machine learning.
For example, many popular distributions can be reparameterized as gamma \cite{leemis2008univariate}, such as exponential, chi-squared, inverse-gamma, log-gamma, beta, and Dirichlet;
other ones can be mixed via gamma \cite{Zhou2012Beta,Zhou2015Negative}, 
like the gamma-normal-mixed student-$t$ 
and the gamma-Poisson-mixed negative binomial.
Accordingly, the presented techniques for gamma can be readily extended to those gamma-related cases of Framework II (\eg variational inference for a model with Dirichlet latent code like \cite{Blei2003Latent}).

From the definition in \eqref{eq:GO_Hessian_1} and the illustrative example in Figure \ref{fig:GO_grad_Hess_usage}, it's clear that three components are crucial in constructing a GO Hessian for a continuous RV, that is, 
\beq\label{eq:3_key_terms_GO_Hessian}
g_{\alphav}^{q_{\alphav} (y)}, \nabla_{y} g_{\alphav}^{q_{\alphav} (y)}, \text{ and } \nabla_{\alphav} g_{\alphav}^{q_{\alphav} (y)}.
\eeq
For a gamma RV $\hat y \sim \Gam(\alpha, \beta)$, the distribution parameters $\alphav$ in general contain both the shape $\alpha$ and the rate $\beta$.
However, we notice the reparameterization of $\hat y = {y}/{\beta},  y \sim \Gam(\alpha, 1)$, with which one can leave the derivatives wrt $\beta$ to AD for simplicity and focus solely on the non-rep part associated with $\alpha$.
Accordingly, we need to calculate the three components in \eqref{eq:3_key_terms_GO_Hessian} for $q_{\alpha} (y) = \Gam(y; \alpha, 1)$.
Moving detailed derivations to Appendix \ref{sec:Gam_Implement} for clarity, we yield
\beq\label{eq:G_alpha_gam_main}
\resizebox{0.75 \hsize}{!}{$
\bali
g_{\alpha}^{q_{\alpha} (y)} & = 
	-[\log y - \psi(\alpha)] \frac{\gamma(\alpha, y)}{y^{\alpha-1} e^{-y}}
	+ \frac{y e^y}{\alpha^2} {}_2 F_2(\alpha, \alpha; \alpha+1, \alpha+1; -y) 
\\
\nabla_{y} g_{\alpha}^{q_{\alpha} (y)} & =
	[\psi(\alpha) - \log y] + g_{\alpha}^{q_{\alpha} (y)} {(y-\alpha+1)}/{y} 
\\
\nabla_{\alpha} g_{\alpha}^{q_{\alpha} (y)} & =
	\psi^{(1)}(\alpha) \frac{\gamma(\alpha, y)}{y^{\alpha-1} e^{-y}}
	+ [\log y - \psi(\alpha)] \frac{y e^y}{\alpha^2} {}_2 F_2(\alpha, \alpha; \alpha+1, \alpha+1; -y) 
\\
& \quad - {2 y e^y}{\alpha^{-3}} {}_3 F_3(\alpha, \alpha, \alpha; \alpha+1,\alpha+1, \alpha+1; -y),
\eali
$}
\eeq
where $\psi(\alpha)$ is the digamma function, $\psi^{(m)}(x)$ the polygamma function of order $m$, $\gamma(\alpha, y)$ the lower incomplete gamma function, and ${}_p F_q(a_1,\cdots,a_p; b_1,\cdots,b_q; x)$ is the generalized hypergeometric function.
Reparameterizing the rate $\beta$ first, followed by substituting the components in \eqref{eq:G_alpha_gam_main} into the approach in Figure \ref{fig:PseudoCode}, one enables easy-to-use exploitation of GO Hessian with AD over a non-rep gamma node.
The low variance of our GO Hessian is illustrated in Figure \ref{fig:Hess_var_toy_gam_main}.



\textbf{A gradient-friendly reparameterization}
To model a gamma node within a stochastic graph, a naive method would parameterize shape $\alpha \!=\! \text{softplus}(\gamma_{\alpha})$ and rate $\beta \!=\! \text{softplus}(\gamma_{\beta})$, where without loss of generality $\gammav\!=\!\{\gamma_{\alpha}, \gamma_{\beta}\}$ is considered as the parameters of interest.
However, we find empirically that such a naive modeling may not be friendly to gradient-based methods, especially when target shape and/or rate are large.
Figure \ref{fig:TwoGAM_RepKLsurf_ab} shows an example with the reverse KL objective $\KL[\Gam(y; \alpha, \beta)||\Gam(y; 200, 200)]$; with that modeling, SGD (labeled as SGD$_{\alpha,\beta}$) bounces between two slopes at the bottom of the ``valley'' and advances slowly.
Alternatively, noticing that the valley bottom is approximately located in a line where $\Gam(y; \alpha, \beta)$ shares the same mean as the target, we propose to reparameterize via mean $\mu$ and standard deviation $\sigma$, \ie $q_{\gammav}(y) = \Gam(y; \frac{\mu^2}{\sigma^2}, \frac{\mu}{\sigma^2})$ with $\mu = \text{softplus}(\gamma_{\mu})$ and $\sigma = \text{softplus}(\gamma_{\sigma})$.
With this reparameterization, we obtain an approximately decorrelated objective surface (see Figure \ref{fig:TwoGAM_RepKLsurf_uv}) that is more friendly to gradient-based methods; it's apparent SGD in the $\mu$-$\sigma$ space, termed SGD$_{\mu,\sigma}$, converges to the optimum much faster.

\vspace{-0.2 cm}
\subsection{GO Hessian for discrete NB RVs}
\vspace{-0.1 cm}

For a NB RV $y \sim q_{\alphav}(y) \!=\! \NB(r, p)$, the distribution parameters $\alphav\!=\!\{r,p\}$ contain both the number of failures $r$ and the success probability $p$.
From the definition in \eqref{eq:GO_Hessian_12}, three components are necessary to calculate the GO Hessian, \ie $g_{\alphav}^{q_{\alphav} (y)}$, $\Dbb_{y} g_{\alphav}^{q_{\alphav} (y)}$, and $\nabla_{\alphav} g_{\alphav}^{q_{\alphav} (y)}$.
Note $\Dbb$ in the second term denotes the difference operator.
Due to space constraints, analytic expressions and detailed derivations are given in Appendix \ref{sec:NB_Implement}.
The low variance of the GO Hessian is demonstrated in Figure \ref{fig:Hess_var_toy_gam_main}.

\vspace{-0.1 cm}
\section{Experiments}
\vspace{-0.1 cm}

\begin{figure}[tb]
	\setlength{\abovecaptionskip}{2.0pt}
	\centering
	\subfigure[$\alpha$-$\beta$ space] {\label{fig:TwoGAM_RepKLsurf_ab}
		\includegraphics[width=0.24\columnwidth]{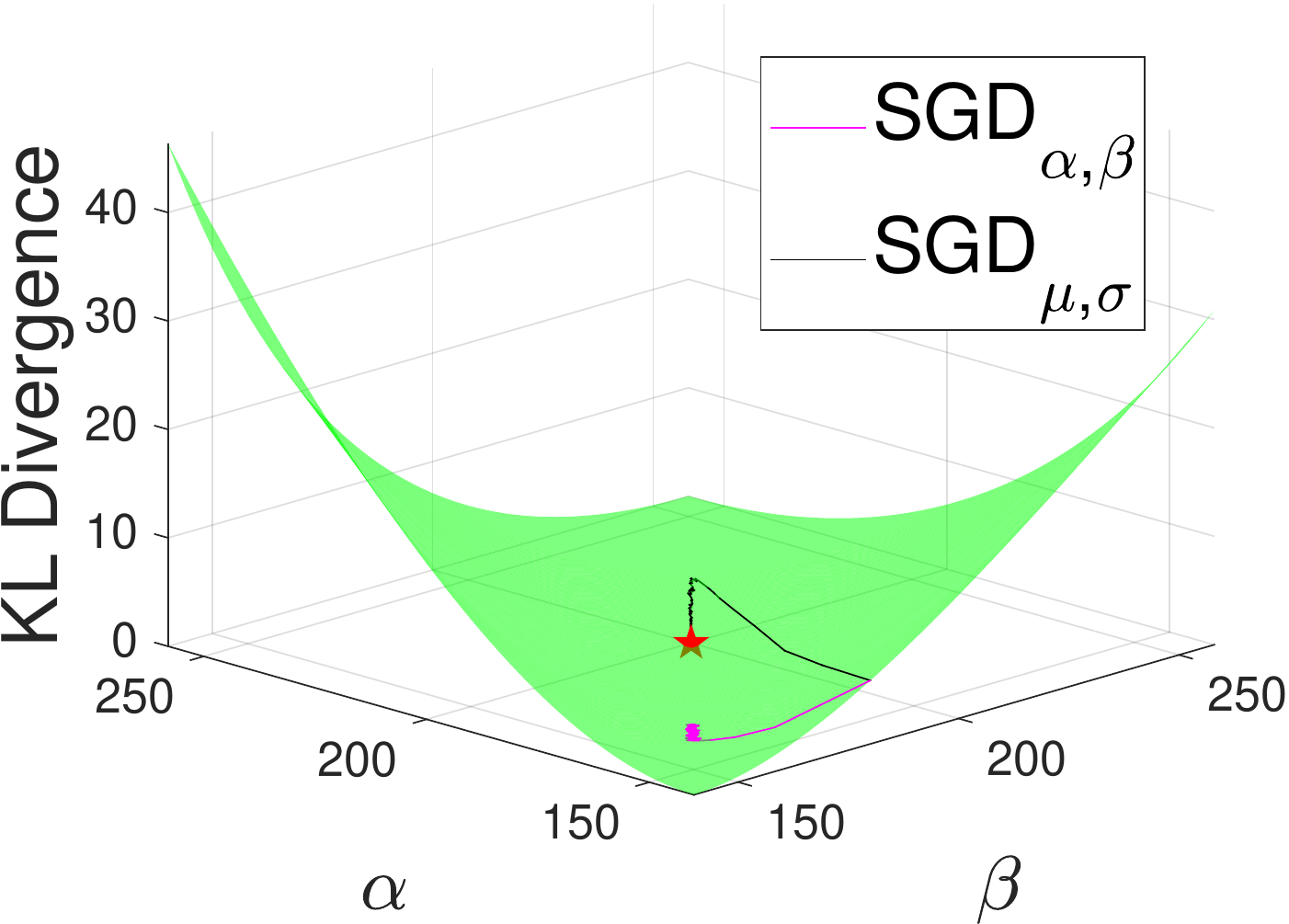}
	}
	\subfigure[$\mu$-$\sigma$ space] {\label{fig:TwoGAM_RepKLsurf_uv}
		\includegraphics[width=0.24\columnwidth]{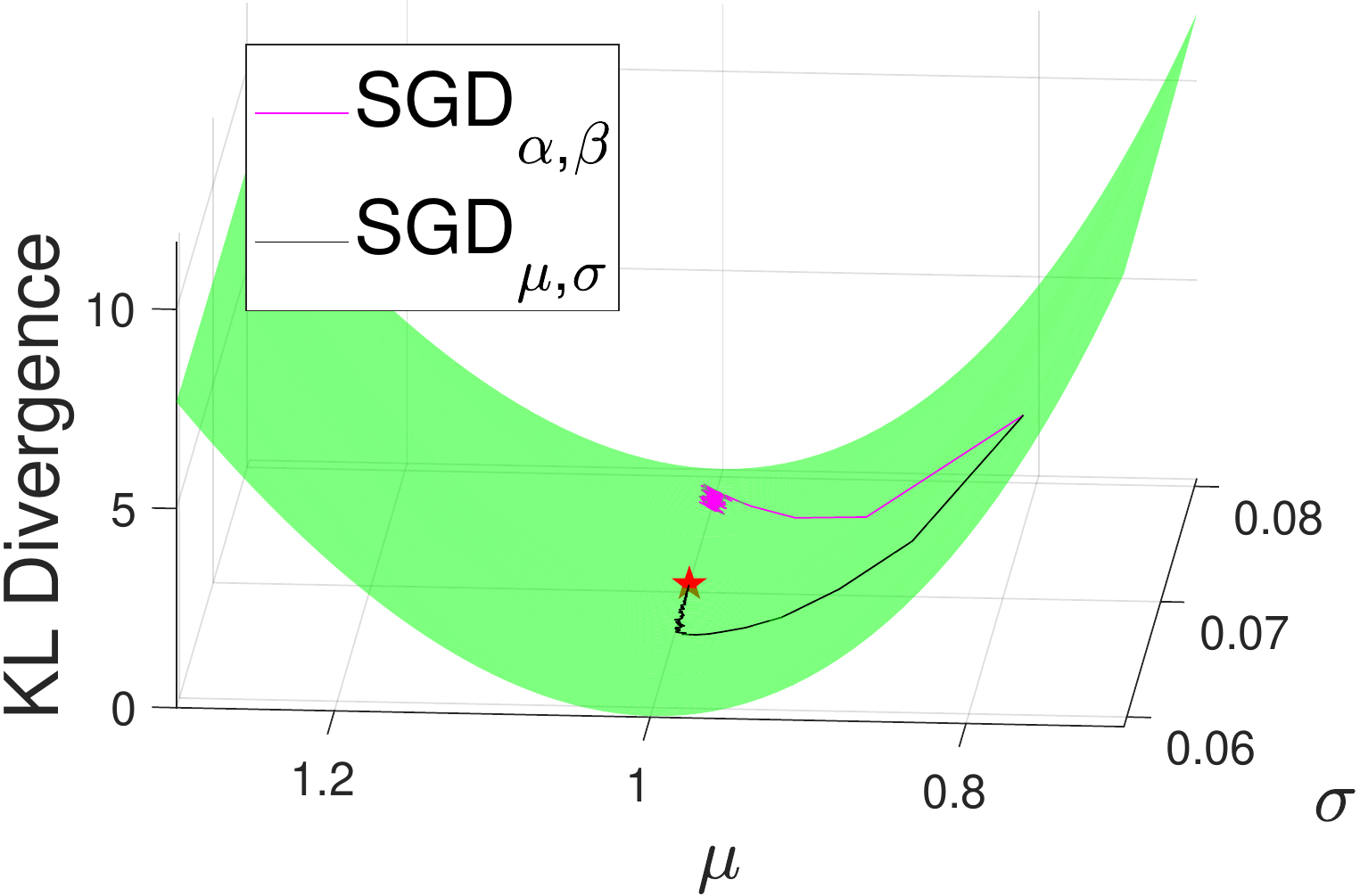}
	}
	\subfigure[] {\label{fig:Two_Gam_KLcurves_uvRep}
		\includegraphics[width=0.22\columnwidth]{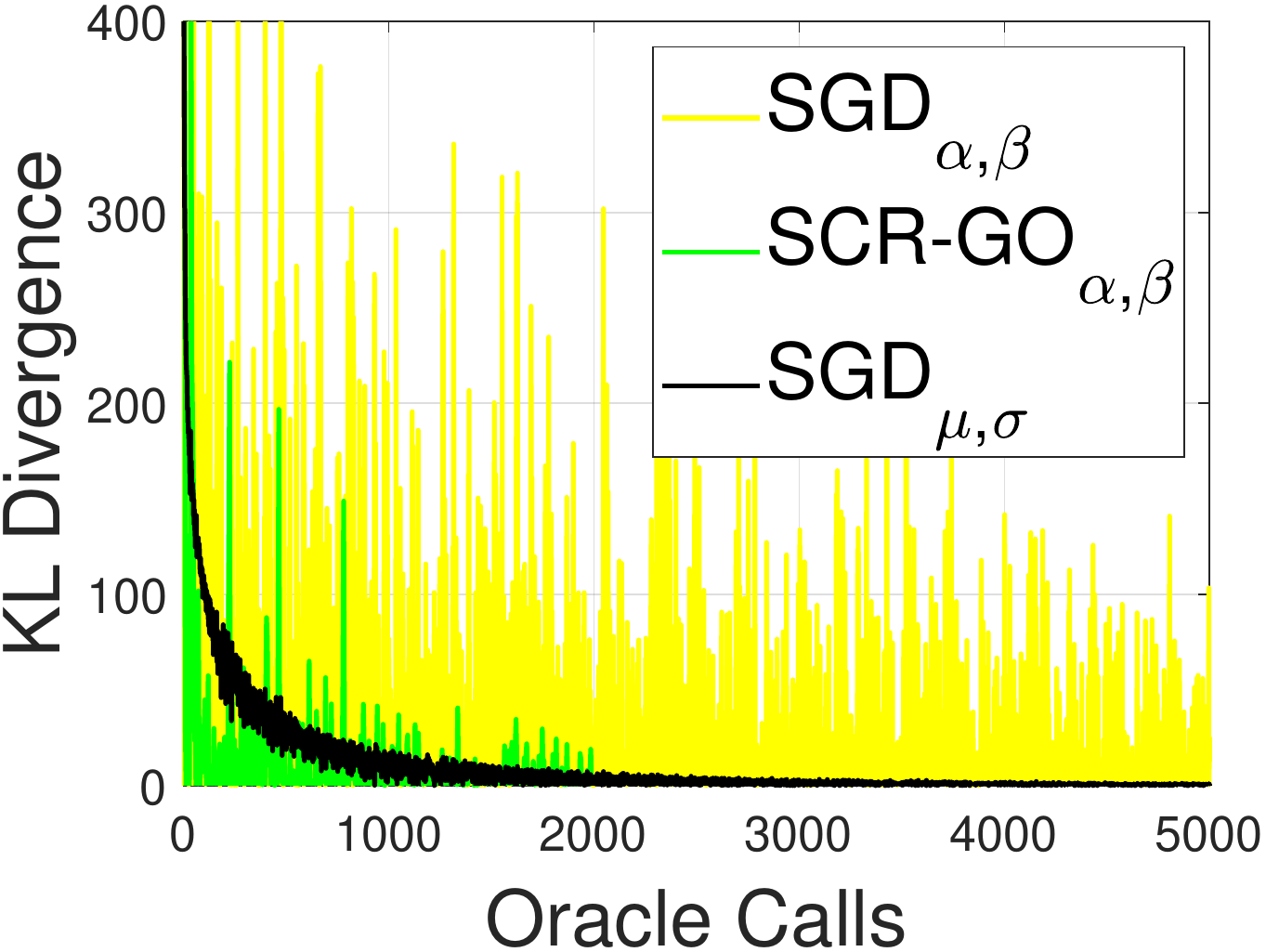}
	}
	\subfigure[] {\label{fig:Two_Gam_KLcurves_SCR}
		\includegraphics[width=0.22\columnwidth]{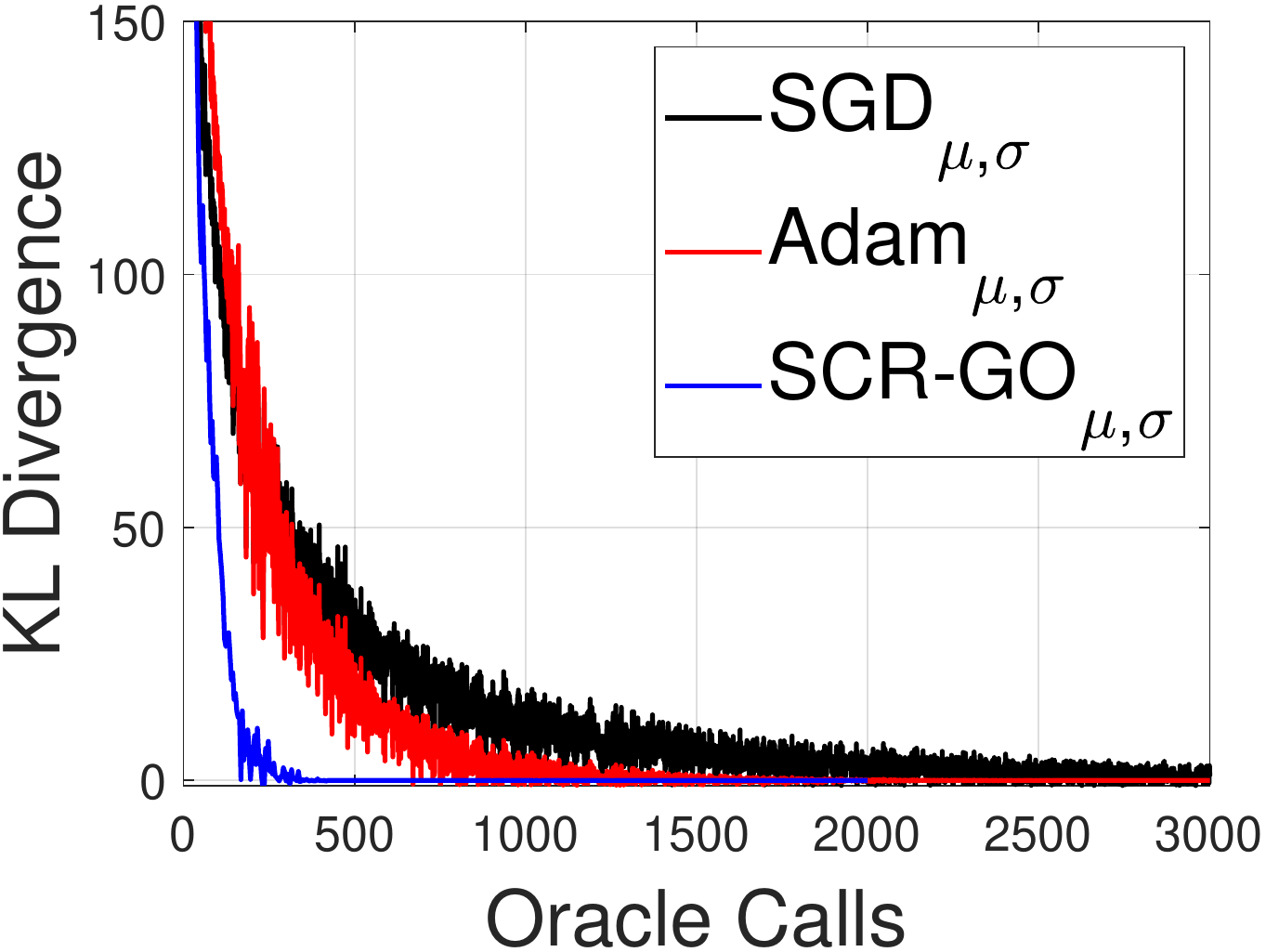}
	}
	\caption{
		(a-b) Demonstration of reverse KL divergences and SGD trajectories in different parameter spaces for Section \ref{sec:GOHess_gamma}.
		The subscript indicates the parameter space. 
		SGD$_{\alpha,\beta}$ leverages a learning rate of $10$, while the more efficient SGD$_{\mu,\sigma}$ only uses $10^{-3}$. 
		The red star denotes the optimum. 
		2,000 iterations are used to generate both trajectories, which are adapted to the other side for clear comparisons.
		(c-d) Training objectives versus the number of oracle calls for Section \ref{sec:KL_twoGamma_SCR}.
		The same curve of SGD$_{\mu, \sigma}$ is shown in both plots for clear comparisons. 
	}
	\label{fig:TwoGAM_RepKL}
	\vspace{-0.3 cm}
\end{figure}


The proposed techniques are verified with rigorous experiments where non-rep gamma RVs are of interest. 
Generalizing Section \ref{sec:GOHess_gamma}, we first test our SCR-GO on minimizing the reverse KL divergence between two gamma RVs.
Next, we consider mean-field variational inference for Poisson factor analysis (PFA; which is closely related to LDA \cite{Blei2003Latent}) \cite{Zhou2015Negative,Zhou2015Poisson}.
Finally concerning deep neural networks, the SCR-GO is tested on training variational encoders, mimicking the VAE \cite{kingma2014auto}, for PFA and its deep generalization of the Poisson gamma belief network (PGBN) \cite{zhou2016augmentable,cong2017deep}.

\textbf{Experimental settings}
We follow \cite{xu2017second,tripuraneni2018stochastic,yu2018third,roosta2018newton,kasai2018inexact} to show training objectives versus the number of oracle calls (calculations of gradients and/or HVPs); 
this is deemed a fair metric because it's independent of implementation-details/system-configurations and ideally an HVP can ``take about the same amount of computation as a gradient''\footnote{
This is the ideal situation.
However, it may not hold for our current implementation, which uses the $2$-backward technique in \eqref{eq:Hess_vec} and calculates special functions with a surrogate lib (see Appendix \ref{sec:Gam_Implement}).
That implementation also makes impossible fair comparisons wrt wall-clock time.
With our implementation/computer, in gamma-related experiments, a GO-HVP is about 3 times more expensive than a GO gradient.
} \cite{pearlmutter1994fast}.
We compare SCR-GO to standard SGD and the popular Adam \cite{kingma2014adam}.
For both SGD and Adam, learning rates from $\{0.001, 0.005, 0.01, 0.05, 0.1, 0.5, 1\}$ are tested with the best-tuned results shown.
Other settings are given in Appendix \ref{secapp:exp_settings}.

\vspace{-0.1 cm}
\subsection{Minimizing the reverse KL divergence between gamma RVs}
\label{sec:KL_twoGamma_SCR}
\vspace{-0.1 cm}

To demonstrate the effectiveness of the $\mu$-$\sigma$ reparameterization introduced in Section \ref{sec:GOHess_gamma} and the efficiency achieved from exploiting the curvature information via the GO Hessian, we first consider a simplified example, with the objective $\KL[\Gam(y; \alpha, \beta)||\Gam(y; 200, 1)]$, for better introduction. 
SGD, Adam, and our SCR-GO are compared within both $\alpha$-$\beta$ and $\mu$-$\sigma$ parameter spaces.

The training curves of the compared methods are given in Figures \ref{fig:Two_Gam_KLcurves_uvRep}-\ref{fig:Two_Gam_KLcurves_SCR}.
By comparing SGD$_{\alpha, \beta}$ with SGD$_{\mu, \sigma}$ in Figure \ref{fig:Two_Gam_KLcurves_uvRep}, it's clear that the $\mu$-$\sigma$ reparameterization method leads to a much faster convergence with smoother training curves, similar to those from deterministic optimization.
By contrast, SGD$_{\alpha, \beta}$ visits both high and low KL values frequently (bouncing within a valley bottom as shown in Figure \ref{fig:TwoGAM_RepKLsurf_ab}), with a much slower convergence to the optimum.
Thanks to the exploited curvature information, our SCR-GO$_{\alpha, \beta}$ shows a clearly improved convergence relative to SGD$_{\alpha, \beta}$.
Moving to the $\mu$-$\sigma$ space (see Figure \ref{fig:Two_Gam_KLcurves_SCR}), our SCR-GO$_{\mu, \sigma}$ delivers an even faster and more stabilized convergence than its counterpart SCR-GO$_{\alpha, \beta}$ and also SGD$_{\mu, \sigma}$ and Adam$_{\mu, \sigma}$, demonstrating the effectiveness of both the $\mu$-$\sigma$ reparameterization and the curvature exploitation via the GO Hessian.

\vspace{-0.1 cm}
\subsection{Mean-field variational inference for PFA}
\label{sec:MFVI_PFA}
\vspace{-0.1 cm}

For practical applications, we leverage the proposed techniques to develop efficient mean-field variational inference for the PFA, whose generative process is 
\beq
	p_{\thetav}(\xv, \zv): \xv \sim \Pois(\xv | \Wmat \zv), \zv \sim \Gam(\zv|\alphav_0, \betav_0),
\eeq
where $\xv$ is the count data variable, $\Wmat$ the topic matrix with each column/topic $\wv_k$ located in the simplex, \ie $w_{vk} \!>\! 0, \sum\nolimits_{v} w_{vk} \!=\! 1$, $\zv$ the latent code, and $\thetav\!=\!\{\Wmat,\alphav_0, \betav_0\}$. 
For mean-field variational inference, we assume variational approximation distribution
$
q_{\phiv}(\zv)\!=\! \Gam(\zv; \frac{\muv^2}{\sigmav^2}, \frac{\muv}{\sigmav^2})
$ with $\phiv \!=\! \{\muv, \sigmav\}$.
Accordingly given training dataset $\{\xv_1, \cdots, \xv_N\}$, the objective is to maximize  
$$
	\ELBO(\thetav, \{\phiv_i\}) = \frac{1}{N} \sum\nolimits_{i=1}^{N} \Ebb_{q_{\phiv_i}(\zv_i)} \Big[ \log \frac{p_{\thetav}(\xv_i, \zv_i)}{q_{\phiv_i}(\zv_i)} \Big].
$$
A well-tuned Adam optimizer and our SCR-GO are implemented for this experiment, with both training curves shown in Figure \ref{fig:MNIST_MF_PFA_ELBOcurves}.
It's clear that with the additional curvature information exploited via GO-HVP, our SCR-GO exhibits a faster convergence to a better local optimum, with a lower variance than the well-tuned Adam optimizer.

\subsection{Variational encoders for PFA and PGBN}
\label{sec:VAE_PFA}

\begin{figure}[tb]
	\setlength{\abovecaptionskip}{1.0pt}
	\centering
	\subfigure[] {\label{fig:MNIST_MF_PFA_ELBOcurves}
		\includegraphics[width=0.27\columnwidth]{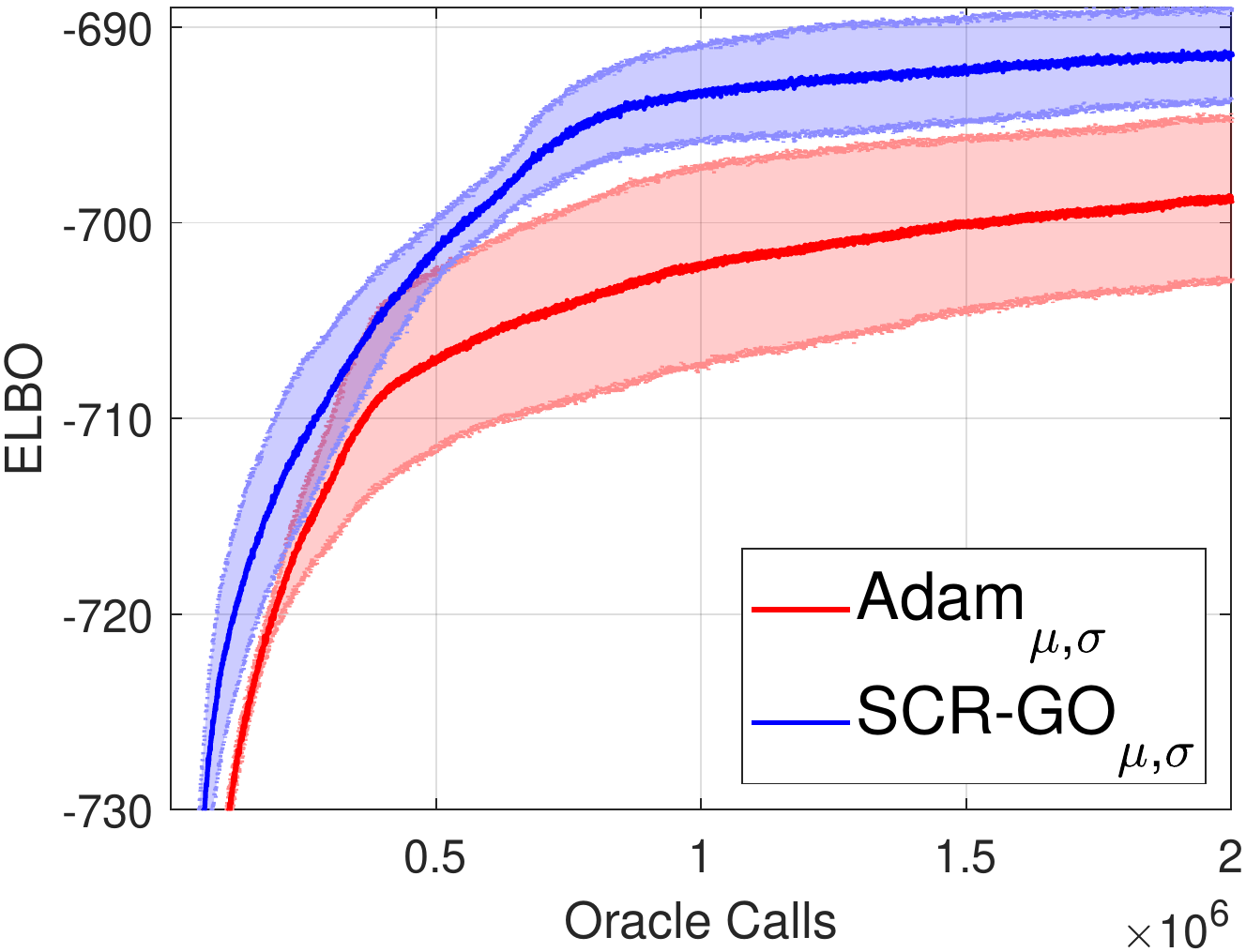}
	}
	\quad
	\subfigure[] {\label{fig:MNIST_PFA_JQNN_oracall}
		\includegraphics[width=0.27\columnwidth]{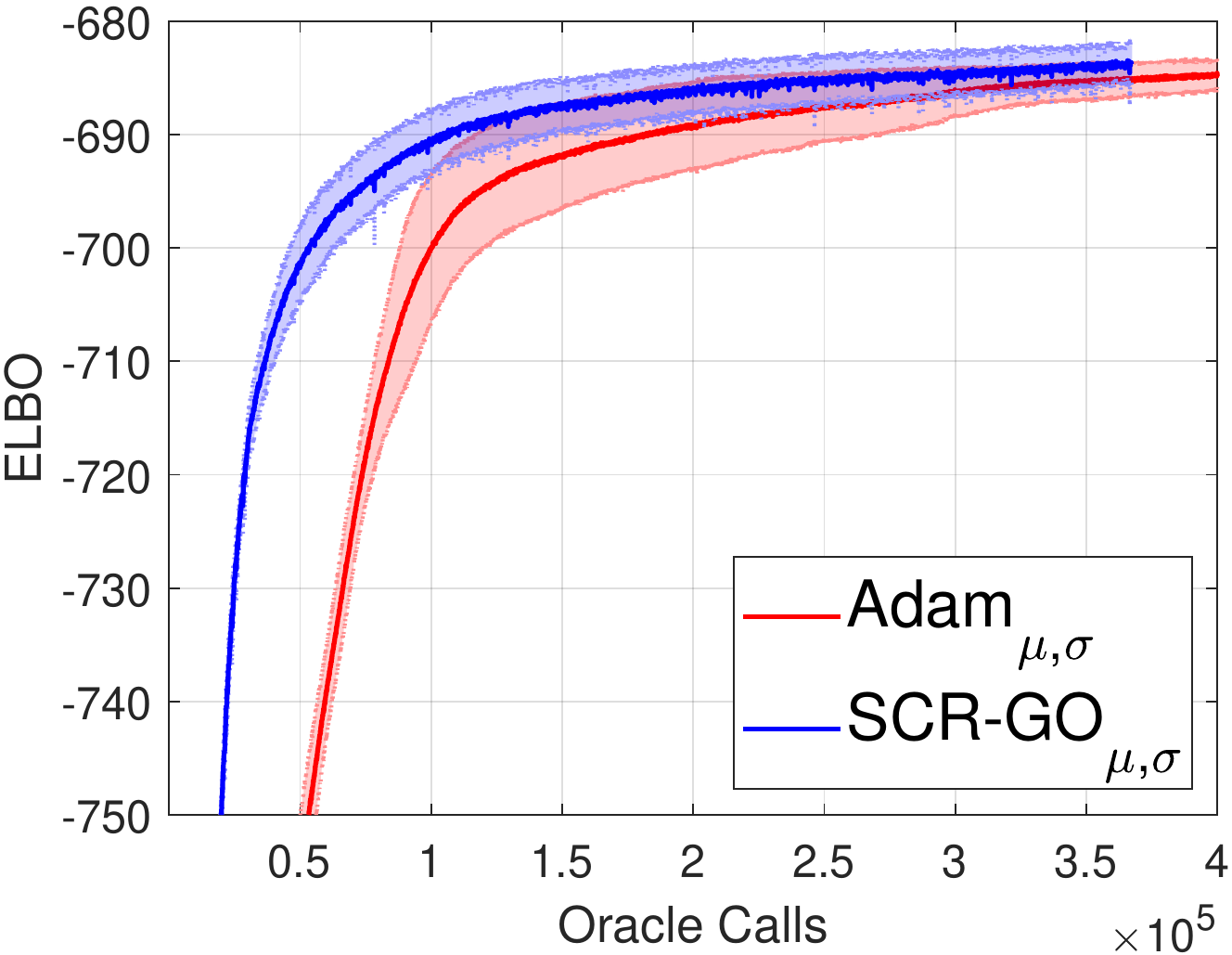}
	}
	\quad
	\subfigure[] {\label{fig:MNIST_PFA_JQNN_observ}
		\includegraphics[width=0.27\columnwidth]{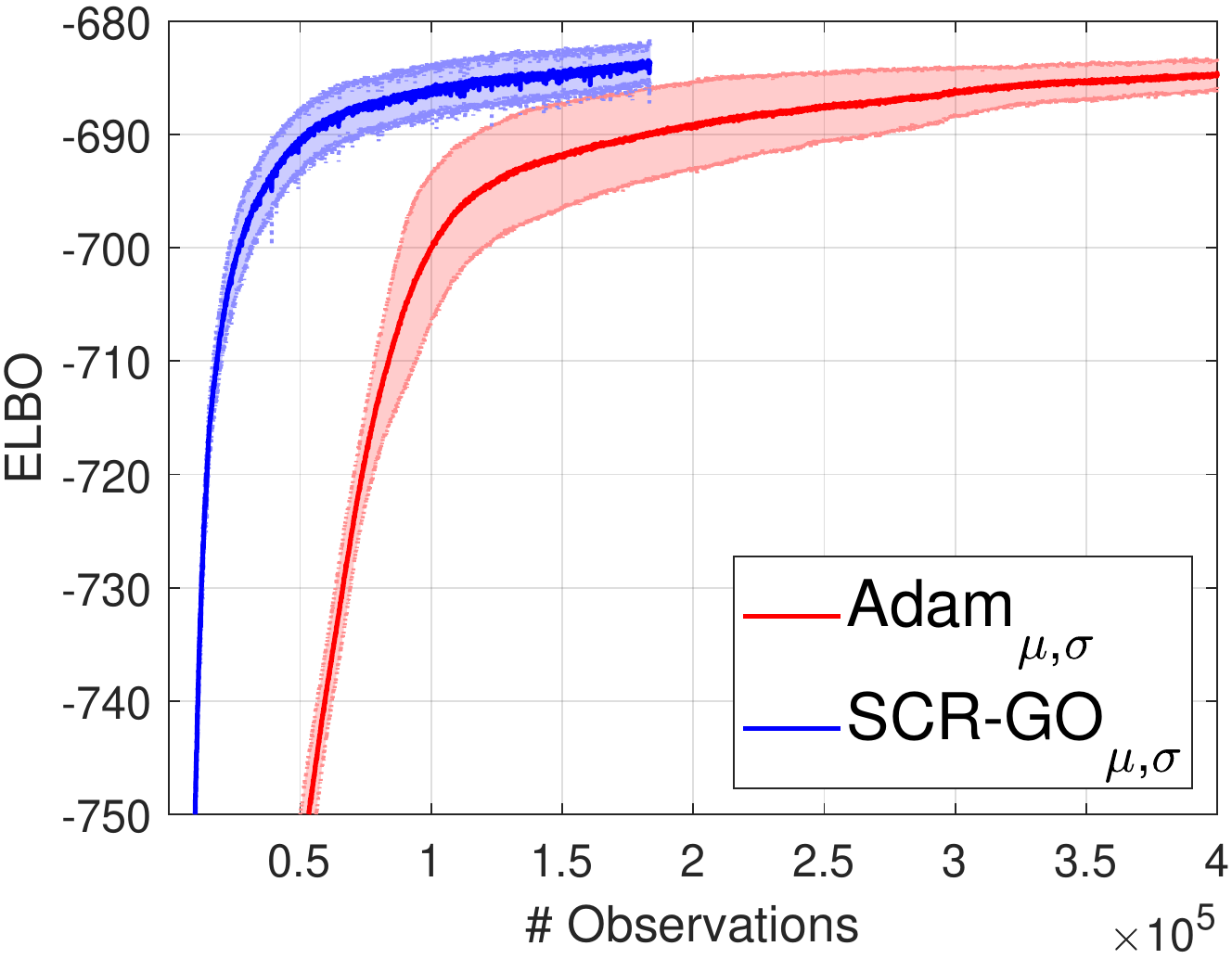}
	}
	\caption{
		Training curves of mean-field variational inference (a) and variational encoder (b-c) for PFA on MNIST.
		Variances are estimated based on $5$ random seeds.
	}\label{fig:MNIST_PFA_JQNN}
	\vspace{-0.3 cm}
\end{figure}

To test the effectiveness of the presented techniques when combined with deep neural networks, we consider developing a variational encoder for PFA mimicking the VAE, that is, 
\beq\label{eq:qz_gx}
q_{\phiv}(\zv | \xv) \!=\! \Gam(\zv; {\muv^2}/{\sigmav^2}, {\muv}/{\sigmav^2}),
\muv \!=\! \NN_{\muv}(\xv), \sigmav \!=\! \NN_{\sigmav}(\xv), 
\eeq
where $\NN(\cdot)$ denotes a neural network and $\phiv$ contains all the parameters of $\NN_{\muv}(\cdot)$ and $\NN_{\sigmav}(\cdot)$. Accordingly the objective is to maximize
$
	\ELBO(\thetav, \phiv) = \Ebb_{q_{\phiv}(\zv | \xv)} \big[ \log p_{\thetav}(\xv, \zv) - \log q_{\phiv}(\zv | \xv) \big].
$

Figures \ref{fig:MNIST_PFA_JQNN_oracall}-\ref{fig:MNIST_PFA_JQNN_observ} show the training objectives versus the number of oracle calls and processed observations. 
It's clear that the proposed SCR-GO performs better than a well-tuned Adam optimizer in terms of oracle calls and data efficiency, when applied to a model with deep neural networks.
The better performance of SCR-GO is attributed to its exploitation of the curvature information via the GO Hessian, which takes into consideration the correlation among parameters within $p_{\thetav}(\xv, \zv)$ and $q_{\phiv}(\zv | \xv)$ and utilizes an (implicit) adaptive learning rate mimicking the classical Newton's method.

For further testing under more challenging settings with a hierarchically-structured $q_{\gammav}(\yv)$ (see \eqref{eq:reverse_issue}), we consider developing a variational encoder for a $2$-layer PGBN. Specifically, with $\zv = \{\zv_1, \zv_2\}$,
\beq
\bali
& p_{\thetav}(\xv, \zv): \xv \sim \Pois(\xv | \Wmat_1 \zv_1), \zv_1 \sim \Gam(\zv_1|\Wmat_2 \zv_2, \cv_2), \zv_2 \sim \Gam(\zv_2|\alphav_0, \betav_0)
\\
& q_{\phiv}(\zv | \xv):  \zv_2 \sim q_{\phiv_2}(\zv_2 | \zv_1), \zv_1 \sim q_{\phiv_1}(\zv_1 | \xv),
\eali
\eeq
where $\thetav\!=\!\{\Wmat_1,\Wmat_2, \cv_2, \alphav_0, \betav_0\}$, $\phiv\!=\!\{\phiv_1,\phiv_2\}$, and both $q_{\phiv_2}(\zv_2 | \zv_1)$ and $q_{\phiv_1}(\zv_1 | \xv)$ are constructed as in \eqref{eq:qz_gx}. 
Due to space constraints, the experimental details and results are moved to Appendix \ref{secapp:VAE_PGBN}, where one observes similar plots as those in Figure \ref{fig:MNIST_PFA_JQNN}, confirming the effectiveness and efficiency of the presented techniques.

\section{Conclusions}

An unbiased low variance Hessian estimator, termed GO Hessian, is proposed to efficiently exploit curvature information for an expectation-based objective over a stochastic computation graph, with continuous rep/non-rep internal nodes and continuous/discrete leaves.
GO Hessian is easy-to-use with AD and HVP, enabling a low cost second-order optimization over high-dimensional parameters.
Based on the proposed GO Hessian, a new second-order optimization method is proposed for the expectation-based objective, which empirically performs better than a well-tuned Adam optimizer in challenging situations with non-rep gamma RVs.
A useful reparameterization is revealed for gamma RVs to make their optimization more friendly to gradient-based methods.

\section*{Broader Impact}

This work focuses on the fundamental research of an unbiased low-variance Hessian estimator for expectation-based objectives over stochastic computation graphs; accordingly, it does not present direct ethical or societal impact.
However, the proposed techniques may benefit many research fields, such as variational inference, generative models, or reinforcement learning, with easy-to-use curvature exploitation and better training/sample efficiency.

\begin{ack}
	
	We thank the anonymous reviewers for their constructive comments. 
	The research was supported by part by DARPA, DOE, NIH, NSF and ONR. The Titan Xp GPU used was donated by the NVIDIA Corporation. 
	
%
%
\end{ack}

{\small
	\bibliography{ReferencesCong}
	\bibliographystyle{plain}
	
}


\newpage
\appendix

\newgeometry{top=1in,bottom=1in,left=0.9in,right=0.9in} 
\setlength{\columnsep}{1.1cm} 
\twocolumn[ 
	\begin{center}
		{\large
			\textbf{
				Appendix of 
				GO Hessian for Expectation-Based Objectives
		}}
	
		\vspace{0.3 cm}
		\textbf{Yulai Cong, Miaoyun Zhao, Jianqiao Li, Junya Chen, Lawrence Carin
		\\Department of ECE, Duke University}
	\end{center}
	
	\vskip 0.3in
] 

\section{Naive derivations of the Hessian of Framework II}
\label{secapp:Naive_derivation_HessianII}

Below we reveal the challenge in deriving the Hessian of Framework II in \eqref{eq:reverse_issue} of the main manuscript.
Specifically, naive derivations fail to deliver an easy-to-use Monte Carlo (MC) estimator.

For simplicity, let's first employ the simplest single-layer continuous settings for presentation, where the continuous random variable (RV) $\yv \sim q_{\gammav} (\yv) = \prod\nolimits_{v} q_{\gammav} (y_v)$.

To compute the Hessian of the objective of Framework II, a naive method would (taking the two-dimensional case as example) 
\beq\label{appeq:second_derivative_example}
\resizebox{\hsize}{!}{$
	\bali
	& \nabla_{\gammav}^2 \Lc(\gammav) = \nabla_{\gammav}^2 \Ebb_{q_{\gammav} (\yv)} [f(\yv)] =
	\\
	& \int \left[ \bali
	& q_{\gammav} (y_1) \nabla_{\gammav}^2 q_{\gammav} (y_2) + 
	[\nabla_{\gammav} q_{\gammav} (y_1)] [\nabla_{\gammav} q_{\gammav} (y_2)]^T
	\\
	& + q_{\gammav} (y_2) \nabla_{\gammav}^2 q_{\gammav} (y_1) +
	[\nabla_{\gammav} q_{\gammav} (y_2)] [\nabla_{\gammav} q_{\gammav} (y_1)]^T
	\eali \right]
	f(\yv) d\yv.
	\eali$}
\eeq
It's almost impossible to directly estimate such a complicated expression with low variance without bias.

Alternatively, noticing that the GO gradient is derived based on the integration-by-parts \cite{cong2019go} and that the pathwise derivative originates from the transport equation \cite{figurnov2018implicit,jankowiak2018pathwise}, one may try to apply those foundations (the integration-by-parts/transport-equation) twice in a naive manner to get a Hessian estimator.
However, as detailed below, the resulting expressions are neither easy-to-use (amenable to auto-differentiation (AD) and Hessian-vector products (HVP)) in practice nor consistent with one's intuition.

As the GO gradient contains the pathwise derivative as a special case (see Section \ref{sec:pre_GO_gradient} of the main manuscript), we derive based on the integration-by-parts.
Thanks to the symmetry in \eqref{appeq:second_derivative_example}, we only apply the integration-by-parts twice to the first two terms as
\beq\label{appeq:NaiveHess11}
\resizebox{\hsize}{!}{$
	\bali
	\int [ q_{\gammav} (y_1) & \nabla_{\gammav}^2 q_{\gammav} (y_2) ]
	f(\yv) d\yv 
	= -\int q_{\gammav} (y_1) [ \nabla_{\gammav}^2 Q_{\gammav} (y_2) ]
	[ \nabla_{y_2} f(\yv) ] d\yv 
	\\
	& = \Ebb_{q_{\gammav} (\yv)} \bigg[\frac{-\nabla_{\gammav}^2 Q_{\gammav} (y_2)}{q_{\gammav} (y_2)} \nabla_{y_2} f(\yv) \bigg]
	\\
	\text{OR} & = \int q_{\gammav} (y_1) \bigg[ \nabla_{\gammav}^2 \int_{-\infty}^{y_2}  Q_{\gammav} (\hat y_2)     d\hat y_2 \bigg]
	[ \nabla_{y_2}^2 f(\yv) ] d\yv ,
	\eali
	$}
\eeq
and
\beq\label{appeq:NaiveHess12}
\resizebox{\hsize}{!}{$
	\bali
	\int [ \nabla_{\gammav} q_{\gammav} (y_1)]  &  [\nabla_{\gammav} q_{\gammav} (y_2) ]^T   f(\yv)   d\yv 
	= -\int [ \nabla_{\gammav} q_{\gammav} (y_1)]  [\nabla_{\gammav} Q_{\gammav} (y_2) ]^T   [\nabla_{y_2} f(\yv)]    d\yv 
	\\
	& = \int [ \nabla_{\gammav} Q_{\gammav} (y_1)]  [\nabla_{\gammav} Q_{\gammav} (y_2) ]^T   [\nabla_{y_2 y_1}^2 f(\yv)]    d\yv 
	\\
	& = \Ebb_{q_{\gammav} (\yv)} \bigg[   \frac{-\nabla_{\gammav} Q_{\gammav} (y_1)}{q_{\gammav} (y_1)}    \Big[    \frac{-\nabla_{\gammav} Q_{\gammav} (y_2)}{q_{\gammav} (y_2)}  \Big]^T   \nabla_{y_2 y_1}^2 f(\yv) \bigg].
	\eali
	$}
\eeq
Substituting the above results to \eqref{appeq:second_derivative_example} and leveraging the symmetry, we yield
\beq\label{appeq:naiveHess1}
\resizebox{\hsize}{!}{$
	\bali
	\nabla_{\gammav}^2 & \Lc(\gammav) 
	= \Ebb_{q_{\gammav} (\yv)} \left[ \bali
	& \frac{-\nabla_{\gammav}^2 Q_{\gammav} (y_2)}{q_{\gammav} (y_2)} \nabla_{y_2} f(\yv) 
	\\
	& + \frac{-\nabla_{\gammav} Q_{\gammav} (y_1)}{q_{\gammav} (y_1)}    \Big[    \frac{-\nabla_{\gammav} Q_{\gammav} (y_2)}{q_{\gammav} (y_2)}  \Big]^T   \nabla_{y_2 y_1}^2 f(\yv)
	\\ 
	& + \frac{-\nabla_{\gammav}^2 Q_{\gammav} (y_1)}{q_{\gammav} (y_1)} \nabla_{y_1} f(\yv) 
	\\
	& + \frac{-\nabla_{\gammav} Q_{\gammav} (y_2)}{q_{\gammav} (y_2)}    \Big[    \frac{-\nabla_{\gammav} Q_{\gammav} (y_1)}{q_{\gammav} (y_1)}  \Big]^T   \nabla_{y_1 y_2}^2 f(\yv)
	\eali \right]
	\eali$}
\eeq
or
\beq\label{appeq:naiveHess2}
\resizebox{\hsize}{!}{$
	\bali
	\nabla_{\gammav}^2 & \Lc(\gammav) 
	= \Ebb_{q_{\gammav} (\yv)} \left[ \bali
	& \frac{\nabla_{\gammav}^2 \int_{-\infty}^{y_2}  Q_{\gammav} (\hat y_2) d\hat y_2}{q_{\gammav} (y_2)} \nabla^2_{y_2} f(\yv) 
	\\
	& + \frac{-\nabla_{\gammav} Q_{\gammav} (y_1)}{q_{\gammav} (y_1)}    \Big[    \frac{-\nabla_{\gammav} Q_{\gammav} (y_2)}{q_{\gammav} (y_2)}  \Big]^T   \nabla_{y_2 y_1}^2 f(\yv)
	\\ 
	& + \frac{\nabla_{\gammav}^2 \int_{-\infty}^{y_1}  Q_{\gammav} (\hat y_1) d\hat y_1}{q_{\gammav} (y_1)} \nabla^2_{y_1} f(\yv) 
	\\
	& + \frac{-\nabla_{\gammav} Q_{\gammav} (y_2)}{q_{\gammav} (y_2)}    \Big[    \frac{-\nabla_{\gammav} Q_{\gammav} (y_1)}{q_{\gammav} (y_1)}  \Big]^T   \nabla_{y_1 y_2}^2 f(\yv)
	\eali \right].
	\eali$}
\eeq

Generalizing the above two equations for a multi-dimensional $\yv$ (still in the single-layer continuous settings), it's apparent that one won't achieve an easy-to-use Hessian estimator that is amenable to AD and HVP, because of either the combination of $\{\nabla_{\yv} f(\yv), \nabla_{\yv}^2 f(\yv)\}$ or the complicated integral $\nabla_{\gammav}^2 \int_{-\infty}^{y_i}  Q_{\gammav} (\hat y_i) d\hat y_i$.

Even in the simplest single-layer continuous settings, naive derivations fail to deliver an easy-to-use Hessian estimator. 
For more complicated settings with stochastic computation graphs, various kinds of conditional structures therein would make it extremely hard (if not impossible) to even derive expressions as \eqref{appeq:naiveHess1}/\eqref{appeq:naiveHess2} (refer to the derivations in Section \ref{secapp:derivation_GO_Hessian}), not to speak of an easy-to-use implementation that is amenable to AD and HVP.

By contrast, our GO Hessian has the clear advantage of being intuitively simple and easy-to-use in practice, \ie amenable to AD and HVP.

\section{Derivations of the GO Hessian}
\label{secapp:derivation_GO_Hessian}

Recall that the objective of Framework II is 
\beq\label{eq:FrameII_app}
\text{Framework II:  }
\min_{\gammav} \Lc(\gammav) \triangleq \Ebb_{q_{\gammav}(\yv)} [f(\yv)].
\eeq
Below we derive the GO Hessian for ($i$) where the RV $\yv$ is single-layer and continuous, ($ii$) where $\yv$ is single-layer/leaf and discrete, and ($iii$) where $\yv$ denotes hierarchically constructed stochastic computation graphs, with continuous internal RVs and continuous/discrete leaf RVs.

\subsection{GO Hessian in single-layer continuous settings}

Consider a single-layer continuous RV $\yv \sim q_{\gammav} (\yv) = \prod\nolimits_{v} q_{\gammav} (y_v)$. 
With $\gamma_a$ denoting the $a$-th element of parameters $\gammav$, the GO gradient \cite{cong2019go} for \eqref{eq:FrameII_app} is 
\beq\label{eq:GOgrad_onedim}
\resizebox{\hsize}{!}{$
\bali
& \nabla_{\gamma_a} \Lc(\gammav) 
= \nabla_{\gamma_a} \Ebb_{q_{\gammav}(\yv)} [f(\yv)]
= \Ebb_{q_{\gammav} (\yv)} \big[
\Gmat_{\gamma_a}^{q_{\gammav} (\yv)}
\nabla_{\yv} f(\yv)
\big]
\\
& = \Ebb_{q_{\gammav} (\yv)} \big[
\sum\nolimits_{v} g_{\gamma_a}^{q_{\gammav} (y_v)}
\nabla_{y_v} f(\yv) 
\big]
= \Ebb_{q_{\gammav} (\yv)} \big[ F(\yv, \gammav) \big],
\eali
$}
\eeq
where $F(\yv, \gammav) \triangleq \sum\nolimits_{v} g_{\gamma_a}^{q_{\gammav} (y_v)} \nabla_{y_v} f(\yv) $.

Noticing ($i$) that the GO gradient empirically shows low variance and often works well with only one sample, ($ii$) that the \emph{variable-nabla} $g_{\gamma_a}^{q_{\gammav} (y_v)}$ is differentiable with often a simple expression (see Table 3 of \cite{cong2019go}; this originates from that the denominator of the \emph{variable-nabla} is likely to be canceled due to the Leibniz integral rule.\footnote{If $y_v$ is rep, that denominator is canceled exactly, reducing GO gradient to the Rep (see Lemma 1 of \cite{cong2019go})}), and ($iii$) the GO gradient has a similar expression as the original objective $\Lc(\gammav)$, we apply the GO gradient to \eqref{eq:GOgrad_onedim} again to yield
\beq\label{eq:GOHess_1continuous_app}
\resizebox{\hsize}{!}{$
\bali
& \nabla^2_{\gamma_b\gamma_a} \Lc(\gammav) 
= \nabla_{\gamma_b} \Ebb_{q_{\gammav} (\yv)} \big[ F(\yv, \gammav) \big]
\\
& = \Ebb_{q_{\gammav} (\yv)} \Big[
\Gmat_{\gamma_b}^{q_{\gammav} (\yv)}
\nabla_{\yv} F(\yv, \gammav) 
+ \nabla_{\gamma_b} F(\yv, \gammav)
\Big]
\\
& = \Ebb_{q_{\gammav} (\yv)} \left[
\bali
& \sum\nolimits_{k} g_{\gamma_b}^{q_{\gammav} (y_k)}
\nabla_{y_k} \sum\nolimits_{v} g_{\gamma_a}^{q_{\gammav} (y_v)} \nabla_{y_v} f(\yv)
\\
& + \nabla_{\gamma_b} \sum\nolimits_{v} g_{\gamma_a}^{q_{\gammav} (y_v)} \nabla_{y_v} f(\yv)
\eali \right]
\\
& = \Ebb_{q_{\gammav} (\yv)} 
\left[ \bali
	& \sum\nolimits_{k} g_{\gamma_b}^{q_{\gammav} (y_k)}
	\left[ \bali
		& [\nabla_{y_k} g_{\gamma_a}^{q_{\gammav} (y_k)}] \nabla_{y_k} f(\yv) + 
		\\
		& \sum\nolimits_{v} g_{\gamma_a}^{q_{\gammav} (y_v)} \nabla^2_{y_k y_v} f(\yv) 
	\eali \right]
	\\ 
	& + \sum\nolimits_{v} [\nabla_{\gamma_b} g_{\gamma_a}^{q_{\gammav} (y_v)}] \nabla_{y_v} f(\yv)
\eali \right]
\\
& = \Ebb_{q_{\gammav} (\yv)} 
\left[ \bali
& \Gmat_{\gamma_b}^{q_{\gammav}(\yv)} [\nabla^2_{\yv} f(\yv)] \Gmat_{\gamma_a}^{q_{\gammav}(\yv)} {}^T + 
\\
& \sum\nolimits_{v} 
\big[ 
	g_{\gamma_b}^{q_{\gammav} (y_v)} \nabla_{y_v} g_{\gamma_a}^{q_{\gammav} (y_v)} 
	+ \nabla_{\gamma_b} g_{\gamma_a}^{q_{\gammav} (y_v)}
\big]
\nabla_{y_v} f(\yv)
\eali \right]
\eali
$}
\eeq
By rewriting the above elemental definition into its vector/matrix form, we yield \eqref{eq:GO_Hessian_1} of the main manuscript, \ie
\beq\label{eq:GOHess_1continuous_F_app}
\resizebox{\hsize}{!}{$
\bali
& \nabla^2_{\gammav} \Lc(\gammav) = \nabla^2_{\gammav} \Ebb_{q_{\gammav}(\yv)} [f(\yv)]
\\
& = \Ebb_{q_{\gammav} (\yv)} 
\left[
	\Gmat_{\gammav}^{q_{\gammav}(\yv)} [\nabla^2_{\yv} f(\yv)] \Gmat_{\gammav}^{q_{\gammav}(\yv)} {}^T
	+ \Hten_{\gammav\gammav}^{q_{\gammav}(\yv)} \nabla_{\yv} f(\yv)
\right],
\eali
$}
\eeq
where $\Hten_{\gammav\gammav}^{q_{\gammav}(\yv)}$ is a $3$-D tensor with its element
\beq
\resizebox{\hsize}{!}{$
[\Hten_{\gammav\gammav}^{q_{\gammav}(\yv)}]_{b,a,v} = g_{\gamma_b}^{q_{\gammav} (y_v)} \nabla_{y_v} g_{\gamma_a}^{q_{\gammav} (y_v)} + \nabla_{\gamma_b} g_{\gamma_a}^{q_{\gammav} (y_v)} \triangleq  h_{\gamma_b \gamma_a}^{q_{\gammav}(y_v)}
$}
\eeq 
and the tensor-vector product $\Hten\av$ outputs a matrix whose elements $[\Hten\av]_{b,a} = \sum\nolimits_{v} \Hten_{b,a,v} \av_{v}$.
We name $h_{\gamma_b \gamma_a}^{q_{\gammav}(y_v)}$ the \emph{variable-hess}, because of its intuitive meaning of the second-order ``derivative'' of a RV $y_v$ wrt parameters $\{\gamma_a,\gamma_b\}$.

\subsection{GO Hessian in single-layer/leaf discrete settings}

Next we derive the GO Hessian in single-layer or leaf discrete settings, where $\yv$ has discrete components $y_v$. 
Recall that the GO gradient \cite{cong2019go} for \eqref{eq:FrameII_app} is now defined as 
\beq\label{eq:}
\resizebox{\hsize}{!}{$
	\bali
	\nabla_{\gamma_a} \Lc(\gammav) 
	& = \nabla_{\gamma_a} \Ebb_{q_{\gammav}(\yv)} [f(\yv)]
	= \Ebb_{q_{\gammav} (\yv)} \big[
	\Gmat_{\gamma_a}^{q_{\gammav} (\yv)}
	\Dbb_{\yv} f(\yv)
	\big]
	\\
	& = \Ebb_{q_{\gammav} (\yv)} \big[
	\sum\nolimits_{v} g_{\gamma_a}^{q_{\gammav} (y_v)}
	\Dbb_{y_v} f(\yv) 
	\big],
	\eali
	$}
\eeq
where $\Dbb_{y_v} f(\yv) \triangleq f(\yv^{v+})-f(\yv)$ with $\yv^{v+} \triangleq [\cdots,y_{v\!-\!1},y_{v}+1,y_{v\!+\!1},\cdots]^T$.

Similarly, based on the GO gradient, we yield 
\beq\label{eq:GOHess_1discrete_app}
\resizebox{\hsize}{!}{$
\bali
	& \nabla^2_{\gamma_b\gamma_a} \Lc(\gammav) 
	= \nabla_{\gamma_b} \Ebb_{q_{\gammav} (\yv)} \big[
		\sum\nolimits_{v} g_{\gamma_a}^{q_{\gammav} (y_v)} 
		\Dbb_{y_v} f(\yv) \big] 
	\\
	& = \Ebb_{q_{\gammav} (\yv)} \left[ \bali
		& \sum\nolimits_{k} g_{\gamma_b}^{q_{\gammav} (y_k)} \Dbb_{y_k} \left[ 
		\sum\nolimits_{v} g_{\gamma_a}^{q_{\gammav} (y_v)} \Dbb_{y_v} f(\yv) 
		\right] 
		\\
		& + \nabla_{\gamma_b} \sum\nolimits_{v} g_{\gamma_a}^{q_{\gammav} (y_v)} \Dbb_{y_v} f(\yv)
	\eali \right] \\
	& = \Ebb_{q_{\gammav} (\yv)} \left[ \bali
		& \sum\nolimits_{k} g_{\gamma_b}^{q_{\gammav} (y_k)} 
		\left[ \bali
		& \sum\nolimits_{v \neq k} g_{\gamma_a}^{q_{\gammav} (y_v)} \Dbb^2_{y_k y_v} f(\yv)
		\\
		& + \Dbb_{y_k} \big[ g_{\gamma_a}^{q_{\gammav} (y_k)} \Dbb_{y_k} f(\yv) \big]
		\eali \right] \\ 
		& 
		+ \sum\nolimits_{v} [\nabla_{\gamma_b} g_{\gamma_a}^{q_{\gammav} (y_v)}] \Dbb_{y_v} f(\yv) 
		\eali \right] 
	\\
	& = \Ebb_{q_{\gammav} (\yv)} \left[ \bali
		& \sum\nolimits_{k} g_{\gamma_b}^{q_{\gammav} (y_k)} 
		\left[ \bali
			& \sum\nolimits_{v} g_{\gamma_a}^{q_{\gammav} (y_v)} \Dbb^2_{y_k y_v} f(\yv)
			\\
			& + [\Dbb_{y_k} g_{\gamma_a}^{q_{\gammav} (y_k)}] \Dbb_{y_k} f(\yv^{k+})
		\eali \right] \\ 
		& 
		+ \sum\nolimits_{v} [\nabla_{\gamma_b} g_{\gamma_a}^{q_{\gammav} (y_v)}] \Dbb_{y_v} f(\yv) 
		\eali \right]
	\\
	& = \Ebb_{q_{\gammav} (\yv)} \left[ \bali
		& \Gmat_{\gamma_b}^{q_{\gammav}(\yv)} [\Dbb^2_{\yv} f(\yv)] \Gmat_{\gamma_a}^{q_{\gammav}(\yv)} {}^T + 
		\\
		& + \sum\nolimits_{v} \left[ \bali
			& g_{\gamma_b}^{q_{\gammav} (y_v)} [\Dbb_{y_v} g_{\gamma_a}^{q_{\gammav} (y_v)}]
			\Dbb_{y_v} f(\yv^{v+})
			\\ 
			& + [\nabla_{\gamma_b} g_{\gamma_a}^{q_{\gammav} (y_v)}] \Dbb_{y_v} f(\yv) 
		\eali \right] 
	\eali \right],
\eali
$}
\eeq
which is quite similar to that for continuous $\yv$ (see \eqref{eq:GOHess_1continuous_app}). The slight difference originates from that, for discrete $x$,
\beq
\bali
& \Dbb_x [f(x)g(x)] = f(x+1)g(x+1) - f(x)g(x)
\\
& = f(x+1)g(x+1) - f(x+1)g(x) 
\\
& \quad + f(x+1)g(x) - f(x)g(x)
\\
& = f(x+1) \Dbb_{x}g(x) + g(x) \Dbb_{x} f(x)
\\
& \ne f(x) \Dbb_{x}g(x) + g(x) \Dbb_{x} f(x),
\eali
\eeq
which results in  
\beq
\resizebox{\hsize}{!}{$\bali
	\Dbb_{y_k} & \big[ g_{\gamma_a}^{q_{\gammav} (y_k)} \Dbb_{y_k} f(\yv) \big]
	\\
	= & \big[ \Dbb_{y_k} f(\yv) \big]_{y_k=y_k+1} [\Dbb_{y_k} g_{\gamma_a}^{q_{\gammav} (y_k)}] 
	+ g_{\gamma_a}^{q_{\gammav} (y_k)} \Dbb_{y_k} [ \Dbb_{y_k} f(\yv) ]
	\\ 
	= & [\Dbb_{y_k} f(\yv^{k+})] [\Dbb_{y_k} g_{\gamma_a}^{q_{\gammav} (y_k)}] 
	+ g_{\gamma_a}^{q_{\gammav} (y_k)} \Dbb^2_{y_k} f(\yv).
\eali$}
\eeq

Rewriting equation \eqref{eq:GOHess_1discrete_app} into the vector/matrix form, we have \eqref{eq:GO_Hessian_12} of the main manuscript as 
\beq\label{eq:GOHess_1discrete_F_app}
\resizebox{\hsize}{!}{$\bali
& \nabla_{\gammav}^2 \Lc(\gammav) = 
\nabla_{\gammav} \Ebb_{q_{\gammav} (\yv)} \big[
\Gmat_{\gammav}^{q_{\gammav} (\yv)}
\Dbb_{\yv} f(\yv) 
\big]
\\
& = \Ebb_{q_{\gammav} (\yv)} \big[
\Gmat_{\gammav}^{q_{\gammav}(\yv)} [\Dbb_{\yv}^2 f(\yv)] \Gmat_{\gammav}^{q_{\gammav}(\yv)} {}^T
+ \overline{\Hten_{\gammav\gammav}^{q_{\gammav}(\yv)} \Dbb_{\yv} f(\yv)}
\big],
\eali$}
\eeq
where $\scriptstyle \overline{\Hten_{\gammav\gammav}^{q_{\gammav}(\yv)} \Dbb_{\yv} f(\yv)}$ represents a matrix with its elements $$
\resizebox{\hsize}{!}{$
	\big[{\scriptstyle{\overline{\Hten_{\gammav\gammav}^{q_{\gammav}(\yv)} \Dbb_{\yv} f(\yv)}}} \big]_{b,a} = \sum_{v} \left[ \bali
	& [g_{\gamma_b}^{q_{\gammav} (y_v)} \Dbb_{y_v} g_{\gamma_a}^{q_{\gammav} (y_v)}] \Dbb_{y_v} f(\yv^{v+}) 
	\\
	& + [\nabla_{\gamma_b} g_{\gamma_a}^{q_{\gammav} (y_v)}] \Dbb_{y_v} f(\yv) 
	\eali \right].
	$}
$$

In fact, there is an alternative expression that also delivers an unbiased estimation of $\nabla_{\gammav}^2 \Lc(\gammav)$ for discrete RVs, because 
\beq
\resizebox{\hsize}{!}{$\bali
	\Dbb_{y_k} & \big[ g_{\gamma_a}^{q_{\gammav} (y_k)} \Dbb_{y_k} f(\yv) \big]
	\\
	= & [ g_{\gamma_a}^{q_{\gammav} (y_k)}]_{y_k=y_k+1} \Dbb_{y_k} [ \Dbb_{y_k} f(\yv) ]
	+ [\Dbb_{y_k} f(\yv)] [\Dbb_{y_k} g_{\gamma_a}^{q_{\gammav} (y_k)}] 
	\\ 
	= & [ g_{\gamma_a}^{q_{\gammav} (y_k)}]_{y_k=y_k+1} \Dbb^2_{y_k} f(\yv)
	+ [\Dbb_{y_k} g_{\gamma_a}^{q_{\gammav} (y_k)}] [\Dbb_{y_k} f(\yv)].
	\eali$}
\eeq
In practice, which one to choose may be dependent on the ease of implementation.
We leave that for future research.


Note the Abel transformation (used to derive the GO gradient for discrete RVs \cite{cong2019go}) is quite similar to the integration-by-parts (foundation of the GO gradient for continuous RVs). 
It's highly possible that the GO gradient for discrete RVs may share properties of the Reparameterization gradient (a special case of the GO gradient for continuous reparameterizable RVs), such as the variance characteristics discussed in Section D.1 of \cite{rezende2014stochastic} or the relationship between smoothness and MC variance discussed in Figure 3 of \cite{mohamed2019monte}.
Rigorous discussions are left for future research.

\subsection{GO Hessian for stochastic computation graphs with continuous internal nodes and continuous/discrete leaves}
\label{appsec:GOHess_SCG}

For better understanding of the derivations, we proceed by first talking about the two-layer simplified settings to introduce main patterns of our GO Hessian; we then leverage those patterns to generalize the derivations to deep settings with stochastic computation graphs (with continuous rep/non-rep internal nodes and continuous/discrete leaves).

\subsubsection{Two-layer settings}
\label{secapp:two_layer_GOHessian}

Assume $q_{\gammav} (\yv) = q_{\gammav_1} (\yv_1) q_{\gammav_2} (\yv_2 | \yv_1)$ with $\yv = \{\yv_1, \yv_2\}$, $\gammav = \{\gammav_1, \gammav_2\}$, and conditional independence, \ie $q_{\gammav_1} (\yv_1) = \prod\nolimits_{v} q_{\gammav_1} (y_v)$ and $q_{\gammav_2} (\yv_2 | \yv_1) = \prod\nolimits_{v'} q_{\gammav_2} (y_{v'} | \yv_1)$. 
The internal RV $\yv_1$ is continuous, while the leaf RV $\yv_2$ could be either continuous or discrete.

For simpler derivations, we also assume continuous leaf RV $\yv_2$ below, because the expression for where with discrete leaves is similar (see \eqref{eq:GOHess_1continuous_F_app} and \eqref{eq:GOHess_1discrete_F_app}) and the generalization is straightforward.
The objective is 
\beq\label{eq:frameII_2layer_app}
\Lc(\gammav) = \Ebb_{q_{\gammav_1}(\yv_1)q_{\gammav_2}(\yv_2 | \yv_1)} [f(\yv_1, \yv_2)].
\eeq

Following the above derivations in single-layer settings, it's straight-forward to show the expressions related to the continuous leave RV $\yv_2$, \ie
\beq\label{eq:GOHess_2layer_gamma22_app}
\resizebox{\hsize}{!}{$\bali
& \nabla_{\gammav_2} \Lc(\gammav) = \Ebb_{q_{\gammav_1} (\yv_1)} \big[ \nabla_{\gammav_2} \Ebb_{q_{\gammav_2} (\yv_2 | \yv_1)} [f(\yv_1, \yv_2)] \big]
\\
& \quad = \Ebb_{q_{\gammav} (\yv)} \big[
\Gmat_{\gammav_2}^{q_{\gammav_2} (\yv_2 | \yv_1)}
\nabla_{\yv_2} f(\yv_1, \yv_2) \big]
\\
& \nabla^2_{\gammav_2} \Lc(\gammav) 
= \Ebb_{q_{\gammav_1} (\yv_1)} \left[ \nabla^2_{\gammav_2} \Ebb_{q_{\gammav_2} (\yv_2 | \yv_1)} [f(\yv_1, \yv_2)] \right]
\\
& \quad = \Ebb_{q_{\gammav} (\yv)} \left[\bali
	& \Gmat_{\gammav_2}^{q_{\gammav_2} (\yv_2 | \yv_1)} 
	[\nabla^2_{\yv_2} f(\yv_1, \yv_2)] 
	\Gmat_{\gammav_2}^{q_{\gammav_2} (\yv_2 | \yv_1)} {}^T
	\\
	& + \Hten_{\gammav_2 \gammav_2}^{q_{\gammav_2} (\yv_2 | \yv_1)} \nabla_{\yv_2} f(\yv_1, \yv_2)
	\eali\right]
\eali$}
\eeq
One can readily get the expression for discrete $\yv_2$ by comparing \eqref{eq:GOHess_1continuous_F_app} and \eqref{eq:GOHess_1discrete_F_app}.

Next, we focus on the derivations related to the parameters $\gammav_1$ of the internal RV $\yv_1$, \ie $\nabla^2_{\gammav_1} \Lc(\gammav)$, and the derivations related to the correlation between $\gammav_1$ and $\gammav_2$, \ie $\nabla^2_{\gammav_1 \gammav_2} \Lc(\gammav)$.

On $\nabla^2_{\gammav_1} \Lc(\gammav)$, we have
\beq\label{eq:GOgradient_2layer_gamma1_app}
\resizebox{\hsize}{!}{$
\bali
	& \nabla_{\gammav_1} \Lc(\gammav) = \Ebb_{q_{\gammav_1} (\yv_1)} \big[ 
		\Gmat_{\gammav_1}^{q_{\gammav_1} (\yv_1)} 
		\nabla_{\yv_1} \Ebb_{q_{\gammav_2} (\yv_2 | \yv_1)} [f(\yv_1, \yv_2)] 
		\big]
	\\
	& \quad = \Ebb_{q_{\gammav} (\yv)} \left[
		\Gmat_{\gammav_1}^{q_{\gammav_1} (\yv_1)} 
		\left[\bali
		& \Gmat_{\yv_1}^{q_{\gammav_2} (\yv_2 | \yv_1)} \nabla_{\yv_2} f(\yv_1, \yv_2)
		\\
		& + \nabla_{\yv_1} f(\yv_1, \yv_2)
		\eali\right]
		\right]
\eali
$}
\eeq
and 
\beq\label{eq:GOHess_2layer_gamma1_app}
\resizebox{\hsize}{!}{$
\bali
	& \nabla^2_{\gammav_1} \Lc(\gammav) = \Ebb_{q_{\gammav_1} (\yv_1)} 
		\left[ \bali
		& \Gmat_{\gammav_1}^{q_{\gammav_1} (\yv_1)} 
		\big[\nabla^2_{\yv_1} \Ebb_{q_{\gammav_2} (\yv_2 | \yv_1)} [f(\yv_1, \yv_2)]\big] 
		\Gmat_{\gammav_1}^{q_{\gammav_1} (\yv_1)} {}^T
	\\
	& + \Hten_{\gammav_1 \gammav_1}^{q_{\gammav_1}(\yv_1)} 
		\nabla_{\yv_1} \Ebb_{q_{\gammav_2} (\yv_2 | \yv_1)} [f(\yv_1, \yv_2)] 
		\eali	\right]
	\\
	& = \Ebb_{q_{\gammav_1} (\yv_1)} \! \left[ \bali
		& \Gmat_{\gammav_1}^{q_{\gammav_1} (\yv_1)} 
		\!\left[ \nabla_{\yv_1} \Ebb_{q_{\gammav_2} (\yv_2 | \yv_1)} \! \left[\bali
		& \Gmat_{\yv_1}^{q_{\gammav_2} (\yv_2 | \yv_1)} \nabla_{\yv_2} f(\yv_1, \yv_2) 
		\\
		& + \nabla_{\yv_1} f(\yv_1, \yv_2)
		\eali\right] \right] \!
		\Gmat_{\gammav_1}^{q_{\gammav_1} (\yv_1)} {}^T
		\\
		& + \Hten_{\gammav_1 \gammav_1}^{q_{\gammav_1}(\yv_1)} \Ebb_{q_{\gammav_2} (\yv_2 | \yv_1)} \left[\bali
			& \Gmat_{\yv_1}^{q_{\gammav_2} (\yv_2 | \yv_1)} \nabla_{\yv_2} f(\yv_1, \yv_2)
			\\
			& + \nabla_{\yv_1} f(\yv_1, \yv_2)
			\eali\right] 
	\eali\right]
	\\
	& = \Ebb_{q_{\gammav} (\yv)} \left[ \bali
		& \Gmat_{\gammav_1}^{q_{\gammav_1} (\yv_1)} \left[\bali
		& \Gmat_{\yv_1}^{q_{\gammav_2} (\yv_2 | \yv_1)} 
		[\nabla^2_{\yv_2} f(\yv_1, \yv_2)] 
		\Gmat_{\yv_1}^{q_{\gammav_2} (\yv_2 | \yv_1)} {}^T
		\\
		& + \Hten_{\yv_1 \yv_1}^{q_{\gammav_2} (\yv_2 | \yv_1)} \nabla_{\yv_2} f(\yv_1, \yv_2)
		\\
		& + \Gmat_{\yv_1}^{q_{\gammav_2} (\yv_2 | \yv_1)} \nabla_{\yv_2} \nabla_{\yv_1} f(\yv_1, \yv_2)
		\\
		& + [\nabla^2_{\yv_1 \yv_2} f(\yv_1, \yv_2)] \Gmat_{\yv_1}^{q_{\gammav_2} (\yv_2 | \yv_1)} {}^T 
		\\
		& + \nabla^2_{\yv_1} f(\yv_1, \yv_1) \eali\right] 
		\Gmat_{\gammav_1}^{q_{\gammav_1} (\yv_1)} {}^T
		\\
		& + \Hten_{\gammav_1 \gammav_1}^{q_{\gammav_1}(\yv_1)} \left[\bali
		& \Gmat_{\yv_1}^{q_{\gammav_2} (\yv_2 | \yv_1)} \nabla_{\yv_2} [f(\yv_1, \yv_2)] 
		\\
		& + \nabla_{\yv_1} f(\yv_1, \yv_2)
		\eali\right] 
	\eali\right]
\eali
$}
\eeq
For better understanding, consider the deterministic optimization with objective 
\beq\label{eq:determin_opt_2layer_app}
\hat \Lc(\gammav) = f[\hat \yv_1(\gammav_1), \hat \yv_2(\gammav_2, \hat \yv_1(\gammav_1))]
\eeq
which is a special case of \eqref{eq:frameII_2layer_app} with $q_{\gammav_1}(\yv_1) = \delta(\yv_1 - \hat \yv_1(\gammav_1))$ and $q_{\gammav_2}(\yv_2|\yv_1) = \delta(\yv_2 - \hat \yv_2(\gammav_2, \hat \yv_1(\gammav_1)))$. For that deterministic objective, its Hessian wrt $\gammav_1$ is 
\beq\label{eq:GOHess_2layer_gamma1_d_app}
	\bali
	& \nabla^2_{\gammav_1} \hat \Lc(\gammav) = 
	\\
	& \left[ \bali
	& [\nabla_{\gammav_1} {\hat \yv_1}] \left[\bali
	& [\nabla_{\hat \yv_1} \hat \yv_2] [\nabla_{\hat \yv_2}^2 f] [\nabla_{\hat \yv_1} \hat \yv_2] {}^T
	\\
	& + [\nabla_{\hat \yv_1}^2 \hat \yv_2] [\nabla_{\hat \yv_2} f]
	\\
	& + [\nabla_{\hat \yv_1} \hat \yv_2] [\nabla^2_{\hat \yv_2 \hat \yv_1} f]
	\\
	& + [\nabla^2_{\hat \yv_1 \hat \yv_2} f] [\nabla_{\hat \yv_1} \hat \yv_2] {}^T 
	\\
	& + \nabla^2_{\hat \yv_1} f \eali\right]
	[\nabla_{\gammav_1} {\hat \yv_1}] {}^T
	\\
	& + [\nabla_{\gammav_1}^2 {\hat \yv_1}] \left[
	[\nabla_{\hat \yv_1} \hat \yv_2] [\nabla_{\hat \yv_2} f] + \nabla_{\hat \yv_1} f
	\right] 
	\eali\right]
	\eali
\eeq
which clearly shows the same patterns as \eqref{eq:GOHess_2layer_gamma1_app}, verifying the facts that 
($i$) $\Gmat_{\gammav}^{q_{\gammav}(\yv)} / \Hten_{\gammav\gammav}^{q_{\gammav}(\yv)}$ can be intuitively interpreted as the ``gradient''/``Hessian'' of the RV $\yv$ wrt parameters $\gammav$
and ($ii$) GO gradient/Hessian acts \emph{in expectation} the same as their deterministic counterpart, GO gradient/Hessian \ie \emph{in expectation} obeys the chain rule.

We then consider the correlation between $\gammav_1$ and $\gammav_2$, \ie $\nabla^2_{\gammav_1 \gammav_2} \Lc(\gammav) = \nabla^2_{\gammav_2 \gammav_1} \Lc(\gammav)$.

Based on \eqref{eq:GOgradient_2layer_gamma1_app}, we have
\beq\label{eq:GOHess_2layer_gamma12_app}
\resizebox{\hsize}{!}{$
\bali
	& \nabla^2_{\gammav_2 \gammav_1} \Lc(\gammav) 
	= \Ebb_{q_{\gammav_1} (\yv_1)} \left[
		\nabla_{\gammav_2} \Ebb_{q_{\gammav_2} (\yv_2 | \yv_1)}
		\left[\bali
			& \Gmat_{\yv_1}^{q_{\gammav_2} (\yv_2 | \yv_1)} \nabla_{\yv_2} f(\yv_1, \yv_2) 
			\\
			& + \nabla_{\yv_1} f(\yv_1, \yv_2)
			\eali\right]
		\Gmat_{\gammav_1}^{q_{\gammav_1} (\yv_1)} {}^T
		\right]
	\\
	& = \Ebb_{q_{\gammav} (\yv)} \left[\bali
		& \Gmat_{\gammav_2}^{q_{\gammav_2} (\yv_2 | \yv_1)} 
		\nabla_{\yv_2} \left[\bali
			& \Gmat_{\yv_1}^{q_{\gammav_2} (\yv_2 | \yv_1)} \nabla_{\yv_2} f(\yv_1, \yv_2) 
			\\
			& + \nabla_{\yv_1} f(\yv_1, \yv_2)
			\eali\right]
		\Gmat_{\gammav_1}^{q_{\gammav_1} (\yv_1)} {}^T
		\\
		& + \nabla_{\gammav_2} \left[\bali
			& \Gmat_{\yv_1}^{q_{\gammav_2} (\yv_2 | \yv_1)} \nabla_{\yv_2} f(\yv_1, \yv_2) 
			\\
			& + \nabla_{\yv_1} f(\yv_1, \yv_2)
			\eali\right]
		\Gmat_{\gammav_1}^{q_{\gammav_1} (\yv_1)} {}^T
		\eali\right]
	\\
	& = \Ebb_{q_{\gammav} (\yv)} \left[ \left[\bali
		& \Gmat_{\gammav_2}^{q_{\gammav_2} (\yv_2 | \yv_1)} [\nabla^2_{\yv_2} f(\yv_1, \yv_2)] \Gmat_{\yv_1}^{q_{\gammav_2} (\yv_2 | \yv_1)} {}^T
		\\
		& + \Gmat_{\gammav_2}^{q_{\gammav_2} (\yv_2 | \yv_1)} [\nabla_{\yv_2} \Gmat_{\yv_1}^{q_{\gammav_2} (\yv_2 | \yv_1)}] \nabla_{\yv_2} f(\yv_1, \yv_2)
		\\
		& + \Gmat_{\gammav_2}^{q_{\gammav_2} (\yv_2 | \yv_1)} \nabla^2_{\yv_2 \yv_1} f(\yv_1, \yv_2)
		\\
		& + [\nabla_{\gammav_2} \Gmat_{\yv_1}^{q_{\gammav_2} (\yv_2 | \yv_1)}] \nabla_{\yv_2} f(\yv_1, \yv_2)
		\eali\right]
		\Gmat_{\gammav_1}^{q_{\gammav_1} (\yv_1)} {}^T 
		\right]
	\\
	& = \Ebb_{q_{\gammav} (\yv)} \left[	\left[\bali
		& \Gmat_{\gammav_2}^{q_{\gammav_2} (\yv_2 | \yv_1)} [\nabla^2_{\yv_2} f(\yv_1, \yv_2)] \Gmat_{\yv_1}^{q_{\gammav_2} (\yv_2 | \yv_1)} {}^T
		\\
		& + \Hten_{\gammav_2 \yv_1}^{q_{\gammav_2} (\yv_2 | \yv_1)} \nabla_{\yv_2} f(\yv_1, \yv_2)
		\\
		& + \Gmat_{\gammav_2}^{q_{\gammav_2} (\yv_2 | \yv_1)} \nabla^2_{\yv_2 \yv_1} f(\yv_1, \yv_2)
		\eali\right]
		\Gmat_{\gammav_1}^{q_{\gammav_1} (\yv_1)} {}^T 
		\right].
\eali
$}
\eeq
Similarly, one can also draw parallel comparisons with the corresponding counterpart of the deterministic objective in \eqref{eq:determin_opt_2layer_app}, where
\beq\label{eq:GOHess_2layer_gamma12_d_app}
	\bali
	& \nabla^2_{\gammav_2 \gammav_1} \hat \Lc(\gammav) 
	\\
	& = \left[	\left[\bali
	& 
	[\nabla_{\gammav_2} \hat \yv_2] [\nabla^2_{\hat \yv_2} f] [\nabla_{\hat \yv_1} \hat \yv_2] {}^T
	\\
	& + [\nabla^2_{\gammav_2 \hat \yv_1} \hat \yv_2] [\nabla_{\hat \yv_2} f]
	\\
	& + [\nabla_{\gammav_2} \hat \yv_2] [\nabla^2_{\hat \yv_2 \hat \yv_1} f]
	\eali\right]
	[\nabla_{\gammav_1} \hat \yv_1] {}^T 
	\right].
	\eali
\eeq
The same patterns (or correspondences) are observed.

By parallel comparing $\nabla^2_{\gammav_2} \Lc(\gammav) $, $\nabla^2_{\gammav_1} \Lc(\gammav) $, and $\nabla^2_{\gammav_2 \gammav_1} \Lc(\gammav)$ (in \eqref{eq:GOHess_2layer_gamma22_app}, \eqref{eq:GOHess_2layer_gamma1_app}, and \eqref{eq:GOHess_2layer_gamma12_app}, respectively) with their deterministic counterparts (in \eqref{eq:grad12_simple_reverse_issue} of the main manuscript, \eqref{eq:GOHess_2layer_gamma1_d_app}, and \eqref{eq:GOHess_2layer_gamma12_d_app}, respectively), it's clear that the same consistent patterns are observed, \ie the GO gradient/Hessian \emph{in expectation} acts the same as the deterministic gradient/Hessian to obey the chain rule.

\subsubsection{Deeper settings with stochastic computation graphs with continuous internal nodes}

To explicitly derive the Hessian of expectation-based objectives over stochastic computation graphs (like the example illustrated in Figure \ref{fig:GO_grad} of the main manuscript) is cumbersome, because of the huge amount of combinations among the parameters of each RV node.

Fortunately, we find that, based on the above \eqref{eq:frameII_2layer_app} and the assumption that $\yv_2$ is continuous, one can readily derive all the second-order derivatives for where $f(\yv_1, \yv_2) = \Ebb_{q_{\gammav_3}(\yv_3 | \yv_1,\yv_2)} [\bar f(\yv_1, \yv_2, \yv_3)]$. 
The key is to recursively reuse two building blocks, \ie the single-layer GO gradient in \eqref{eq:GO_1} and the single-layer GO Hessian in \eqref{eq:GO_Hessian_1} of the main manuscript.
The detailed derivations are left to the readers.
The same consistent patterns as those in the previous Section \ref{secapp:two_layer_GOHessian} will be observed, \ie the GO gradient/Hessian \emph{in expectation} acts the same as the deterministic gradient/Hessian.

So forth to recursively add new continuous internal RV nodes until continuous/discrete leaf nodes, one may ``generate'' a stochastic computation graph and simultaneously prove via mathematical induction that the GO gradient/Hessian \emph{in expectation} obeys the chain rule, acting the same as the deterministic gradient/Hessian, for stochastic computation graphs with continuous internal nodes and continuous/discrete leaves.

\textbf{For practical implementation}, one merely needs to make sure correct \emph{variable-nabla}/\emph{variable-hess} as the first-order/second-order derivatives for each node of the stochastic graph, with the approach shown in Figure \ref{fig:PseudoCode} of the main manuscript, to deliver easy-to-use exploitation (via AD and HVP) of the GO Hessian over stochastic graphs with continuous internal nodes.

\subsubsection{One-sample/Multi-sample estimation of GO Hessian}

\textbf{Note the \emph{one-sample} term of the GO gradient/Hessian doesn't mean one data sample.}

By referring to Algorithm \ref{alg:SCR_GO} of the main manuscript, the \emph{one-sample} term means one sample $\yv_i$ \emph{stochastically activated} for each $\xv_i$ from $q_{\gammav_t}(\yv|\xv_i)$;
in other words, the one-sample (multi-sample) estimation means a single glance (multiple glances) of the information of parameters $\gammav_t$. 
It's important to note that the \emph{one-sample} term of GO gradient/Hessian and the batch size $N_g / N_H$ (the number of used data samples $\xv_i$) are two entirely orthogonal dimensions.

By default, we utilize one-sample estimation of the GO gradient/Hessian, because both of its low empirical variance and the common practice of one-sample-based training \cite{kingma2014auto,zhang2018whai,cong2019go}.

Despite that, one can of course leverage multiple MC samples (\ie multiple $\yv_i$ (multiple glances of parameters $\gammav_t$) for each $\xv_i$) to achieve multi-sample estimation for lower variance.
The multi-sample estimation of GO gradient/Hessian can be straight-forwardly implemented via multiple parallel forward/backward passes.

\subsection{GO Hessian contains the deterministic Hessian as a special case}
\label{appsec:GOHess_Hessian_specialcase}

It's shown in Table 3 of \cite{cong2019go} (refer to Appendix C therein) that, for a special $q_{\gamma}(y)=\delta(y-\mu)$ with $\gamma=\{\mu\}$, the \emph{variable-nabla}
\beq\label{eqapp:delta_nabla}
	g_{\mu}^{q_{\gamma}(y)} = \frac{-\nabla_{\mu} Q_{\gamma}(y)}{q_{\gamma}(y)} = 1.
\eeq

Next, we employ the notations from Section \ref{secapp:derivation_GO_Hessian} and point out first that conditional independence is satisfied, \ie 
\beq \bali
	& q_{\gammav_1}(\yv_1) = \delta(\yv_1 - \hat \yv_1(\gammav_1))
	\\ 
	& = \prod\nolimits_{v} \delta(y_{1v} - \hat y_{1v}(\gammav_1)) 
	= \prod\nolimits_{v} q_{\gammav_1}(y_{1v})
\eali\eeq
and
\beq 
\resizebox{\hsize}{!}{$\bali
	& q_{\gammav_2}(\yv_2|\yv_1) = \delta(\yv_2 - \hat \yv_2(\gammav_2, \hat \yv_1(\gammav_1)))
	\\ 
	& = \prod\nolimits_{k} \delta(y_{2k} - \hat y_{2k}(\gammav_2, \hat \yv_1(\gammav_1)))
	= \prod\nolimits_{k} q_{\gammav_2}(y_{2k}|\yv_1) .
\eali$}
\eeq

Based on \eqref{eqapp:delta_nabla}, taking $q_{\gammav_1}(\yv_1)$ as example, it's easy to show
\beq
\bali
	g_{\gammav_1}^{q_{\gammav_1}(y_{1v})} 
	& = [\nabla_{\gammav_1} \hat y_{1v}] g_{\hat y_{1v}}^{q_{\gammav_1}(y_{1v})}
	= \nabla_{\gammav_1} \hat y_{1v}
	\\
	\nabla_{y_{1v}} g_{\gammav_1}^{q_{\gammav_1}(y_{1v})} 
	& = \nabla_{y_{1v}} [\nabla_{\gammav_1} \hat y_{1v}] = 0
	\\
	\nabla_{\gammav_1} g_{\gammav_1}^{q_{\gammav_1}(y_{1v})} 
	& = \nabla_{\gammav_1} [\nabla_{\gammav_1} \hat y_{1v}] = \nabla^2_{\gammav_1} \hat y_{1v}.
\eali
\eeq
Accordingly, substituting them to the definition of $\Gmat_{\gammav_1}^{q_{\gammav_1}(\yv_{1})}$ and $\Hten_{\gammav_1 \gammav_1}^{q_{\gammav_1}(\yv_1)}$, we yield
\beq
\bali
	\Gmat_{\gammav_1}^{q_{\gammav_1}(\yv_{1})} 
	& = \nabla_{\gammav_1} \hat \yv_{1}
	\\
	\Hten_{\gammav_1 \gammav_1}^{q_{\gammav_1}(\yv_1)}
	& = \nabla^2_{\gammav_1} \hat \yv_{1}.
\eali
\eeq

Similar derivations can be readily verified for other \emph{variable-nabla}/\emph{variable-hess} used in calculating $\nabla^2_{\gammav_2} \Lc(\gammav) $, $\nabla^2_{\gammav_1} \Lc(\gammav) $, $\nabla^2_{\gammav_2 \gammav_1} \Lc(\gammav)$, and, more generally, the components of the GO Hessian of stochastic computation graphs with continuous internal nodes.  
To conclude, the GO Hessian contains the deterministic Hessian as a special case.

\section{Demonstrating the low variance of the GO Hessian}
\label{appsec:GO_Hessian_Variance}

We demonstrate the low variance of our GO Hessian with two representative/challenging examples, \ie continuous non-rep gamma RVs and discrete negative binomial (NB) RVs.

As variational inference (VI) is one of the closest related research fields, we adopt the terminology therein for better presentation.

With $\xv$ and $\zv$ denoting the observation and latent code, respectively, VI trains a variational posterior $q_{\phiv}(z|\xv)$ to approximate the true underlying posterior $p(z|\xv)$ via maximizing the ELBO,
\beq
\resizebox{\hsize}{!}{$\bali
	\ELBO (\phiv) 
	& = \Ebb_{q_{\phiv}(z|\xv)} \left[ \log p(\xv, z) - \log q_{\phiv}(z|\xv) \right] 
	\\
	& = \Ebb_{q_{\phiv}(z|\xv)} \left[ \log p(z|\xv) - \log q_{\phiv}(z|\xv) + \log p(\xv) \right] 
	\\
	& = -\KL [q_{\phiv}(z|\xv)||p(z|\xv)]  + \log p(\xv),
	\eali$}
\eeq
which is equivalent to minimizing the reverse KL divergence between $q_{\phiv}(z|\xv)$ and $p(z|\xv)$, \ie 
\beq\label{eq:KL_gamma}
\min_{\phi} \KL [q_{\phiv}(z|\xv)||p(z|\xv)] 
= \Ebb_{q_{\phiv}(z|\xv)} \left[ \log \frac{q_{\phiv}(z|\xv)}{p(z|\xv)} \right].
\eeq

\subsection{Gamma example}

We assume an analytic posterior gamma distribution $p(z|\xv)=\Gam(\alpha_p, \beta_p)$, where $\alpha_p=10, \beta_p=10$ are predefined, and specify the variational posterior $q_{\phiv}(z|\xv)=\Gam(\alpha_q, \beta_q)$ with learnable $\phiv=\{\alpha_q, \beta_q\}$ to approximate $p(z|\xv)$, via minimizing the reverse KL in \eqref{eq:KL_gamma} directly.

For demonstration, we compare the proposed GO Hessian to the naive log-trick estimation, implemented as in \eqref{eq:GO_Hessian_1} and \eqref{eq:Hess_REINFORCE} of the main manuscript, respectively.  
For implementation details, please see Section \ref{sec:Gam_Implement}.

Other experimental settings are listed as follows.
One Monte Carlo (MC) sample is used to estimate our GO Hessian (termed GO Hessian) and the log-trick estimation (termed log-trick). 
We test both estimators within the region of $\alpha_q \in [7, 13], \beta_q \in [7, 13]$. 
To measure the variance of the estimators, we use the Hessian error defined as 
$
\| \Hmat_{\text{esti}}(\alpha_q, \beta_q) - \Hmat_{\text{true}}(\alpha_q, \beta_q) \|_{\text{Fro}}
$, where $\Hmat_{\text{esti}}(\alpha_q, \beta_q)/\Hmat_{\text{true}}(\alpha_q, \beta_q)$ denotes the estimated/true Hessian at location $(\alpha_q, \beta_q)$ and $\| \Hmat \|_{\text{Fro}}$ is the Frobenius norm of the matrix $\Hmat$.

\subsection{NB example}

Following most settings from the above gamma example, we employ the true NB posterior $p(z|\xv)=\NB(r_p, p_p)$ with $r_p=10, p_p=0.5$ and specify the variational posterior $q_{\phiv}(z|\xv)=\NB(r_q, p_q)$ with learnable $\phiv=\{r_q, p_q\}$ to approximate $p(z|\xv)$, via minimizing the reverse KL in \eqref{eq:KL_gamma} directly.
Implementation details are provided in Section \ref{sec:NB_Implement}.

One MC sample is used to estimate our GO Hessian (termed GO Hessian) and the log-trick estimation (termed log-trick). 
Both estimators are tested within the region of $r_q \in [7, 13], p_q \in [0.35, 0.65]$. 
The true Hessian $\Hmat_{\text{true}}(r_q, p_q)$ of the reverse KL objective is estimated with 20,000 MC samples.

\subsection{On the log-trick estimation with control variates}

Following \cite{cong2019go}, we compare our Monte Carlo estimator (the GO Hessian) to the log-trick estimation without control variates, under the settings of one-sample-based estimation. 
Other concerns motivating the experimental setup of Figure \ref{fig:Hess_var_toy_gam_main} as listed as follows.

($i$) Leveraging additional control variates for the log-trick estimation may compromise the fairness when comparing the two Monte Carlo estimators (\ie the log-trick estimation and the proposed GO Hessian).

($ii$) Under the settings of one-sample-based estimation (appealing in practice and likely to be the common practice in variational inference and reinforcement learning), it's not straight-forward to design a control variate for the log-trick estimation, because
	\begin{itemize}
		\item it's clear that a variance reduction baseline (\ie the sample average) is not applicable for one-sample-based estimation;
		\item since only Hessian estimation is of interest, running average (often used when training with the log-trick estimation) is not an option as there is no running at all.
	\end{itemize}

($iii$) Exhaustive empirical experience has shown that the log-trick gradient estimator (or the REINFORCE), even with powerful control variates, is unlikely to work as well as the reparameterization (Rep) gradient (a special case of the GO gradient) in practice where the Rep is applicable. Double application of the log-trick (the log-trick Hessian estimation) is likely to further worsen the situation.

The experimental results are given in Figure \ref{fig:Hess_var_toy_gam_main} of the main manuscript, from where it's clear that GO Hessian has a much lower variance than that of the log-trick estimation in both examples.

\section{Subsolvers of the SCR-GO in Algorithm \ref{alg:SCR_GO} of the main manuscript}
\label{secapp:SCR_subsolvers}

\begin{algorithm}[H]
	\caption{Cubic-Subsolver via Gradient Descent} \label{alg:SCR_GO_subsolver}
	\begin{algorithmic}
		\REQUIRE {$\gv$, $\Hmat[\cdot]$, tolerance $\epsilon$.}
		
		\STATE {
			$R_c \leftarrow - \frac{\gv^T \Hmat[\gv]}{\rho \|\gv\|^2}
			+ \sqrt{\left[ \frac{\gv^T \Hmat[\gv]}{\rho \|\gv\|^2} \right]^2 + \frac{2\|\gv\|}{\rho}}$
		}
		\STATE {$\Deltav \leftarrow -R_c \frac{\gv}{\|\gv\|}$}
		
		\IF {$\| \gv \| \le l^2 / \rho$} 
		
		\STATE {$\sigma \leftarrow c'\frac{\sqrt{\epsilon\rho}}{l}, \eta=\frac{1}{20l}$}
		\STATE {$\tilde \gv \leftarrow \gv + \sigma \zeta$ for $\zeta \sim \text{Unif}(\Sbb^{d-1})$}
		
		\FOR {$t=1,\cdots, T(\epsilon)$}
			
		\STATE {$\Deltav \leftarrow \Deltav - \eta (\tilde \gv + \Hmat[\Deltav] + \frac{\rho}{2} \|\Deltav\| \Deltav)$}
		
		\ENDFOR
		
		\ENDIF
		
		\STATE {$\Delta \leftarrow \gv^T \Deltav + \frac{1}{2} \Deltav^T \Hmat[\Deltav] + \frac{\rho}{6} \|\Deltav\|^3$}
		
		\ENSURE {$\Deltav, \Delta$}
	\end{algorithmic}
\end{algorithm}

$\text{Unif}(\Sbb^{d-1})$ denotes the uniform distribution on the unit sphere in $\Rbb^d$.

\begin{algorithm}[H]
	\caption{Cubic-Finalsolver via Gradient Descent} \label{alg:SCR_GO_finalsolver}
	\begin{algorithmic}
		\REQUIRE {$\gv$, $\Hmat[\cdot]$, tolerance $\epsilon$.}
		
		\STATE {$\Deltav \leftarrow 0, \gv_m \leftarrow \gv, \eta \leftarrow \frac{1}{20l}$}
		
		\WHILE {$\|\gv_m\| > \frac{\epsilon}{2}$}
		
		\STATE {$\Deltav \leftarrow \Deltav - \eta \gv_m$}
		\STATE {$\gv_m \leftarrow \gv + \Hmat[\Deltav] + \frac{\rho}{2} \|\Deltav\| \Deltav$}
		
		\ENDWHILE
		
		\ENSURE {$\Deltav$}
	\end{algorithmic}
\end{algorithm}

\section{Convergence analysis of Algorithm \ref{alg:SCR_GO}}
\label{secapp:converge_alg1}


\begin{assumption}\label{assump1}
	The function $\Lc(\gammav)$ has:
	\begin{itemize}
		\item $l$-Lipschitz gradients: for all $\gammav_1$ and $\gammav_2$,
		\begin{equation*}
		\left\|\nabla \Lc(\gammav_1) -\nabla \Lc(\gammav_2)\right\|\leq l\|\gammav_1-\gammav_2\|;
		\end{equation*}
		\item $\rho$-Lipschitz Hessians: for all $\gammav_1$ and $\gammav_2$,
		\begin{equation*}
		\left\|\nabla^2 \Lc(\gammav_1)-\nabla^2 \Lc(\gammav_2)\right\|\leq \rho\|\gammav_1-\gammav_2\|.
		\end{equation*}
	\end{itemize}
\end{assumption}
\begin{assumption}\label{assump2}
	The function $\Lc({\gammav}) = \Ebb_{q_{\gammav}(\yv)}[f(\yv)]$ has 
	\begin{itemize}
		\item for all ${\gammav}$, $\left\|\hat{\Dbb}_{q_{\gammav}}f(\yv) - \nabla \Lc(\gammav)\right\|\leq M_1$ a.s.;
		\item for all $\gammav$, $\left\| \hat{\Hten}_{q_{\gammav}}f(\yv)-\nabla^2 \Lc({\gammav})\right\|\leq M_2$ a.s.
	\end{itemize}
\end{assumption}

\begin{theorem}
	There exists an absolute constant $c$ such that if $\Lc(\gammav)$ satisfies Assumptions \ref{assump1} and \ref{assump2}, Cubic-Subsolver $\left(\tilde \gv_t, \tilde \Hmat_t [\cdot], \epsilon\right)$ satisfies Condition 1 in \cite{tripuraneni2018stochastic} with $c$, 
	\begin{equation*}
	\begin{split}
	&n_1\geq\max\left(\frac{M_1}{c\epsilon},
	\frac{\sigma_1^2}{c^2\epsilon^2}\right)	
	\log\left(\frac{d\sqrt{\rho}\Delta_\Lc}{\epsilon^{1.5}\delta_c}\right),\\
	&~\mbox{and}~n_2\geq 
	\max\left(\frac{M_2}{c\sqrt{\rho\epsilon}}, \frac{\sigma_2^2}{c^2\rho\epsilon}\right)
	\log\left(\frac{d\sqrt{\rho}\Delta_\Lc}{\epsilon^{1.5}\delta_c}\right),
	\end{split}
	\end{equation*}
	then for all $\delta>0$, $\Delta_\Lc \geq \Lc(\gammav_0)$ and sufficiently small $\epsilon\leq \min\left(\frac{\sigma_1^2}{c_1M_1}, \frac{\sigma_2^4}{c_2^2M_2^2\rho}\right)$,
	Algorithm \ref{alg:SCR_GO} will output an $\epsilon$-second-order point of $\Lc$ with the probability at least $1-\delta$ within
	\beq
	\tilde{\mathcal{O}}
	\left(\frac{\sqrt{\rho}\Delta_\Lc}{\epsilon^{1.5}}
	\left(\frac{\sigma_1^2}{\epsilon^2}	+\frac{\sigma_2^2}{\rho\epsilon}\cdot \mathcal{T}(\epsilon)\right)\right)
	\eeq
	total stochastic gradient and HVP evaluations.
\end{theorem}
\begin{proof}
	The proof is almost the same as Theorem 1 in \cite{tripuraneni2018stochastic} and the only difference lies in the concentration conditions. It is easy to prove that  
	\begin{equation*}
	\begin{split}
	\|\tilde \gv_t - \nabla \Lc(\gammav_t)\|&\leq c_1\cdot \epsilon,\\
	\forall \vv, \|\left(\tilde \Hmat_t - \nabla^2 \Lc(\gammav_t)\right)\vv\|&\leq c_2\cdot \sqrt{\rho\epsilon}\|\vv\|.
	\end{split}
	\end{equation*}
	hold for sufficiently small $c_1, c_2$.	
\end{proof}

Instead of the parameterization of shape $\alphav$ and rate $\betav$ for gamma distribution, a different parameterization is used below, that is, the shape $\kv = \alphav$ and the scale $\thetav = \frac{1}{\betav}$.

\begin{theorem}
	Consider $\Lc(\gammav) =  \KL({q_{\gammav}, p})$, where $p(\cdot)$ and $q_{\gammav}(\cdot)$ are two gamma distributions:
	$q_{\gammav}(\yv) = \Gam\left(\yv; \kv_{\qv}, \thetav_{\qv}\right), p(\yv) = \Gam\left(\yv;\kv_{\pv}, \thetav_{\pv}\right)$, if $\kv_{\pv}, \kv_{\qv}, \thetav_{\pv}, \thetav_{\qv}$ are bounded constants, then Algorithm \ref{alg:SCR_GO} will output an $\epsilon$-second order point of $\Lc$.	
\end{theorem}
\begin{proof}
	Loss function can be written in the explicit form:
	\begin{equation*}
	\resizebox{\hsize}{!}{$
	\begin{split}
	\Lc(\gammav) &= \KL\left({q_{\gammav}, p}\right)\\ 
	&= \left(\kv_{\qv} - \kv_{\pv}\right)\psi^{(0)}(\kv_{\qv})-\log \Gamma\left(\kv_{\qv}\right) \\
	&+\log \Gamma\left(\kv_{\pv}\right)+ \kv_{\pv}\left(\log \thetav_{\pv} - \log\thetav_{\qv}\right)+\kv_{\qv} \frac{\thetav_{\qv} - \thetav_{\pv}}{\thetav_{\pv}}.
	\end{split}
	$}
	\end{equation*}
	Take derivatives with respect to the parameters of $\qv_{\gammav}(\cdot)$:
	\begin{equation*}
	\resizebox{\hsize}{!}{$\begin{split}
	\partial_{\kv_{\qv}} \Lc(\gammav) &= \psi^{(0)}(\kv_{\qv}) - (\kv_{\qv}-\kv_{\pv})\psi^{(1)}(\kv_{\qv}) -\psi^{(0)}(\kv_{\qv})\\
	& + \frac{\thetav_{\qv} - \thetav_{\pv}}{\thetav_{\pv}}\\
	\partial_{\thetav_{\qv}} \Lc(\gammav) &=\frac{\kv_{\pv}}{\thetav_{\qv}} + \frac{\kv_{\qv}}{\thetav_{\pv}},
	\end{split}$}
	\end{equation*}

	By triangle inequation
	\begin{equation}\label{triangle}
	\resizebox{\hsize}{!}{$
	\begin{split}
	&\left\|\partial_{\kv_{\qv}} \Lc(\gammav) - \partial_{\kv_{\qv}} \Lc(\tilde{\gammav}) \right\|\\
	\leq&\left\| \psi^{(0)}(\kv_{\qv}) -\psi^{(0)}(\tilde{\kv}_{\qv})\right\|
	+ \left\|{\kv}_{\qv}\psi^{(1)}({\kv}_{\qv})-\tilde{\kv}_{\qv}\psi^{(1)}(\tilde{\kv}_{\qv})\right\|\\
	&+ \left\|{\kv}_{\pv}\left(\psi^{(1)}(\kv_{\qv}) - \psi^{(1)}(\tilde{\kv}_{\qv})\right)\right\| \\
	&+ \left\|\psi^{(1)}(\kv_{\qv}) - \psi^{(1)}(\tilde{\kv}_{\qv})\right\|
	\left\|\partial_{\thetav_{\qv}} \Lc(\gammav) - \partial_{\theta_q} L(\tilde{\gamma}) \right\|\\
	\leq& \left\|\kv_{\pv}\left(\frac{1}{\theta_{\qv}} - \frac{1}{\tilde{\theta}_q}\right)\right\|
	+ \left\|\frac{\kv_{\qv}-\tilde{\kv}_{\qv}}{\thetav_{\pv}}\right\|
	\end{split}
	$}
	\end{equation}
	Estimating each term:
	\begin{equation*}
	\resizebox{\hsize}{!}{$
	\begin{split}
	&\left\|\psi^{(0)}(\kv_{\qv}) - \psi^{(0)}(\tilde{\kv}_{\qv})\right\|
	\leq \left\|\psi^{(1)}(\hat{\kv}_{\qv})\right\| \left\|\kv_{\qv}-\tilde{\kv}_{\qv}\right\|\\
	\leq& \psi^{(1)}(\kv_{\qv})\left\|\kv_{\qv}-\tilde{\kv}_{\qv}\right\|\\
	&\left\|\kv_{\pv}\left(\psi^{(1)}(\kv_{\qv}) - \psi^{(1)}(\tilde{\kv}_{\qv})\right)\right\| 
	+\left\|\psi^{(1)}(\kv_{\qv}) - \psi^{(1)}(\tilde{\kv}_{\qv})\right\|\\
	\leq& \left(\kv_{\pv}+1\right)\|\psi^{(2)}(\hat{\kv}_{\qv})\|\|\kv_{\qv} - \tilde{\kv}_{\qv}\|\\
	\leq& -\left({\kv}_{\pv}+1\right)\psi^{(2)}(\kv_{\qv})\|\kv_{\qv} - \tilde{\kv}_{\qv}\|\\
	&\left\|\kv_{\qv}\psi^{(1)}({\kv}_{\qv})-\tilde{\kv}_{\qv}\psi^{(1)}(\tilde{\kv}_{\qv})\right\|\\
	\leq& \left\|\hat{\kv}_{\qv}\psi^{(2)}(\hat{\kv}_{\qv})+ \psi^{(1)}(\hat{\kv}_{\qv})\right\|\left\|{\kv}_{\qv}-\tilde{\kv}_{\qv}\right\|\\
	\leq& \left(-\tilde{\kv}_{\qv}\psi^{(2)}({\kv}_{\qv})+
	\psi^{(1)}(\kv_{\qv})\right)
	\left\|\kv_{\qv}-\tilde{\kv}_{\qv}\right\|\\
	&\left\|\kv_{\pv}\left(\frac{1}{\thetav_{\qv}} - \frac{1}{\tilde{\thetav}_{\qv}}\right)\right\|
	= \frac{\kv_{\pv}}{\thetav_{\qv}\tilde{\thetav}_{\qv}}\left\|\thetav_{\qv} - \tilde{\thetav}_{\qv}\right\|
	\leq \frac{{\kv}_{\pv}}{\thetav_{\qv}^2}\left\|\thetav_{\qv} - \tilde{\thetav}_{\qv}\right\|\\
	&\left\|\frac{{\kv}_{\qv}-\tilde{\kv}_{\qv}}{\thetav_{\pv}}\right\| 
	= \frac{1}{\thetav_{\pv}}\left\|\kv_{\qv} - \tilde{\kv}_{\qv}\right\|
	\end{split}
	$}
	\end{equation*}
	
	Combining these terms into (\ref{triangle}), we have:
	\begin{equation*}
	\begin{split}
	&\left\|\partial_{\kv_{\qv}}\Lc(\gammav) - \partial_{\kv_{\qv}}\Lc(\tilde{\gammav})\right\|\\
	\leq& \left[2\psi^{(1)}(\kv_{\qv}) - \left(\kv_{\pv} +1+\tilde{\kv}_{\qv}\right)\psi^{(2)}({\kv}_{\qv})\right]\left\|\gammav - \tilde{\gammav}\right\|\\
	&\left\|\partial_{\thetav_{\qv}}\Lc(\gammav) - \partial_{\thetav_{\qv}}L(\tilde{\gammav})\right\|\\
	\leq& \left(\frac{{\kv}_{\pv}}{\thetav_{\qv}^2} + \frac{1}{\thetav_{\pv}}\right)\|\gammav - \tilde{\gammav}\|
	\end{split}
	\end{equation*}
	When $\kv_{\pv}, \kv_{\qv}, \thetav_{\pv}, \thetav_{\qv}$ are bounded constants, $\Lc(\gammav)$ satisfies the gradient Lipschitz condition.The Hessian Lipschitz condition can be verified in the same fashion.
\end{proof}

\section{Implementing GO gradient/Hessian for gamma RVs}
\label{sec:Gam_Implement}

Recall that the GO gradient and GO Hessian for single-layer continuous RVs are defined as 
\beq\label{eq:GOGradHessian}
\resizebox{\hsize}{!}{$\bali
\nabla_{\gammav} \Ebb_{q_{\gammav} (\yv)} [f(\yv)] 
& = \Ebb_{q_{\gammav} (\yv)} \Big[
\Gmat_{\gammav}^{q_{\gammav} (\yv)}
\nabla_{\yv} f(\yv)
\Big]
\\
\nabla_{\gammav}^2 \Ebb_{q_{\gammav} (\yv)} [f(\yv)] 
& = \Ebb_{q_{\gammav} (\yv)} \left[\bali
	& \Gmat_{\gammav}^{q_{\gammav}(\yv)} [\nabla_{\yv}^2 f(\yv)] \Gmat_{\gammav}^{q_{\gammav}(\yv)} {}^T
	\\
	& + \Hten_{\gammav\gammav}^{q_{\gammav}(\yv)} \nabla_{\yv} f(\yv)
\eali \right]
\eali$}
\eeq
where $\Gmat_{\gammav}^{q_{\gammav} (\yv)} = \big[ \cdots, g_{\gammav}^{q_{\gammav} (y_v)}, \cdots \big]$ with \emph{variable-nabla} $g_{\gammav}^{q_{\gammav} (y_v)} \triangleq  \frac{-1}{q_{\gammav} (y_v)} \nabla_{\gammav}  Q_{\gammav} (y_v)$ and $\Hten_{\gammav\gammav}^{q_{\gammav}(\yv)}$ is a $3$-D tensor with its elemental definition
$
(\Hten_{\gammav\gammav}^{q_{\gammav}(\yv)})_{b,a,v} = g_{\gamma_b}^{q_{\gammav} (y_v)} \nabla_{y_v} g_{\gamma_a}^{q_{\gammav} (y_v)} + \nabla_{\gamma_b} g_{\gamma_a}^{q_{\gammav} (y_v)}.
$ 

It has been shown in Section \ref{sec:GOHess_HVP} of the main manuscript that one merely needs to guarantee correct \emph{variable-nabla}/\emph{variable-hess} for each RV node to deliver correct exploitation via AD of the GO Hessian of a stochastic computation graph.

Accordingly, we focus on a scalar gamma RV for clarity in the following derivations. 
It's clear from both definitions in \eqref{eq:GOGradHessian} that three basic terms are necessary to calculate the GO gradient/Hessian for a scalar continuous RV $y \sim q_{\gammav}(y)$, that is
\beq\label{eq:3comp_gamma}
\bali
g_{\gammav}^{q_{\gammav} (y)},
\nabla_{y} g_{\gammav}^{q_{\gammav} (y)},
\nabla_{\gammav} g_{\gammav}^{q_{\gammav} (y)}.
\eali
\eeq
For simplicity, we first notice that a gamma RV $\hat y \sim \Gam(\alpha, \beta)$, with shape $\alpha$ and rate $\beta$, has the reparameterization $\hat y = {y}/{\beta},  y \sim \Gam(\alpha, 1)$, with which $\beta$ can be reparameterized to enable AD for exploiting the corresponding derivatives.

Therefore, we only need to deal with the challenging non-rep part of back-propagating (twice) through the Gamma RV $y \sim q_{\alpha} (y) = \Gam(\alpha, 1)$; accordingly, that three basic terms become
\beq\label{eq:3comp_gamma_sim}
\bali
g_{\alpha}^{q_{\alpha} (y)},
\nabla_{y} g_{\alpha}^{q_{\alpha} (y)},
\nabla_{\alpha} g_{\alpha}^{q_{\alpha} (y)}.
\eali
\eeq

With references from wolfram functions (mostly from \url{http://functions.wolfram.com/GammaBetaErf/Gamma2/20/01/01/} and \url{http://functions.wolfram.com/GammaBetaErf/Gamma2/20/01/02/}) and tedious derivations, we have 
\beq\label{eq:G_alpha_gam}
\bali
& g_{\alpha}^{q_{\alpha} (y)}
\\ 
& = [g_{\alpha}^{q_{\alpha} (y)}]_L \triangleq \left[ \bali
	& -[\log z - \psi(\alpha)] \frac{\gamma(\alpha, z)}{z^{\alpha-1} e^{-z}}	
	\\
	& + \frac{z e^z}{\alpha^2} {}_2 F_2(\alpha, \alpha; \alpha+1, \alpha+1; -z) 
	\eali \right]
\\
& = [g_{\alpha}^{q_{\alpha} (y)}]_R \triangleq \frac{1}{z^{\alpha-1} e^{-z}} \left[ \bali
	& [\log z - \psi(\alpha)] \Gamma(\alpha, z) 
	\\
	& + G_{2,3}^{3,0} (z |_{0,0,\alpha}^{1,1}) 
	\eali \right] ,
\eali
\eeq
where $[\cdot]_L$ and $[\cdot]_R$ represent two equivalent calculations for the same term (with different properties as detailed below), $\psi(\alpha)$ is the digamma function, $\gamma(\alpha, y)$ the lower incomplete gamma function, ${}_p F_q(a_1,\cdots,a_p; b_1,\cdots,b_q; x)$ is the generalized hypergeometric function \url{http://functions.wolfram.com/HypergeometricFunctions/HypergeometricPFQ/}, and $G_{p,q}^{m,n} (x |_{b_1,\cdots,b_m,b_{m+1},\cdots,b_q}^{a_1,\cdots,a_n,a_{n+1},\cdots,a_p})$ is the Meijer G-function \url{http://functions.wolfram.com/HypergeometricFunctions/MeijerG/},
\beq\label{eq:Ny_G_alpha_gam}
\resizebox{\columnwidth}{!}{$
\bali
& \nabla_{y} g_{\alpha}^{q_{\alpha} (y)} 
\\
& = [\nabla_{y} g_{\alpha}^{q_{\alpha} (y)}]_L \triangleq \left[\bali
	& [\psi(\alpha) - \log z] + \frac{y-\alpha+1}{y} 
	\\
	& \times \left[ \bali
		& [\psi(\alpha) - \log z] \frac{\gamma(\alpha, z)}{z^{\alpha-1} e^{-z}} + \\
		& \frac{z e^z}{\alpha^2} {}_2 F_2(\alpha, \alpha; \alpha+1, \alpha+1; -z) 
		\eali \right]
	\eali \right]
\\
& = [\nabla_{y} g_{\alpha}^{q_{\alpha} (y)}]_R \triangleq \left[\bali
	& [\psi(\alpha) - \log z] + \frac{1}{z^{\alpha-1} e^{-z}} \frac{y-\alpha+1}{y} 
	\\
	& \times \left[ \bali
	& [\log z - \psi(\alpha)]\Gamma(\alpha, z) \\
	& + G_{2,3}^{3,0} (z |_{0,0,\alpha}^{1,1})  
	\eali \right]
	\eali \right],
\eali
$}
\eeq
where $\Gamma(\alpha, y) $ the upper incomplete gamma function, and 
\beq\label{eq:Nalpha_G_alpha_gam}
\resizebox{\columnwidth}{!}{$
\bali
& \nabla_{\alpha} g_{\alpha}^{q_{\alpha} (y)} 
\\
& = [\nabla_{\alpha} g_{\alpha}^{q_{\alpha} (y)}]_L \triangleq
\left[ \bali 
& \psi^{(1)}(\alpha) \frac{\gamma(\alpha, z)}{z^{\alpha-1} e^{-z}} + [\log z - \psi(\alpha)]
\\
& \times \frac{z e^z}{\alpha^2} {}_2 F_2(\alpha, \alpha; \alpha+1, \alpha+1; -z) \\
& - \frac{2 z e^z}{\alpha^3} {}_3 F_3(\alpha, \alpha, \alpha; \alpha+1,\alpha+1, \alpha+1; -z)
\eali\right] 
\\
& = [\nabla_{\alpha} g_{\alpha}^{q_{\alpha} (y)}]_R \triangleq
	\frac{1}{z^{\alpha-1} e^{-z}} \left[ \bali
	& -\psi^{(1)}(\alpha) \Gamma(\alpha, z) 
	\\
	& - [\psi(\alpha) - \log z] G_{2,3}^{3,0} (z |_{0,0,\alpha}^{1,1}) 
	\\
	& + 2 G_{3,4}^{4,0} (z |_{0,0,0,\alpha}^{1,1,1})
	\eali \right],
\eali
$}
\eeq
where $\psi^{(m)}(x)$ is the polygamma function of order $m$ with $\psi^{(0)}(x)=\psi(x)$.

Since existing AD softwares do not support the calculation of the above special functions like the generalized hypergeometric function or the Meijer G-function, we practically resort to the mpmath library \cite{mpmath} for help, which is developed for real and complex floating-point arithmetic with arbitrary precision and provides supports for those special functions of interest.

Based on the mpmath library, one can calculate the equivalent $[\cdot]_L$ and $[\cdot]_R$ forms for the three items in \eqref{eq:3comp_gamma_sim}. However in practice, we empirically find that the $[\cdot]_L$ forms are more computationally efficient (about $20$ times faster than the $[\cdot]_R$ forms with mpmath) with its reliability focusing on the left side, \ie $y \le y_{L}$ (with $y_{L}$ being some threshold, as shown in Figure \ref{fig:LR_terms_compare}), for a specific $\alpha$ and a given computational precision; while the $[\cdot]_R$ forms run slower (probably because of the Meijer G-function) with its reliability focusing on the right side \ie $y \ge y_{R}$ for some threshold $y_{R}$ (see Figure \ref{fig:LR_terms_compare}). An example for $\alpha=100$ and a mpmath decimal precision of mp.dps=15 is given in Figure \ref{fig:LR_terms_compare}. One can of course set a better precision to make both forms more reliable but with additional computational cost, for example setting mp.dps=50 will ``correct'' the $[\cdot]_L$ forms to align them to the green true values in the tested region. We empirically found that mp.dps=50 achieves a proper compromise; accordingly, we use mp.dps=50 in all our experiments. Considering the efficiency of the $[\cdot]_L$ forms and their reliability after setting mp.dps=50, we used them instead of the $[\cdot]_R$ forms in our implementation.

\begin{figure*}
	\centering
	\subfigure[$g_{\alpha}^{q_{\alpha} (y)} $] {\label{fig:g_alpha_100}
		\includegraphics[width=0.6\columnwidth]{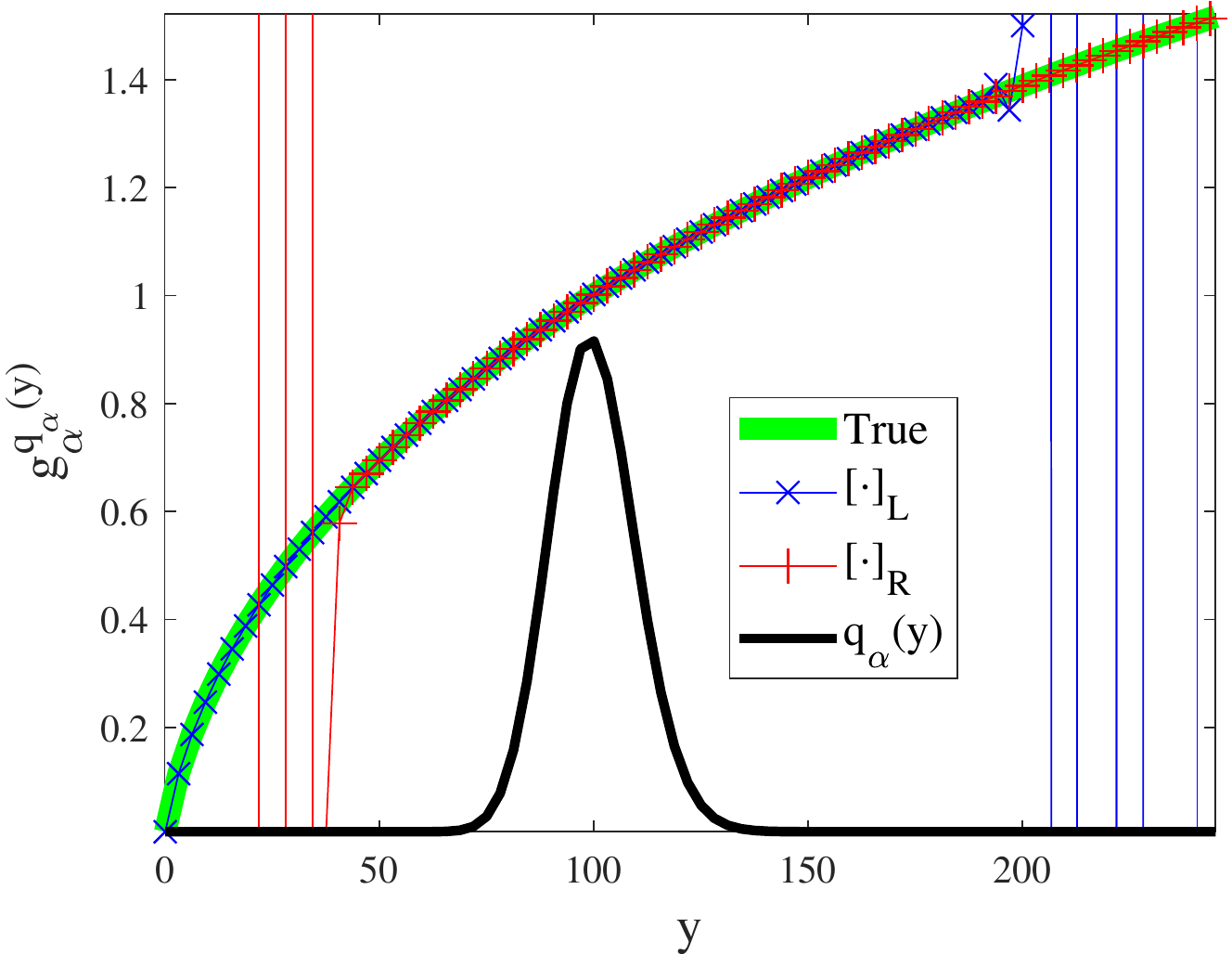}
	}
	\subfigure[$\nabla_{y} g_{\alpha}^{q_{\alpha} (y)}$] {\label{fig:Nz_g_alpha_100}
		\includegraphics[width=0.6\columnwidth]{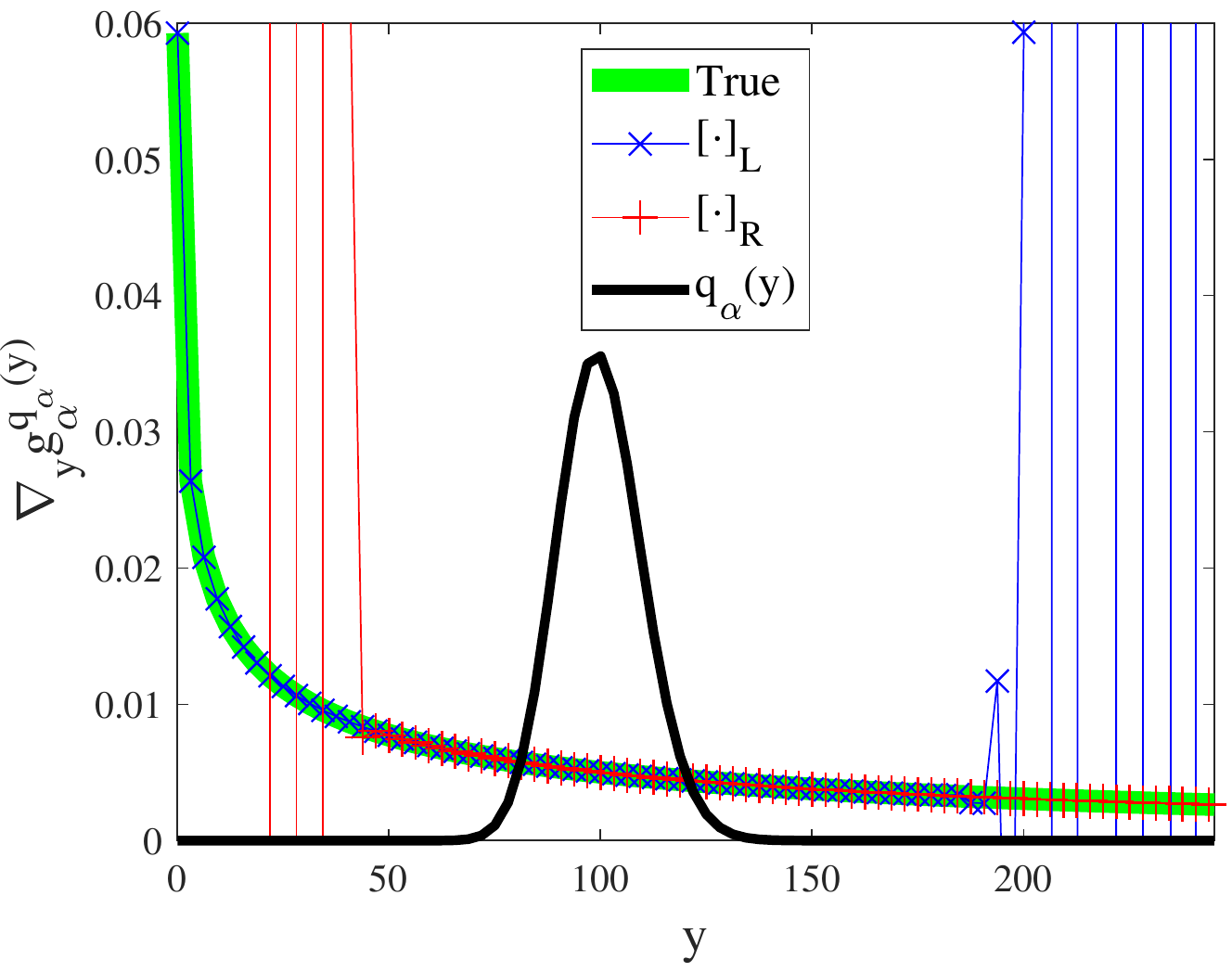}
	}
	\subfigure[$\nabla_{\alpha} g_{\alpha}^{q_{\alpha} (y)}$] {\label{fig:N_alpha_g_alpha_100}
		\includegraphics[width=0.6\columnwidth]{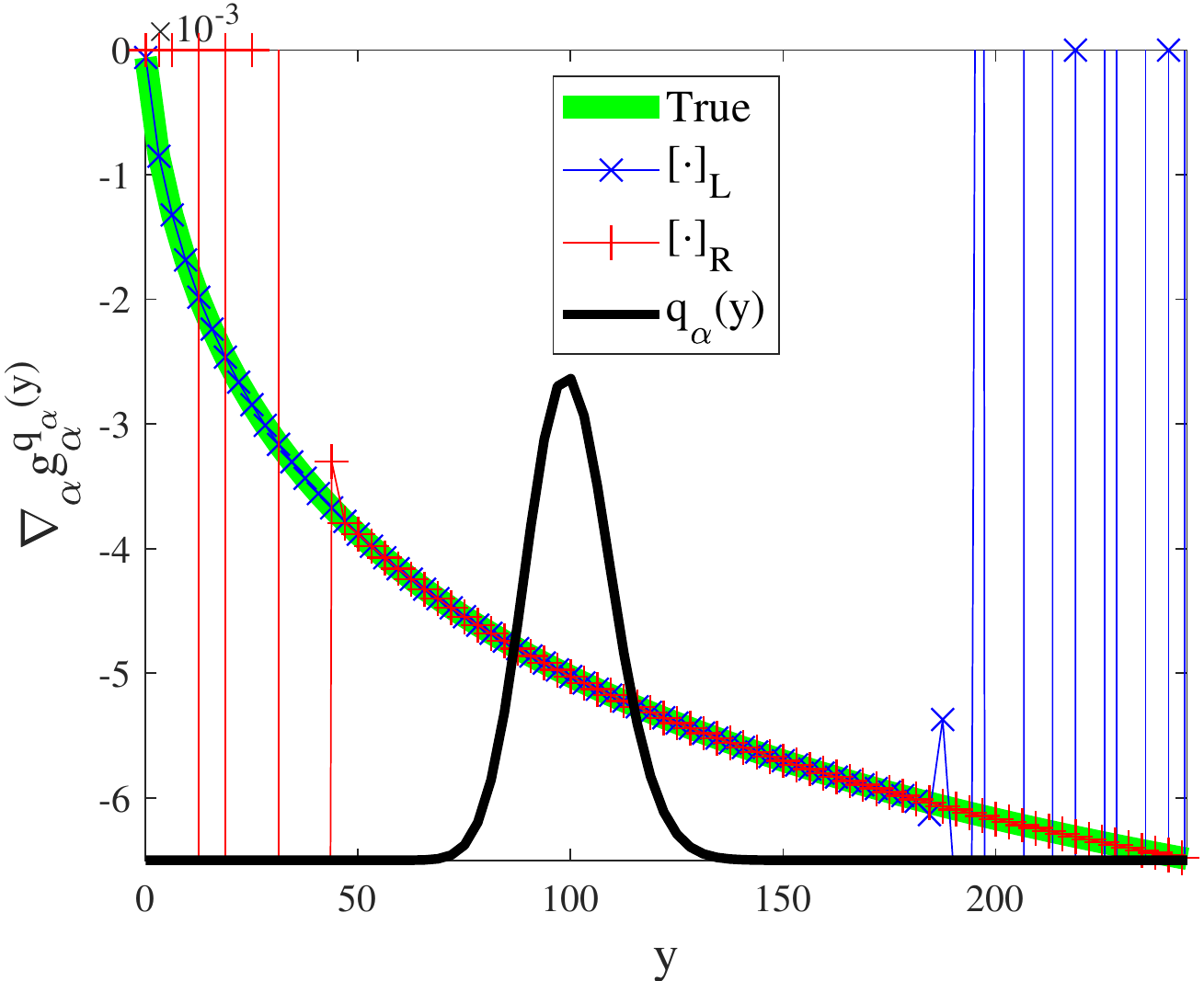}
	}
	\caption{Demonstrating the reliability of three $[\cdot]_L$ and $[\cdot]_R$ forms for a gamma RV with $\alpha=100$ and $\beta=1$ with default mpmath decimal precision (mp.dps=15). 
		$[\cdot]_L / [\cdot]_R$ is more reliable in the left/right. 
		For reference, the black lines show the rescaled gamma PDF, \ie $c q_{\alpha}(y)$ with some constant $c$. 		
		By setting a better precision (\ie mp.dps=50), one can ``correct'' the $[\cdot]_L$ terms to align with the true values (the green lines) in the tested region. Note empirically, the $[\cdot]_L$ terms run about $20$ times faster than the $[\cdot]_R$ ones. By referring to the rescaled gamma distribution, it is clear that setting mp.dps=50 would make the $[\cdot]_L$ terms reliable in practice. 
	}
	\label{fig:LR_terms_compare}
\end{figure*}

Having solved the above three basic terms in \eqref{eq:3comp_gamma_sim}, we empirically found another precision issue for practically implementations, relating to back-propagating through $y$ for small $\alpha$. That issue originates from the fact that a gamma RV sample would be exponentially close to zero with the decreasing of its shape $\alpha$ (with rate $\beta=1$). For example, when $\alpha=0.1$, a sample $y$ would be $7.3540 \times 10^{-66}$, while for $\alpha=0.01$, $y$ can approach the precision limit of $2.2251 \times 10^{-308}$. 
For practical calculations where a term $\frac{1}{y^2}$ emerges (like in calculating the gradient of the reverse KL divergence of two gamma distributions as in \eqref{eq:KL_gamma}), precision overflow would happen, leading to training error.
Therefore, in practice we use the data type of $64$-bit floating point (\ie torch.float64 in PyTorch), constraint $\alpha \ge 0.05$, and truncate the gamma sample $y \ge 10^{-120}$ for safe training.

\section{Implementing GO gradient/Hessian for negative binomial RVs}
\label{sec:NB_Implement}

Recall that the GO gradient and GO Hessian for single-layer/leaf discrete RVs are defined as 
\beq\label{eq:GOGradHessianNB}
\resizebox{\hsize}{!}{$\bali
\nabla_{\gammav} \Ebb_{q_{\gammav} (\yv)} [f(\yv)] 
& = \Ebb_{q_{\gammav} (\yv)} \Big[
\Gmat_{\gammav}^{q_{\gammav} (\yv)}
\Dbb_{\yv} f(\yv)
\Big]
\\
\nabla_{\gammav}^2 \Ebb_{q_{\gammav} (\yv)} [f(\yv)] 
& = \Ebb_{q_{\gammav} (\yv)} \left[\bali
& \Gmat_{\gammav}^{q_{\gammav}(\yv)} [\Dbb_{\yv}^2 f(\yv)] \Gmat_{\gammav}^{q_{\gammav}(\yv)} {}^T
\\
& + \overline{\Hten_{\gammav\gammav}^{q_{\gammav}(\yv)} \Dbb_{\yv} f(\yv)}
\eali \right]
\eali$}
\eeq
where $\Gmat_{\gammav}^{q_{\gammav} (\yv)} = \big[ \cdots, g_{\gammav}^{q_{\gammav} (y_v)}, \cdots \big]$ with \emph{variable-nabla} $g_{\gammav}^{q_{\gammav} (y_v)} \triangleq  \frac{-1}{q_{\gammav} (y_v)} \nabla_{\gammav}  Q_{\gammav} (y_v)$ 
and $\scriptstyle \overline{\Hten_{\gammav\gammav}^{q_{\gammav}(\yv)} \Dbb_{\yv} f(\yv)}$ represents a matrix with its elements $$
\resizebox{\hsize}{!}{$
	\big[{\scriptstyle{\overline{\Hten_{\gammav\gammav}^{q_{\gammav}(\yv)} \Dbb_{\yv} f(\yv)}}} \big]_{b,a} = \sum_{v} \left[ \bali
	& [g_{\gamma_b}^{q_{\gammav} (y_v)} \Dbb_{y_v} g_{\gamma_a}^{q_{\gammav} (y_v)}] \Dbb_{y_v} f(\yv^{v+}) 
	\\
	& + [\nabla_{\gamma_b} g_{\gamma_a}^{q_{\gammav} (y_v)}] \Dbb_{y_v} f(\yv) 
	\eali \right].
	$}
$$

Based on Section \ref{sec:GOHess_HVP} of the main manuscript, we only needs to guarantee correct \emph{variable-nabla}/\emph{variable-hess} for each RV node to deliver correct exploitation via AD of the GO Hessian of a stochastic computation graph.
Accordingly, we focus on a scalar negative binomial (NB) RV for clarity in the following derivations.

In addition to the $f$-related calculations, to calculate the GO gradient/Hessian for a discrete NB RV $y \sim q_{\alphav}(y) = \NB(r, p) = \frac{\Gamma(y+r)}{y!\Gamma(r)} (1-p)^r p^y$ with the number of failures $r$, the success probability $p$, and the distribution parameters $\alphav=\{r,p\}$,
three groups of basic terms are necessary, that is,
\beq\label{eq:GOgradHessNBterms}
\bali
& \{g_{r}^{q_{\alphav} (y)},
g_{p}^{q_{\alphav} (y)}\},
\\
& \{\Dbb_{y} g_{r}^{q_{\alphav} (y)},
\Dbb_{y} g_{p}^{q_{\alphav} (y)}\},
\\
& \{\nabla_{r} g_{r}^{q_{\alphav} (y)},
\nabla_{p} g_{r}^{q_{\alphav} (y)},
\nabla_{r} g_{p}^{q_{\alphav} (y)},
\nabla_{p} g_{p}^{q_{\alphav} (y)}\}.
\eali
\eeq

With references from wolfram functions (mostly from \url{http://functions.wolfram.com/GammaBetaErf/BetaRegularized/20/01/01/}, \url{http://functions.wolfram.com/GammaBetaErf/BetaRegularized/20/01/02/}, and \url{http://functions.wolfram.com/HypergeometricFunctions/HypergeometricPFQRegularized/26/01/01/}) and tedious derivations, we have 
\beq\label{eq:Gr_Gp}
\resizebox{\hsize}{!}{$\bali
& g_{r}^{q_{\alphav} (y)} = (y+r) \left[\bali
	&-\left(\bali
		& \log(1-p) - \psi(r)
		\\
		& + \psi(r+y+1)
	\eali\right) 
	\frac{B_{1-p}(r,y+1)}{(1-p)^r p^y}
	\\
	& + \frac{{}_3 F_2(r,r,-y;r+1,r+1,1-p)}{p^y r^2}
\eali\right]
\\ 
& g_{r}^{q_{\alphav} (y)} = \frac{y+r}{1-p},
\eali$}
\eeq
where $\psi(\cdot)$ is the digamma function, 
$B_{1-p}(r, y+1)$ the incomplete gamma function \url{http://functions.wolfram.com/GammaBetaErf/Beta3/}, 
and ${}_p F_q(a_1,\cdots,a_p; b_1,\cdots,b_q; x)$ is the generalized hypergeometric function \url{http://functions.wolfram.com/HypergeometricFunctions/HypergeometricPFQ/}.

Based on \eqref{eq:Gr_Gp}, it's straight-forward to calculate 
\beq\label{appeq:Dy_gr_Dy_gp}
\{\Dbb_{y} g_{r}^{q_{\alphav} (y)},
\Dbb_{y} g_{p}^{q_{\alphav} (y)}\}.
\eeq

Further leveraging the references from \url{http://functions.wolfram.com/GammaBetaErf/BetaRegularized/20/01/02/0002/} and 
$$
\resizebox{\hsize}{!}{$\bali
\frac{\partial^2 I_z (a,b)}{\partial z \partial a} 
& = \frac{\partial^2 I_z (a,b)}{\partial a \partial z}
\\
& = \frac{z^{a-1} (1-z)^{b-1}}{B(a,b)} 
	\Big[ \log z - \psi(a) + \psi(a+b) \Big],
\eali$}
$$
where $I_z (a,b)$ is the regularized incomplete beta function \url{http://functions.wolfram.com/GammaBetaErf/BetaRegularized/} and $B(a, b)$ is the beta function, we yield
\beq\label{eq:Nr_Gr}
\resizebox{\columnwidth}{!}{$\bali
& \nabla_{r} g_{r}^{q_{\alphav} (y)} =
\\
& \left[
	\bali
		& \left[\bali
			& - \Big( \log(1-p) - \psi(r) + \psi(r+z+1) \Big)
			\\
			& + (z+r) \big( \psi^{(1)}(r) - \psi^{(1)}(r+z+1) \big)
		\eali\right]
		\frac{B_{1-p}(r,y+1)}{(1-p)^r p^y}
		\\
		& + \left[1 + (z+r) \Big( \bali
			& \log(1-p) - \psi(r) 
			\\
			& + \psi(r+z+1) 
			\eali \Big)\right]
		\frac{{}_3 F_2(r,r,-y;r+1,r+1,1-p)}{p^y r^2}
		\\
		& - 2(z+r) \frac{{}_4 F_3(r,r,r,-y;r+1,r+1,r+1,1-p)}{p^y r^3}
	\eali
\right],
\eali$}
\eeq
and
\beq\label{eq:Np_Gr}
\resizebox{\columnwidth}{!}{$
	\nabla_{p} g_{r}^{q_{\alphav} (y)} = \left[\bali
	& -(z+r) (\frac{r}{1-p} - \frac{z}{p}) \left[\bali
		& \left[\bali
			& \log(1-p) - \psi(r) 
			\\
			& + \psi(r+z+1) 
			\eali\right] \frac{B_{1-p}(r,y+1)}{(1-p)^r p^y}
		\\
		& - \frac{{}_3 F_2(r,r,-y;r+1,r+1,1-p)}{p^y r^2}
	\eali\right]
	\\
	& + \frac{z+r}{1-p} \Big( \log(1-p) - \psi(r) + \psi(r+z+1) \Big)
\eali\right],
$}
\eeq
where $\psi^{(m)}(\cdot)$ is the polygamma function of order $m$.

Finally, based on \eqref{eq:Gr_Gp}, it's straight-forward to derive
\beq\label{appeq:Nr_gp_Np_gp}
\bali
\nabla_{r} g_{p}^{q_{\alphav} (y)} & = \frac{1}{1-p}
\\
\nabla_{p} g_{p}^{q_{\alphav} (y)} & = \frac{z+r}{(1-p)^2}
\eali
\eeq

Collecting the results of \eqref{eq:Gr_Gp}-\eqref{appeq:Nr_gp_Np_gp},
we yield the three groups of basic terms in \eqref{eq:GOgradHessNBterms}.
By substituting them into \eqref{eq:GOGradHessianNB}, we deliver the GO gradient/Hessian for a discrete NB (leaf) node.

Because of the involved special functions, we currently rely on the mpmath library \cite{mpmath} for implementation.
Noticing the similarity between GO Hessian definitions for continuous and discrete RVs (compare \eqref{eq:GOGradHessian} with \eqref{eq:GOGradHessianNB}), it's highly possible that an approach (mimicking the one shown in Figure \ref{fig:PseudoCode} of the main manuscript) can be developed to enable an easy-to-use implementation via auto-differentiation.
However, we consider that beyond the scope of this paper and leave that for future research.

\section{Experimental settings}
\label{secapp:exp_settings}

Three methods are compared, \ie the standard SGD, the popular adaptive first-order method Adam, and our SCR-GO. 
The number of oracle calls per iteration for the compared methods are summarized in Table \ref{tab:Complexity_compare}. 

We give below more detailed experimental settings, in additional to what's given in the main manuscript, for reproducible research.
Code will be available at \url{github.com/YulaiCong/GOHessian}.   

\begin{table}[tb]
	\centering
	\caption{The number of oracle calls per iteration. 
		$N_g / N_H$ denotes the batch size used to estimate the gradient/Hessian. $T_{\text{sub}}$ is the number of iterations used in the Cubic-Subsolver of SCR-GO.
		\label{tab:Complexity_compare}}	
	\begin{tabular}{c|ccc}
		Method & SGD & Adam & SCR-GO  \\
		\hline
		\# Oracle Calls   & $N_g$ & $N_g$ & $N_g + N_H T_{\text{sub}}$  \\
	\end{tabular}
\end{table}

\subsection{Settings used in Section \ref{sec:KL_twoGamma_SCR}}

In the $\alpha$-$\beta$ space, we parameterize $\alpha = \text{softplus}(\gamma_{\alpha})$ and $\beta = \text{softplus}(\gamma_{\beta})$, with $\gammav=\{\gamma_{\alpha}, \gamma_{\beta}\}$ is the trainable parameters.

One-sample estimation of the GO gradient or GO Hessian is employed for the compared methods.
As no observation $\xv$ exists in this experiment (see Algorithm \ref{alg:SCR_GO} of the main manuscript), one may interpret all batch sizes to be 1, \eg $N_g=N_H=1$.

In the $\alpha$-$\beta$ space, for SGD$_{\alpha,\beta}$, we use the learning rate of $0.01$, which is selected by searching within $\{0.01, 0.05, 0.1, 0.5, 1\}$.
For SCR-GO$_{\alpha,\beta}$, we use $\rho=5$, $T(\epsilon)=3$, and a $0.0001$ noise for the cubic sub-problem.

By contrast in the $\mu$-$\sigma$ space, we parameterize $\alpha = \frac{\mu^2}{\sigma^2}, \beta = \frac{\mu}{\sigma^2}$, with $\mu = \text{softplus}(\gammav_{\muv})$, $\sigmav = \text{softplus}(\gammav_{\sigmav})$, and trainable $\gammav=\{\gamma_{\mu}, \gamma_{\sigma}\}$.
For SGD$_{\mu,\sigma}$, we search and select the learning rate of $0.1$. 
For Adam$_{\mu,\sigma}$, we use the learning rate of $1$ and the default hyperparameters.
For SCR-GO$_{\mu,\sigma}$, we use $\rho=0.1$ for the cubic sub-problem. 
Other parameters are the same with those in the $\alpha$-$\beta$ space.

\subsection{Settings used in Section \ref{sec:MFVI_PFA}}
\label{secapp:MFVI_PFA}

In this experiment, we adopt a modified MNIST dataset for demonstration. Specifically, we choose $5$ digits per class to form a new dataset containing only $50$ data samples. 
The dimensionality of $\zv$ is set to $20$. 
The softmax function is applied to each column of $\Wmat$ to make sure they are located in the simplex. 
$\alphav_0=\betav_0=1$.
$\{\muv_i,\sigmav_i\}$ is parameterized as $\muv_i = \text{Softplus}(\gammav_{\muv i})$ and $\sigmav_i=\text{Softplus}(\gammav_{\sigmav i})$, with $\phiv_i = \{\gammav_{\muv i}, \gammav_{\sigmav i} \}$ the learnable parameters associated with the $i$th observation $\xv_i$.

To remove the influence of the second-order optimization on $\thetav=\{\Wmat\}$, we use the same RMSprop optimizer (with learning rate $0.1$) on $\thetav$ for both Adam and SCR-GO.
The difference between Adam and SCR-GO is that the former utilizes the Adam optimizer when optimizing over $\{\phiv_i\}$ while the latter leverages our SCR-GO to train $\{\phiv_i\}$.

For the Adam optimizer, we search the learning rate within $\{0.001, 0.005, 0.01, 0.05, 0.1, 0.5, 1\}$ and choose the best learning rate of $0.1$. 
For our SCR-GO (see Algorithm \ref{alg:SCR_GO} of the main manuscript), we use $\rho=0.1$, $T_{sub}=5$, and noise $0.01$ for the cubic sub-problem. 
To solve the cubic sub-problem, instead of using the standard gradient decent method, we alternatively use the RMSprop optimizer with learning rate $10^{-2}$, which empirically performs better.
We use the whole $50$ data samples $\{\xv_i\}$, \ie $N_g=N_H=50$, with one-sample-estimated latent codes $\{\zv_i\}$ to estimate both GO gradient and GO Hessian (\ie one-sample estimation).
5 runs based on different random seeds are used to estimate the error-bars/variances.
Note our SCR-GO could be more efficient if we use a smaller batch size $N_H$ to estimate GO Hessian.

\subsection{Settings used in Section \ref{sec:VAE_PFA}}

\subsubsection{Variational encoder for PFA}
\label{secapp:VAE_PFA}

\begin{table}[tb]\centering
	\caption{Neural network architecture of $\NN_{\muv}(\xv)$ and $\NN_{\sigmav}(\xv)$.}\label{tabapp:PFA_QNN_net}
	\begin{tabular}{cc}
		\hline\hline
		Layer             & Output                  \\ 
		\hline
		Linear             & $20$                  \\ 
		Softplus($\beta=0.05$) & $-$   \\ 
		\hline\hline
	\end{tabular}
\end{table}

Table \ref{tabapp:PFA_QNN_net} shows the network architecture used to parameterize both $\NN_{\muv}(\xv)$ and $\NN_{\sigmav}(\xv)$.

In this experiment, we adopt joint learning for $\{\thetav, \phiv\}$ to test our SCR-GO in a more practical setup, as joint learning is more commonly used than the alternate optimization considered in Section \ref{secapp:MFVI_PFA} (or Section \ref{sec:MFVI_PFA} of the main manuscript).
One-sample estimation of the GO gradient or GO Hessian is employed for the compared methods.

For the Adam optimizer, we search the learning rate within $\{0.001, 0.005, 0.01, 0.05, 0.1, 0.5, 1\}$ and choose the best learning rate of $0.005$. For our SCR-GO (see Algorithm \ref{alg:SCR_GO} of the main manuscript), we use $\rho=0.1$, $T_{sub}=5$, noise $10^{-7}$, RMSprop with learning rate $5\times 10^{-4}$ for the cubic sub-problem. Different from Section \ref{secapp:MFVI_PFA} (or Section \ref{sec:MFVI_PFA} of the main manuscript), we use a batch size of 10, \ie $N_H=10$, to estimate the GO Hessian. Other settings are the same with Section \ref{secapp:MFVI_PFA}.

As confirmed by the results in Figure \ref{fig:MNIST_PFA_JQNN} of the main manuscript, subsampling data to estimate GO Hessian doesn't influence the final performance too much but indeed provide better efficiency wrt oracle calls.

\subsubsection{Variational encoder for PGBN/DLDA}
\label{secapp:VAE_PGBN}

\begin{figure}[tb]
	\centering
	\subfigure[] {\label{fig:}
		\includegraphics[width=0.7\columnwidth]{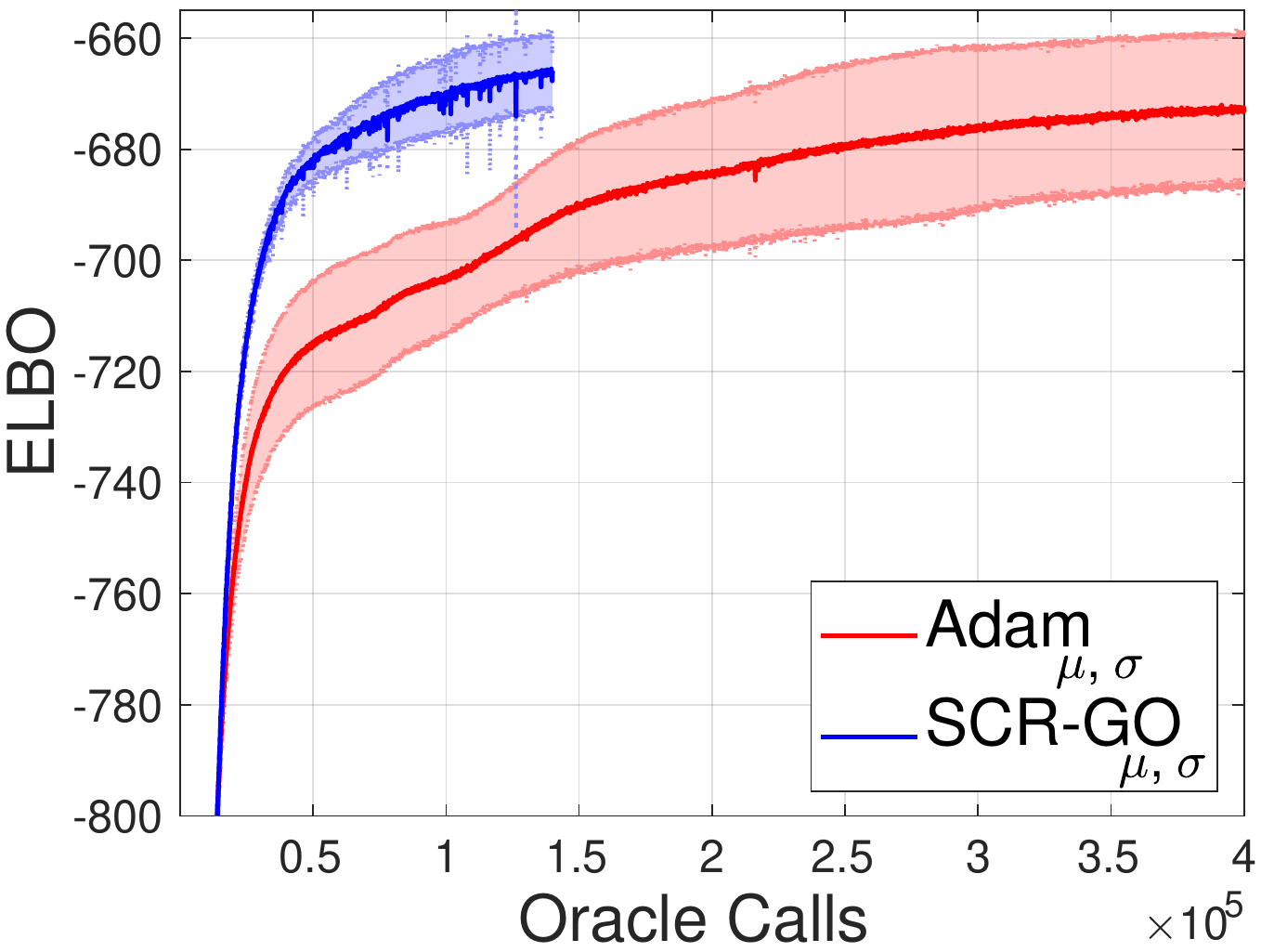}
	}
	\subfigure[] {\label{fig:}
		\includegraphics[width=0.7\columnwidth]{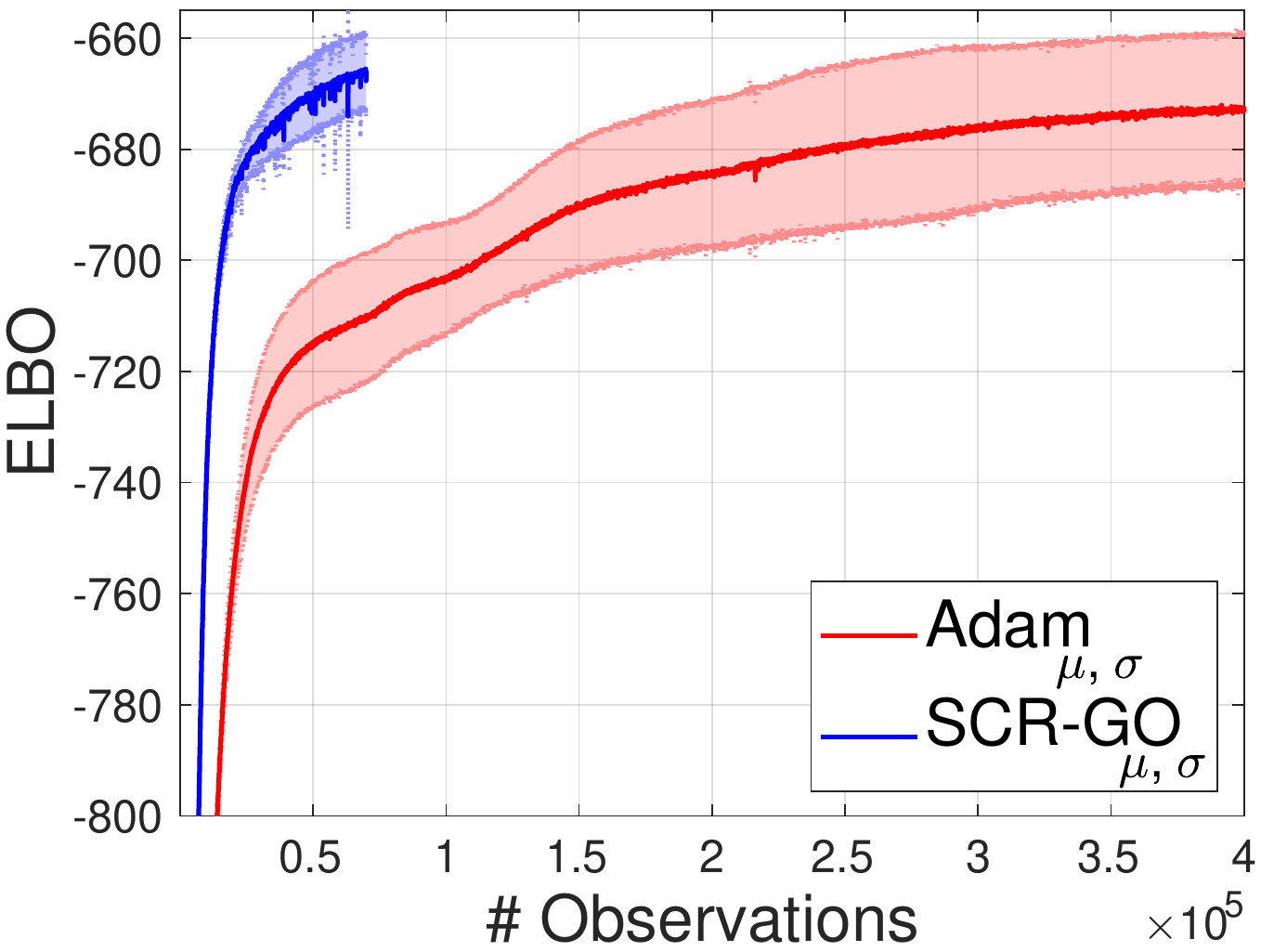}
	}
	\caption{Comparing our SCR-GO with a well-tuned Adam optimizer on training a variational encoder for PFA.
	Error-bars/Variances are estimated based on $5$ random seeds.
	}
	\label{fig:MNIST_PFA2_JQNN}
	\vspace{-0.2 cm}
\end{figure}

Generalizing the PFA, a Poisson gamma belief networks (PGBN, identical to the deep latent Dirichlet allocation (DLDA)) \cite{Zhou2015Poisson,zhou2016augmentable,cong2017deep} is a deep latent variable model with the generative process of (take the $2$-layer special case as an example)
\beq
p_{\thetav}(\xv, \zv): \left\{
\bali
\xv & \sim \Pois(\xv | \Wmat_1 \zv_1)
\\ 
\zv_1 & \sim \Gam(\zv_1| \Wmat_2 \zv_2, \cv_2)
\\
\zv_2 & \sim \Gam(\zv_2|\alphav_0, \betav_0),
\eali
\right.
\eeq
where $\xv$ is the count data variable, $\Wmat_l$ the topic matrix of layer $l$ with each column/topic $\wv_k$ located in the simplex, \ie $w_{vk} > 0, \sum\nolimits_{v} w_{vk} = 1$, $\zv_l$ the latent code of layer $l$, $\zv=\{\zv_l\}$, and $\thetav=\{\{\Wmat_l\}, \{\cv_l\}, \alphav_0, \betav_0\}$. 
Often $\{\alphav_0, \betav_0\}$ are assumed constants such as $1$.
For simplicity, we further assume $\cv_2=1$ in this experiment.

The variational inference arm is constructed hierarchically as 
\beq
q_{\phiv}(\zv | \xv) = q_{\phiv_2}(\zv_2 | \zv_1) q_{\phiv_1}(\zv_1 | \xv),
\eeq
where $\phiv = \{\phiv_1,\phiv_2\}$, 
\beq
q_{\phiv_2}(\zv_2 | \zv_1) = \Gam(\zv_2; \frac{\muv_2^{2}}{\sigmav_2^{2}}, \frac{\muv_2}{\sigmav_2^{2}}),
\eeq
with $\muv_2 = \NN_{\muv_2}(\zv_1), \sigmav_2 = \NN_{\sigmav_2}(\zv_1)$, and
\beq
q_{\phiv_1}(\zv_1 | \xv) = \Gam(\zv_1; \frac{\muv_1^{2}}{\sigmav_1^{2}}, \frac{\muv_1}{\sigmav_1^{2}}),
\eeq
with $\muv_1 = \NN_{\muv_1}(\xv), \sigmav_1 = \NN_{\sigmav_1}(\xv)$.
The NN functions are parameterized the same as in Table \ref{tabapp:PFA_QNN_net}.

The training objective is to maximize the ELBO.
\beq
\ELBO(\thetav, \phiv) = \Ebb_{q_{\phiv}(\zv | \xv)} \big[ \log p_{\thetav}(\xv, \zv) - \log q_{\phiv}(\zv | \xv) \big].
\eeq

For the Adam optimizer, we search the learning rate within $\{0.001, 0.005, 0.01, 0.05, 0.1, 0.5, 1\}$ and choose the best learning rate of $0.01$. For our SCR-GO (see Algorithm \ref{alg:SCR_GO} of the main manuscript), we use $\rho=0.1$, $T_{sub}=5$, noise $10^{-7}$, RMSprop with learning rate $5\times 10^{-4}$ for the cubic sub-problem. $10$ data samples (\ie $N_H=10$) are used to estimate the GO Hessian. 
Other settings are the same with Section \ref{secapp:VAE_PFA}.

The training objectives versus the number of oracle calls and processed observations are shown in Figure \ref{fig:MNIST_PFA2_JQNN}. 
It's clear that similar to what's observed in the above section (variational encoder for PFA),  the proposed SCR-GO performs better than a well-tuned Adam optimizer in terms of oracle calls and data efficiency, when tested on the more challenging problem of training a deep variational encoder (constructed via neural networks) for a deep latent variable model (PGBN/DLDA).
The better performance of SCR-GO is attributed to its exploitation of the curvature information via the GO Hessian, which takes into consideration the correlation among parameters within $p_{\thetav}(\xv, \zv)$ and $q_{\phiv}(\zv | \xv)$ and utilizing an (implicit) adaptive learning rate mimicking the classical Newton's method.



\end{document}